%% file: paper.tex
\documentclass[11pt, a4paper]{scrartcl}

\usepackage[utf8]{inputenc}
\usepackage[english]{babel}
\usepackage{amsmath,amssymb,amsfonts,mathrsfs}
\usepackage{mathtools}
\usepackage[amsmath,thmmarks]{ntheorem}
\usepackage{dsfont}
\usepackage{colortbl}
\usepackage{fullpage}
\usepackage{algorithmic}
\usepackage{multirow,booktabs,bigdelim}
\usepackage{caption}
\usepackage{algorithm}
\usepackage{comment}
\usepackage{graphicx}
\usepackage{placeins}
\usepackage{natbib}
\usepackage{siunitx}
\usepackage{url}
\usepackage{paralist}
\usepackage{enumitem}
\usepackage{tikz}
\usepackage{tkz-graph}
\usetikzlibrary{shapes.geometric}
\usetikzlibrary{backgrounds}
\usetikzlibrary{arrows}
\usepackage[framemethod=TikZ]{mdframed}
\usepackage{dirtytalk}
\usepackage[labelsep=period]{caption}
\usepackage[hidelinks]{hyperref}
\usepackage{float}
\usepackage{appendix}
\floatstyle{plain}
\newfloat{dataset}{ht}{lop}
\floatname{dataset}{Data Set}

\newdimen\arrowsize
\pgfarrowsdeclare{arcsq}{arcsq}
{
  \arrowsize=0.2pt
  \advance\arrowsize by .5\pgflinewidth
  \pgfarrowsleftextend{-4\arrowsize-.5\pgflinewidth}
  \pgfarrowsrightextend{.5\pgflinewidth}
}
{
  \arrowsize=1.5pt
  \advance\arrowsize by .5\pgflinewidth
  \pgfsetdash{}{0pt} 
  \pgfsetroundjoin   
  \pgfsetroundcap    
  \pgfpathmoveto{\pgfpoint{0\arrowsize}{0\arrowsize}}
  \pgfpatharc{-90}{-140}{4\arrowsize}
  \pgfusepathqstroke
  \pgfpathmoveto{\pgfpointorigin}
  \pgfpatharc{90}{140}{4\arrowsize}
  \pgfusepathqstroke
}

\DeclarePairedDelimiterX{\norm}[1]{\lVert}{\rVert}{#1}
\DeclarePairedDelimiterX{\abs}[1]{\lvert}{\rvert}{#1}

\DeclareMathOperator*{\argmin}{argmin}

\newcommand{\B}[1]{\mathbf{#1}}
\newcommand{\prob}{{\mathbb P}}
\newcommand{\E}{{\mathbb E}}
\newcommand{\R}{{\mathbb R}}

\newcommand{\N}{{\mathbb N}}

\newcommand{\var}{{\mathbf{var}}}  
\newcommand{\eps}{\varepsilon}  
\newcommand{\PA}[2][]{{\B{PA}}^{#1}_{#2}}

\newcommand{\dY}[2]{\dot{Y}^{(#1)}_{#2}}
\newcommand{\dYshort}[1]{\dot{Y}_{#1}}
\newcommand{\dYest}[2]{\widehat{\dot{Y}}^{(#1)}_{#2}}

\newenvironment{proof}{{\noindent\textbf{Proof}}}{\hfill$\square$\vskip\baselineskip}
\newtheorem{assumption}{Assumption}
\newtheorem{theorem}{Theorem}

\newtheorem{lemma}[theorem]{Lemma}

\graphicspath{{plots/}}

\usepackage{xpatch}
\makeatletter
\xpatchcmd{\endmdframed}
  {\aftergroup\endmdf@trivlist\color@endgroup}
  {\endmdf@trivlist\color@endgroup\@doendpe}
  {}{}
\makeatother

\title{Learning stable and predictive structures in kinetic systems:
  Benefits of a causal approach}

\author{Niklas Pfister\\ 
ETH Z\"urich, Switzerland\\
\url{niklas.pfister@stat.math.ethz.ch}
\and
Stefan Bauer\\ 
ETH Z\"urich, Switzerland\\
MPI T\"ubingen, Germany\\
\url{stefan.bauer@tuebingen.mpg.de}
\and Jonas Peters \\ 
University of Copenhagen, Denmark\\
\url{jonas.peters@math.ku.dk}}

\begin{document}
\maketitle

\begin{abstract}
  Learning kinetic systems from data is one of the core challenges in
  many fields.  Identifying stable models is essential for the
  generalization capabilities of data-driven inference.  We introduce
  a computationally efficient framework, called CausalKinetiX, that
  identifies structure from discrete time, noisy observations,
  generated from heterogeneous experiments.  The algorithm assumes the
  existence of an underlying, invariant kinetic model, a key criterion
  for reproducible research.  Results on both simulated and real-world
  examples suggest that learning the structure of kinetic systems
  benefits from a causal perspective. The identified variables and
  models allow for a concise description of the dynamics across
  multiple experimental settings and can be used for prediction in
  unseen experiments. We observe significant improvements compared to
  well established approaches focusing solely on predictive
  performance, especially for out-of-sample generalization.
\end{abstract}

\section*{Introduction}
Quantitative models of kinetic systems have become a cornerstone of
the modern natural sciences and are universally used in scientific
fields as diverse as physics, neuroscience, genetics, bioprocessing,
robotics or economics \cite{friston2003dynamic,chen1999modeling,
  ogunnaike1994process,murray2017mathematical, zhang2005differential}.
In systems biology, mechanistic models based on differential
equations, although not yet standard, are being increasingly used, for
example, as biomarkers to predict patient outcomes
\cite{fey2015signaling}, to improve predicting ligand dependent tumors
\cite{hass2017predicting} or for developing mechanism-based cancer
therapeutics \cite{arteaga2014erbb}. While the advantages of a
mechanistic modeling approach are by now well established, deriving
such models from hand is a difficult and labour intensive manual
effort.  With new data acquisition technologies \cite{ren2003unique,
  regev2017science, rozman2018identification} learning kinetic systems
from data has become a core challenge.

Existing data driven approaches infer the parameters of ordinary
differential equations by considering the goodness-of-fit of the
integrated system as a loss function \cite{Bard1974, Benson79}.  To
infer the structure of such models, standard model selection
techniques and sparsity enforcing regularizations can be used.  When
evaluating the loss function or perfoming an optimization step, these
methods rely on numerically integrating the kinetic system.  There are
various versions, and here we concentrate on the highly optimized
Matlab implementation \texttt{data2dynamics} \cite{data2dynamics}.  It
can be considered as a state-of-the-art implementation for directly
performing an integration in each evaluation of the loss
function. However, even with highly optimized integration procedures,
the computational cost of existing methods is high and depending on
the model class, these procedures can be infeasible. Moreover,
existing data driven approaches, not only those using numerical
integration, infer the structure of ordinary differential equations
from a single environment, possibly containing data pooled from
several experiments, and focus solely on predictive performance. Such
predictive based procedures have difficulties in capturing the
underlying causal mechanism and as a result, they may not predict well
the outcome of experiments that are different from the ones used for
fitting the model.

Here, we propose an approach to model the dynamics of a single target
variable rather than the full system. The resulting computational gain
allows our method to scale to systems with many variables. By
efficiently optimizing a non-invariance score our algorithm
consistently identifies causal kinetic models that are invariant
across heterogeneous experiments. In situations, where there is not
sufficient heterogeneity to guarantee identification of a single
causal model, the proposed variable ranking may still be used to
generate causal hypotheses and candidates suitable for further
investigation.  We demonstrate that our novel framework is robust
against model misspecification and the existence of hidden
variables. The proposed algorithm is implemented and available as an
open source R-package. The results on both simulated and real-world
examples suggest that learning the structure of kinetic systems
benefits from taking into account invariance, rather than focusing
solely on predictive performance.  This finding aligns well with a
recent debate in data science proposing to move away from
predictability as the sole principle of inference
\citep{ScholkopfJPSZMJ2012, yu2013, Peters2016jrssb,
  Bareinboim2016pnas, Meinshausen2016pnas, Shiffrin2016pnas, yu2019}.

\section*{Results}

\paragraph{Predictive models versus causal models.}
Established methods mostly focus on predictability when inferring
biological structure from data by selecting models.  This learning
principle, however, does not necessarily yield models that generalize
well to unseen experiments, since purely predictive models remain
agnostic with respect to changing environments or experimental
settings.  Causal models \cite{Pearl2009, Imbens2015} explicitly model
such changes by the concept of interventions.  The principle of
autonomy or modularity of a system \cite{Haavelmo1944,Aldrich1989}
states that the mechanisms which are not intervened on, remain
invariant (or stable). This is why causal models are expected to work
more reliably when predicting under distributional shifts
\cite{PearlMackenzie18,Peters2017book, Oates2014}.

\paragraph{Causality through stability.}
In most practical applications the causal structure is unknown, but it
may still be possible to infer the direct causes of a target variable
$Y$, say, if the system is observed under different, possibly
unspecified experimental settings. For non-dynamical data, this can be
achieved by searching for models that are stable across all
experimental conditions, i.e., the parameter estimates are similar.
Covariates that are contained in all stable models, i.e., in their
intersection, can be proven to be causal predictors for $Y$
\citep{Peters2016jrssb, Eaton2007}. The intersection of stable models,
however, is not necessarily a good predictive model.  In this work, we
propose a method for dynamical systems that combines stability with
predictability, we show that the inferred models generalize to unseen
experiments and we formalize its relation to causality (Methods).

\paragraph{CausalKinetiX: combining stability and predictability.}
The observed data consist of a target variable $Y$ and covariates $X$
measured at several time points across different experimental setups
and is assumed to be corrupted with observational noise,
Figure~\ref{fig:framework} (top).  Our proposed method, CausalKinetiX,
exploits the assumption that the model governing the dynamics of $Y$
remains invariant over the different experiments. We assume there is a
subset $S^*$ of covariates, s.t.\ for all $n$ repetitions,
$\frac{d}{dt} Y_t$ depends on the covariates in the same way, i.e.,
\begin{equation}
  \label{eq:invariance}
  \tfrac{d}{dt}Y^{(i)}_{t} = f\big({X}^{S^*,(i)}_t\big), \, \text{ for all } i = 1, \ldots, n.
\end{equation}
The covariates are allowed to change arbitrarily across different
repetitions $i$. Instead of only fitting based on predictive power,
CausalKinetiX explicitly measures and takes into account violations of
the invariance in [\ref{eq:invariance}]. Figure~\ref{fig:framework}
depicts the method's full workflow.
\begin{figure}[ht]
  \centering
  \includegraphics[width=\columnwidth]{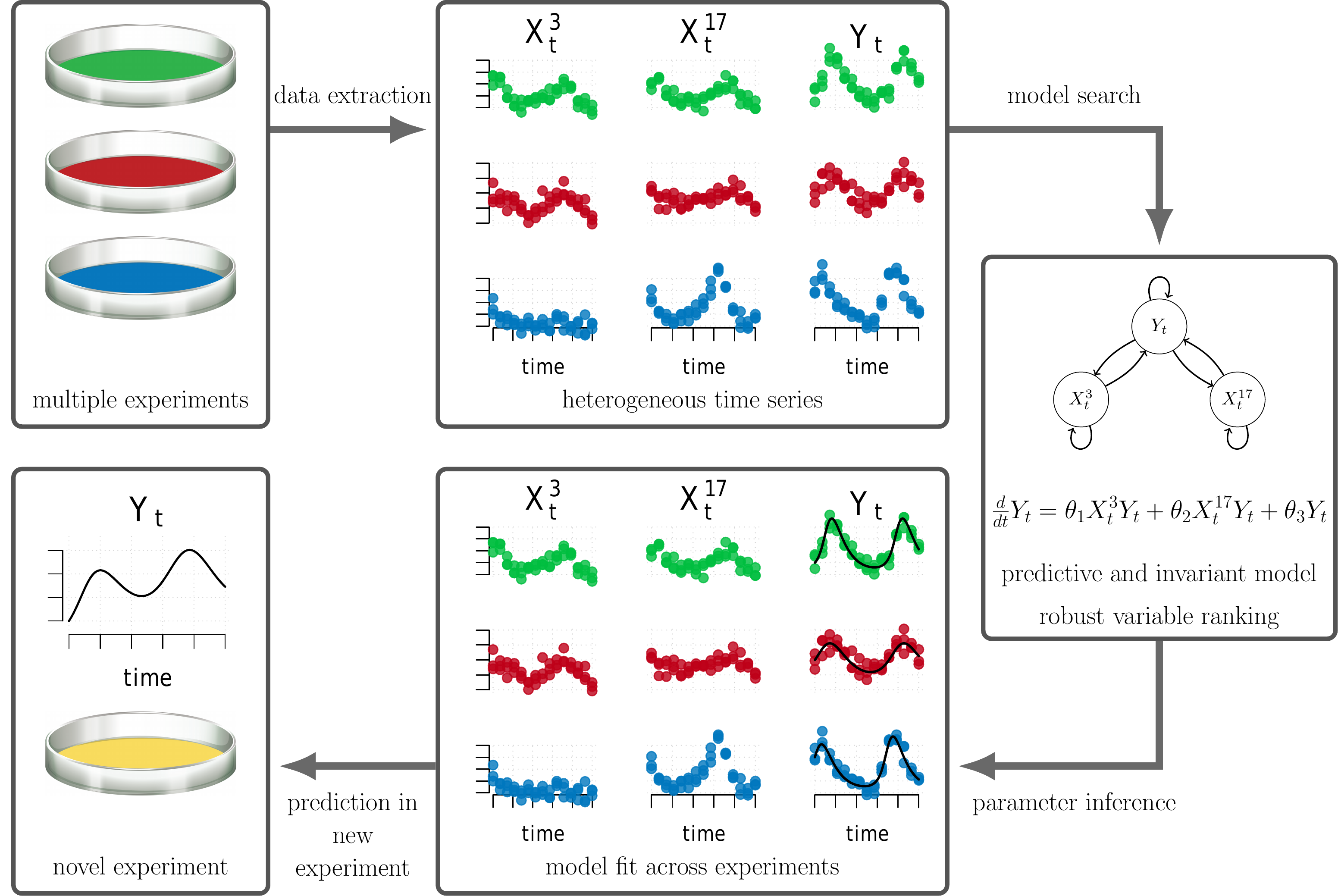}
  \caption{The framework of CausalKinetiX: the data for a target
    variables $Y$ and predictors $X$ come from different experiments; we
    rank models according to their ability to fit the target well in
    all experiments; the top ranked model is then fit to the data; it
    allows to predict the target in an unseen experiment.}
  \label{fig:framework}
\end{figure}
It ranks a collection of candidate models
$\mathcal{M}=\{M_1,\dots,M_m\}$ for the target variable (Methods)
based both on their predictive performance and whether the invariance
in [\ref{eq:invariance}] is satisfied. For a single model, e.g.,
$\frac{d}{dt}Y_t^{(i)} = \theta X^{8, (i)}_t$, and noisy realizations
$\widetilde{Y}_{t_1}^{(i)}, \ldots, \widetilde{Y}_{t_L}^{(i)}$, we
propose to compare the two data fits illustrated in
Figure~\ref{fig:illustrative_plot}.
\begin{figure}[ht]
  \centering
  \includegraphics[width=\columnwidth]{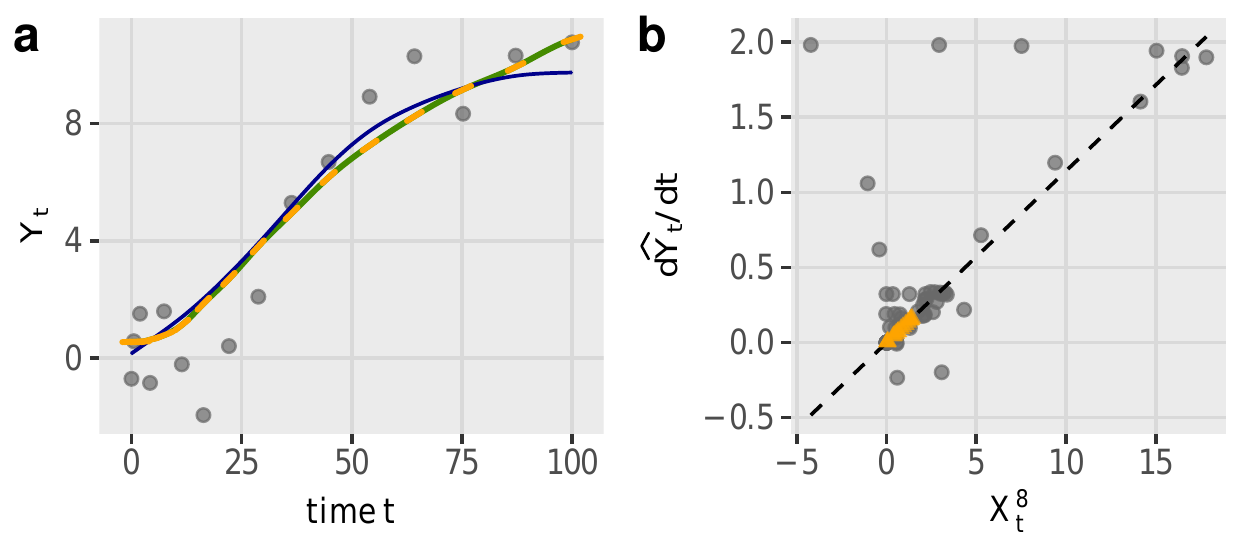}
  \caption{CausalKinetiX assigns a score to each model that trades-off
    predictability and invariance.  Here, we consider the model
    $dY_t/dt = \theta X^8_t$.  \textbf{a},~For each realization, two
    data fits are considered. An entirely data-driven nonlinear
    smoother (\emph{Data Fit A}, blue) is compared against a model
    based fit (\emph{Data Fit B}, green) with constraints on the derivatives
    (orange lines).  \textbf{b},~The derivative constraints are
    obtained from all other experiments: they correspond to fitted
    values (orange triangles) in a regression of the estimated
    derivatives on the predictors.}
  \label{fig:illustrative_plot}
\end{figure}
\emph{Data Fit A} calculates a smoothing spline to the data using all
realization from the same experiment, see
[\ref{eq:smooth_trajectory}].  This fit only serves as a baseline for
comparison: it does not incorporate the form of the underlying kinetic
model, but is entirely data-driven.  To obtain \emph{Data Fit B}, we
fit the considered model, $\frac{d}{dt}Y_t = \theta X^8_t$, on the
data from all other experiments (explicitly leaving out the current
experiment) by regressing estimated derivatives on the predictor
variables. In this example, the model is linear in its parameters, and
it therefore suffices to use linear regression. \emph{Data fit~B} fits
a smoothing spline to the same data, subject to the constraint that
its derivatives coincide with the fitted values from the regression
inferred solely based on the other experiments, see
[\ref{eq:smooth_trajectory2}]. \emph{Data fits A} and \emph{B} are
compared by considering the goodness-of-fit for each realization
$i=1, \ldots, n$.  More specifically, each model $M\in\mathcal{M}$
obtains, similar in spirit to \cite{lim2016}, the non-invariance score
\begin{equation*}
  T(M)
  \coloneqq\frac{1}{n}\sum_{i=1}^n\left[{\operatorname{RSS}_{B}^{(i)}-\operatorname{RSS}_{A}^{(i)}}\right]
  /\left[\operatorname{RSS}_{A}^{(i)}\right],
\end{equation*}
where
$\operatorname{RSS}_{*}^{(i)}\coloneqq
\frac{1}{L}\sum_{\ell=1}^L(\hat{y}^{(i)}_{*}(t_{\ell})-\widetilde{Y}^{(i)}_{t_{\ell}})^2$
is the residual sum of squares based on the respective data fits
$\hat{y}_A^{(i)}$ and $\hat{y}_B^{(i)}$. Due to the additional
constraints, $\operatorname{RSS}_{B}$ is always larger than
$\operatorname{RSS}_{A}$.

The score is large either if the considered model does not fit the
data well or if the model's coefficients differ between the
experiments.  Models with a small score are predictive and invariant.
These are models that can be expected to perform well in novel,
previously unseen experiments.  Models that receive a small residual
sum of squares, e.g., because they overfit, do not necessarily have a
small score $T$. We will see in the experimental section that such
models may not generalize as well to unseen experiments.  This assumes
that~[\ref{eq:invariance}] holds (approximately) when including the
unseen experiments, too.  Naturally, if the unseen experiments may
differ arbitrarily from the training experiments, neither
CausalKinetiX nor any other method will be able to generalize between
experiments.

The score $T$ can be used to rank models. We prove mathematically that
with an increasing number of realizations and a finer time resolution,
truly invariant models will indeed receive a higher rank than
non-invariant models (Methods).

\paragraph{Stable variable ranking procedure.}
In biological applications, modeling kinetic systems is a common
approach that is used to generate hypotheses related to causal
relationships between specific variables, e.g., to find species
involved in the regulation of a target protein. The non-invariance
score can be used to construct a stability ranking of individual
variables. The ranking we propose is similar to Bayesian model
averaging (BMA) \cite{hoeting1999} and is based on how often each
variable appears in a top ranked model. The key advantage of such a
ranking is that it leverages information from several fits leading to
an informative ranking.  It also allows testing whether a specific
variable is ranked significantly higher than would be expected from a
random ranking (Methods). Moreover, we provide a theoretical guarantee
under which the top ranked variables are indeed contained in the true
causal model (Methods).

We compare the performance of this ranking on a simulation study based
on a biological ODE system from the \emph{BioModels Database}
\cite{li2010biomodels} which describes reactions in heated
monosaccharide-casein systems (SI~4).  (In fact, the example in
Figure~\ref{fig:illustrative_plot} comes from this model, with $Y$ and
$X^8$ being the concentrations of Melanoidin and AMP, respectively.)
We compare our method to dynamic Bayesian networks \cite{Koller09}
based on conditional independence (DBN-CondInd), gradient matching
(GM) and an integrated version thereof, which from now on we refer to
as difference matching (DM); the last two methods both use $\ell_1$
penalization for regularization (SI~4). Figure~\ref{fig:roc_plot}
\textbf{a} shows median receiver operator curves (ROCs) for recovering
the correct causal parents based on $500$ simulations for all four
methods. CausalKinetiX has the fastest recovery rate and, in more than
50\% of the cases, it is able to recover all causal parents without
making any false discoveries, see Figure~\ref{fig:roc_plot}
\textbf{c}.  The recovery of the causal parents as a function of noise
level is given in Figure~\ref{fig:roc_plot} \textbf{b}. On the x-axis,
we plot the relative size of the noise, where a value of $1$ implies
that the size of the noise is on the same level as the target dynamic
and the signal is very weak. For all noise levels, CausalKinetiX is
better at recovering the correct model than all competing
methods. More comparisons can be found in SI~4.

\begin{figure}[ht]
  \centering
  \includegraphics[width=\columnwidth]{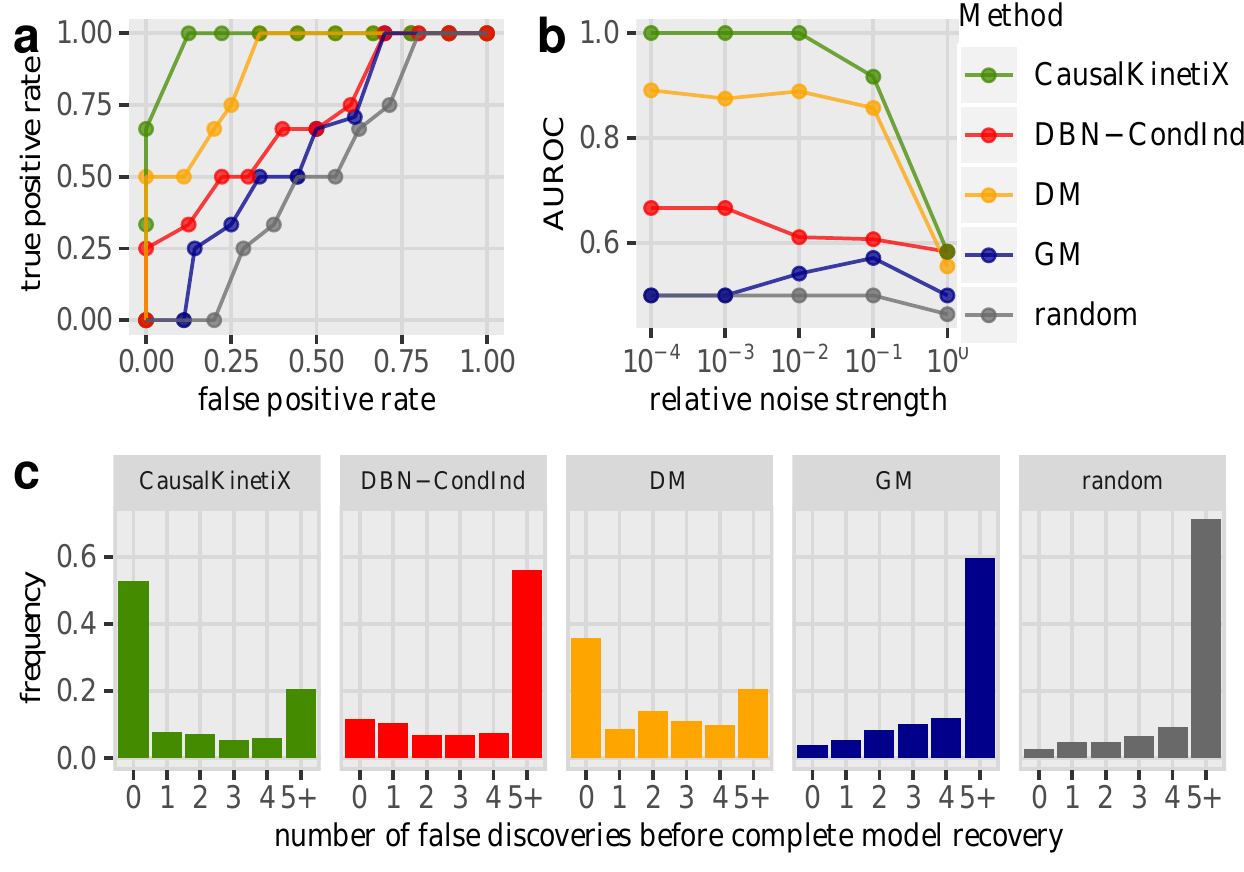}
  \caption{\textbf{a}, Median receiver operator curves (ROCs) for
    recovering the correct causal parents based on $500$
    simulations. CausalKinetiX has the fastest recovery rate.
    \textbf{b},~Median area under the receiver operator curve (AUROC)
    for different relative noise levels, CausalKinetiX outperforms all
    other methods. \textbf{c},~Number of recoveries before all correct
    variables enter the model. In the majority of cases, CausalKinetiX
    has no false discovery.}
  \label{fig:roc_plot}
\end{figure}

\paragraph{Numerical stability, scalability, and misspecified models.}
The method CausalKinetiX builds on standard statistical procedures,
such as smoothing, quadratic programming, and regression.  As opposed
to standard nonlinear least squares, it does not make use of any
numerical integration techniques. This avoids computational issues
that arise when the dynamics result in stiff systems
\cite{shampine2018numerical}. For each model, the runtime is less
than cubic in the sample size, which means that the key computational
cost is the exhaustive model search.  We propose to use a screening
step to reduce the number of possible models (SI~3~C), which allows
applying the method to systems with hundreds of variables (e.g.,
metabolic network below). Moreover, it does not require any
assumptions on the dynamics of the covariates.  In this sense, the
method is robust with respect to model misspecifications on the
covariates that can originate from hidden variables or misspecified
functional relationships.  Consistency of the proposed variable
ranking (SI~3), for example, only requires the model for the target
variable to be correctly specified.  Simulation experiments show that
this robustness can be observed empirically (SI~4). Finally, there is
empirical evidence that incorporating invariance can be interpreted as
regularization preventing overfitting, and that the method is robust
against correlated measurement error (SI~4).

\paragraph{Generalization in metabolic networks.}
We apply the proposed method to a real biological data set of a
metabolic network (Methods). Ion counts of one target variable and
cell concentrations of 411 metabolites are measured at 11 time points
across five different experimental conditions, each of which contains
three biological replicates. The experiments include both up- and
downshifts of the target variable, i.e., some of the conditions induce
an increase of the target trajectory, compared to its starting value,
other conditions induce a decrease.

We compare CausalKinetiX with the performance of nonlinear least
squares (NONLSQ). To make the methods feasible to such a large
dataset, we combine them with a screening based on DM.  We thus call
the method based on nonlinear least squares DM-NONLSQ; its parameters
are estimated using the software Data2Dynamics
(d2d)~\cite{data2dynamics}, which uses CVODES of the SUNDIALS suite
\cite{Hindmarsh2005} for numerical integration.
\begin{figure}
  \includegraphics[width=\columnwidth]{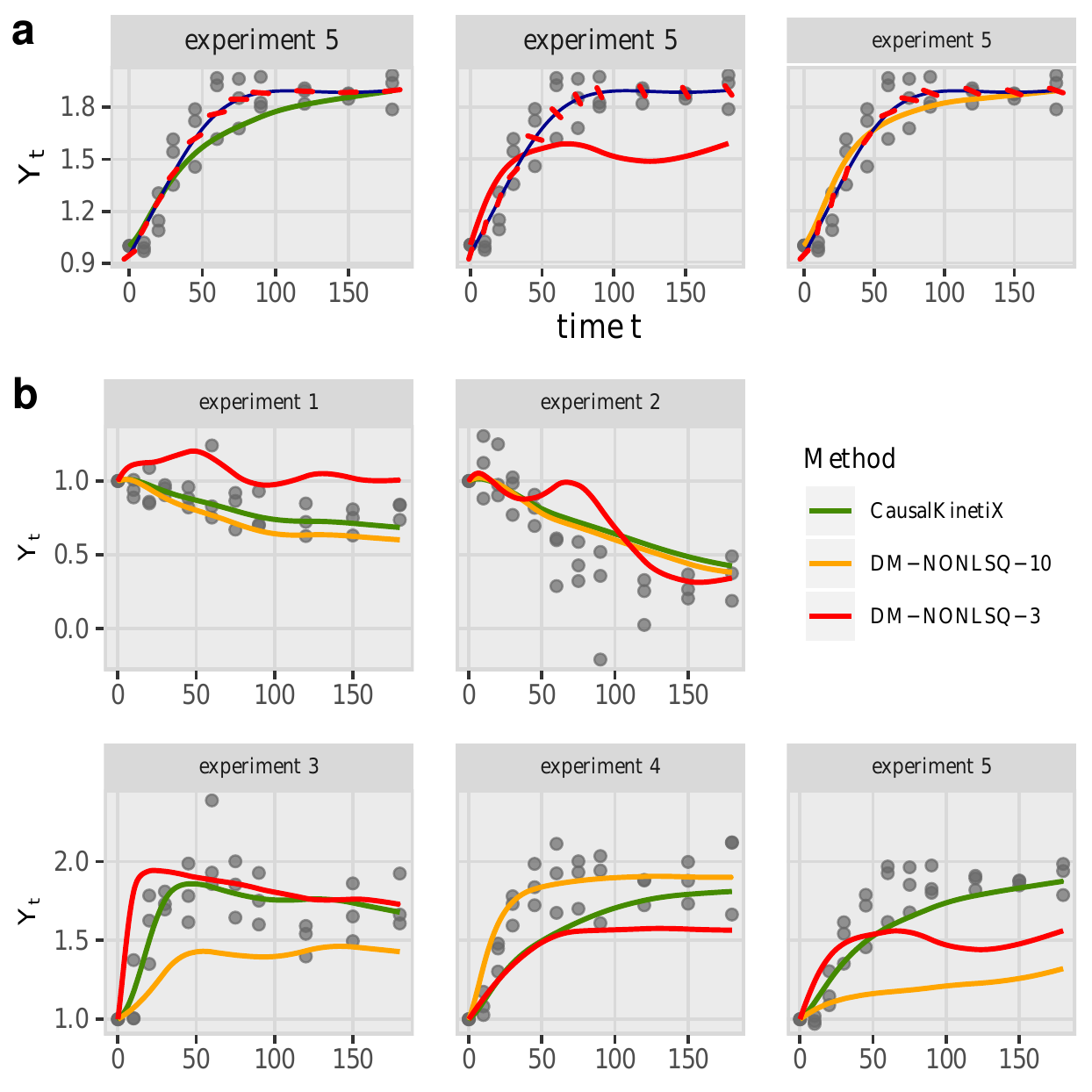}
  \caption{Metabolic network analysis. \textbf{a}, In-sample fit.  All
    five experiments are used for model selection and parameter
    estimation.  The plot shows model-based trajectories (numerically
    integrated) for experiment 5.  DM with 10 terms (right) fits the
    data better than CausalKinetiX (left) or DM with 3 terms (middle).
    \textbf{b},~Out-of-sample fit.  The plot shows the models' ability
    to generalize to new experiments.  Each plot shows model-based
    trajectories that are obtained when that experiment is not used
    for parameter estimation.  CausalKinetiX shows the best
    generalization performance.  The large DM model does not
    generalize well to unseen data due to overfitting.}
  \label{fig:real_data}
\end{figure}
Figure~\ref{fig:real_data} (a) shows the models' ability to describe
the dynamics in the observed experiments (in-sample
performance). DM-NONLSQ directly optimizes the residual sum of squares
(RSS) and therefore fits the data better than CausalKinetiX, which
takes into account stability, as well. The RSS for DM-NONLSQ-10 (based
on 10 terms) is lower ($0.83$) compared to CausalKinetiX ($0.96$)
averaged over all in-sample experiments. The plot contains
diagnostics for analyzing kinetic models.  The integrated dynamics are
shown jointly with a smoother (blue) through the observations (grey).
At the observed time points, the predicted derivatives (red lines) are
also shown using smoothed $X$ and $Y$ values. Model fits that explain
the data reasonably well in the sense that the integrated trajectory
is not far from the observations, may predict derivatives (small red
lines) on the smoother that do not agree well with the data: they fail
to explain the underlying dynamics.  For an example, see the
in-sample-fit of DM-NONLSQ-10 in Figure~\ref{fig:real_data}~(a). We
regard plotting a smoothing spline and the predicted derivatives for
the fitted values as a highly informative tool when analyzing models
for kinetic systems.

Pooling data across heterogeneous experiments, as, for example, done
by DM-NONLSQ, is already a natural regularization technique; if there
is sufficient heterogeneity in the data, the causal model is the only
invariant model. Finitely many experiments, however, only exhibit
limited heterogeneity and one can benefit from focusing specifically
on invariant models.  To compare the out-of-sample performance of the
methods, we consider the best ranked model from
Figure~\ref{fig:real_data}~\textbf{a}, hold out one experiment, fit
the parameters on the remaining four experiments and predict the
dynamics on the held out experiment. While DM-NONLSQ-10 explains the
observations well in-sample, it does not generalize to the held out
experiments. Neither does DM-NONLSQ-3 (based only $3$ terms) which
avoids overfitting. The average RSS of the held-out experiments are
$1.41$, $2.95$, and $3.45$ for CausalKinetiX, DM-NONLSQ-10, and
DM-NONLSQ-3, respectively; see
Figure~\ref{fig:real_data}~\textbf{b}. Another comparison, when the
methods are fully agnostic about one of the experimental conditions is
provided in SI~4.  By trading off invariance and predictability,
CausalKinetiX yields models that perform well on unseen experiments
that have not been used for parameter estimation.

\section*{Discussion}

In the natural sciences, differential equation modeling is a
widely-used tool for describing kinetic systems.  The discovery and
verification of such models from data has become a fundamental
challenge of science today.  Existing methods are often based on
standard model selection techniques or various types of sparsity
enforcing regularization; they usually focus on predictive
performance, and sometimes consider stability with respect to
resampling \citep{meinshausen2010stability, basu2018iterative}. In
this work, we develop novel methodology for structure search in
ordinary differential equation models. Exploiting ideas from causal
inference, we propose to rank models not only by their predictive
performance, but also by taking into account invariance, i.e., their
ability to predict well in different experimental settings.  Based on
this model ranking, we construct a ranking of individual variables
reflecting causal importance. It provides researchers with a list of
promising candidate variables that may be investigated further by
performing interventional experiments, for example.  Our ranking
methodology (both for models and variables) comes with theoretical
asymptotic guarantees and with a clear statement of the required
assumptions. Extensive experimental evaluation on simulated data shows
that our method is able to outperform current state-of-the art
methods.  Practical applicability of the procedure is further
illustrated on a not yet published biological data set. Our
implementation is readily available as an open-source R-package.

The principle of searching for invariant models opens up a promising
direction for learning causal structure from realistic, heterogeneous
datasets. The proposed CausalKinetiX framework is flexible in that it
can be combined with a wide range of dynamical models and any
parameter inference method. 
This is particularly relevant when the
differential equations depend nonlinearly on the parameters. Future
extensions may further include the extension to stochastic, partial
and delay differential equations and the transfer to other areas of
application like robotics, climate sciences, neuroscience and
economics.

\section*{Methods}
In this section, we provide additional details about data format,
methodology and the experiments. We further prove that our method is
statistically consistent, i.e., it infers correct models when the
sample size grows towards infinity.

\subsection*{Data format}
The data consist of $n$ repetitions of discrete time observations of
the $d$ variables $\B{X}$ (or their noisy version $\widetilde{\B{X}}$)
on the time grid $\B{t}=(t_1, \ldots, t_L)$. Each of the repetitions
is part of an experiment $\{e_1, \ldots, e_m\}$.  The experiments
should be thought of as different states of the system and may stem
from different interventions. One of the variables, $X^1$, say, is
considered as the target, we write $Y_t^{1,(i)} = X_t^{1,i}$.  We
further assume an underlying dynamical model (which then results in
various statistical dependencies between the variables and different
time points).

\subsection*{Mass-action kinetic models}
Many ODE based systems in biology are described by the law of
mass-action kinetics. The resulting ODE models are linear combinations
of various orders of interactions between the predictor variables
$\B{X}$. Assuming that the underlying ODE model of our target
$Y = X^1$ is described by a version of the mass-action kinetic law,
the derivative $\dYshort{t} := \frac{d}{dt} Y_t$ equals
\begin{equation}\label{eq:ydotmassaction}
  \dYshort{t} = 
  g_{\theta}(\B{X}_t) = \sum_{k=1}^d\theta_{0,k}X^{k}_t + \sum_{j=1}^d\sum_{k=j}^d\theta_{j,k}X^{j}_tX^{k}_t,
\end{equation}
where
$\theta=(\theta_{0,1},\dots,\theta_{0,d},\theta_{1,1},\theta_{1,2},\dots,\theta_{d,d})\in\R^{{d(d+1)}/{2}
  + d}$ is a parameter vector. We denote the subclass of all such
linear models of degree $1$ consisting of at most $p$ terms (i.e., $p$
non-zero terms in the parameter vector $\theta$) by
$\mathcal{M}^{\operatorname{Exhaustive}}_p$ and call these models
exhaustive linear models of degree $1$. A more detailed overview of
different collections of models is given in SI~2.

\subsection*{Model scoring}
For each model $M$ the score $T = T(M)$ is computed using the
following steps.  They include fitting two models to the data: one in
(M3) and the other one in (M4) and (M5).
\begin{compactenum}[(M1)]
\item\label{it:input} \emph{Input:} Data as described above and a
  collection $\mathcal{M} = \{M^1, M^2, \ldots, M^m\}$ of models over
  $d$ variables that is assumed to be rich enough to describe the
  desired kinetics.  In the case of mass-action kinetics, e.g.,
  $\mathcal{M} = \mathcal{M}^{\operatorname{Exhaustive}}_p$.
\item\label{it:modelscreening} \emph{Screening of predictor terms
    (optional):} For large systems, reduce the search space to less
  predictor terms. Essentially, any variable reduction technique based
  on the regression in step \ref{it:model2} can be used. We propose
  using $\ell_1$-penalized regression (SI~2).
\item\label{it:model1} \emph{Smooth target trajectories:} For each
  repetition $i\in\{1, \ldots, n\}$, smooth the (noisy) data
  $\widetilde{Y}^{(i)}_{t_1},\dots,\widetilde{Y}^{(i)}_{t_L}$ using a
  smoothing spline
  \begin{equation} \label{eq:smooth_trajectory} \hat{y}^{(i)}_{a} :=
    \argmin_{y\in\mathcal{H}_{C}}\, \sum_{\ell = 1}^L
    \big(\widetilde{Y}^{(i)}_{t_\ell} - y(t_\ell)\big)^2 + \lambda
    \int \ddot{y}(s)^2\, ds,
  \end{equation}
  where $\lambda$ is a regularization parameter, which in practice is
  chosen using cross-validation; $\mathcal{H}_{C}$ contains all smooth
  functions $[0,T] \rightarrow \mathbb{R}$, for which values and first
  two derivatives are bounded in absolute value by $C$.  We denote the
  resulting functions by $\hat{y}^{(i)}_{a}:[0,T]\rightarrow\R$,
  $i \in \{1, \ldots, n\}$. For each of the $m$ candidate target
  models $M\in\mathcal{M}$ perform steps
  \ref{it:model2}--\ref{it:model4}.
\item\label{it:model2} \emph{Fit candidate target model:} For every
  $i\in\{1,\dots,n\}$, find the function $g^i\in\mathcal{G}$ s.t.\
  \begin{equation}
    \label{eq:model_fit}
    \dY{k}{t} = g^i\big(\B{X}^{(k)}_t\big)
  \end{equation}
  is satisfied as well as possible for all $t\in\B{t}$ and for all
  repetitions $k$ belonging to a different experiment than repetition
  $i$. Below, we describe two procedures for this estimation step
  resulting in estimates $\hat{g}^i$. For each repetition
  $i\in\{1, \ldots, n\}$, this yields $L$ fitted values
  $\hat{g}^i(\widetilde{\B{X}}^{(i)}_{t_1}),\dots,\hat{g}^i(\widetilde{\B{X}}^{(i)}_{t_L})$.
  Leaving out the experiment of repetition $i$ ensures that only an
  invariant model leads to a good fit, as these predicted derivatives
  are only reasonable if the dynamics generalize across experiments.
\item\label{it:model3} \emph{Smooth target trajectories with
    derivative constraint:} Refit the target trajectories for each
  repetition $i\in\{1, \ldots, n\}$ by constraining the smoother to
  these derivatives, i.e., find the functions
  $\hat{y}^{(i)}_{b}:[0,T]\rightarrow\R$ which minimize
  \begin{align}
    \label{eq:smooth_trajectory2}
    \begin{split}
      \hat{y}^{(i)}_{b}\coloneqq &\argmin_{y\in\mathcal{H}_C}
      \sum_{\ell = 1}^L \big(\widetilde{Y}^{(i)}_{t_\ell} - y(t_\ell)\big)^2 + \lambda \int \ddot{y}(s)^2\, ds,\\
      &\text{such that}\quad \dot{y}(t_{\ell}) =
      \hat{g}^i(\widetilde{\B{X}}^{(i)}_{t_{\ell}}) \text{ for all }
      {\ell} = 1, \ldots, L.
    \end{split}
  \end{align} 
\item\label{it:model4} \emph{Compute score:} If the candidate model
  $M$ allows for an invariant fit, the fitted values
  $\hat{g}^i(\widetilde{\B{X}}^{(i)}_1),\dots,\hat{g}^i(\widetilde{\B{X}}^{(i)}_L)$
  computed in \ref{it:model2} will be reasonable estimates of the
  derivatives $\dY{i}{t_1},\dots,\dY{i}{t_L}$. This, in particular,
  means that the constrained fit in \ref{it:model3} will be good,
  too. If, conversely, the candidate model $M$ does not allow for an
  invariant fit, the estimates produced in \ref{it:model2} will be
  poor. We thus score the models by comparing the fitted trajectories
  $\hat{y}^{(i)}_{a}$ and $\hat{y}^{(i)}_{b}$ across repetitions as
  follows
  \begin{equation}
    \label{eq:scoreTS2}
    T(M)\coloneqq\frac{1}{n}\sum_{i=1}^n\left[\operatorname{RSS}_{b}^{(i)}-\operatorname{RSS}_{a}^{(i)}\right]
    /\left[\operatorname{RSS}_{a}^{(i)}\right],
  \end{equation}
  where
  $\operatorname{RSS}_{*}^{(i)}\coloneqq
  \frac{1}{L}\sum_{\ell=1}^L\big(\hat{y}^{(i)}_{*}(t_{\ell})-\widetilde{Y}^{(i)}_{t_{\ell}}\big)^2$.
  If there is a reason to believe that the observational noise has
  similar variances across experiments the division in the score can
  be removed to improve numerical stability.
\end{compactenum}
The scores $T(M)$ induce a ranking on the models $M\in\mathcal{M}$,
where models with a smaller score have more stable fits than models
with larger scores. Below, we show consistency of the model ranking.

\subsection*{Variable ranking}
The following method ranks individual variables according to their
importance in obtaining invariant models.  We score all models in the
collection $\mathcal{M}$ based on their stability score $T(M)$ (see
[\ref{eq:scoreTS2}]) and then rank the variables according to how many
of the top ranked models depend on them. This can be summarized in the
following steps.
\begin{compactenum}[(V1)]
\item \emph{Input:} same as in~\ref{it:input}.
\item\label{it:var1} \emph{Compute stabilities:} For each model
  $M\in\mathcal{M}$ compute the non-invariance score $T(M)$ as
  described in [\ref{eq:scoreTS2}]. Denote by $M_{(1)},\dots,M_{(K)}$
  the $K$ top ranked models, where $K\in\N$ is chosen to be the number
  of expected invariant models in $\mathcal{M}$.
\item\label{it:var2} \emph{Score variables:} For each variable
  $j\in\{1,\dots,d\}$, compute the following score
  \begin{equation}
    \label{eq:variable_score}
    s_j\coloneqq\frac{|\{k\in\{1,\dots,K\}\,\vert\,
      M_{(k)}\text{ depends on }j \}|}{K}.
  \end{equation}
  Here, ``$M_{(k)}\text{ depends on }j$'' means that the variable $j$
  has an effect in the model $M_{(k)}$ (SI~2). If there are $K$
  invariant models, the above score represents the fraction of
  invariant models that depend on variable $j$. It equals $1$ for
  variable $j$ if and only if every invariant model depends on that
  variable.
\end{compactenum}
These scores $s^j$ are similar to what is referred to as inclusion
probabilities in Bayesian model averaging \cite{hoeting1999}. Below,
we construct hypothesis tests for the test whether a score is
significantly higher than if the models are ranked randomly.
  
A natural choice for the parameter $K$ should equal the number of
invariant models. This may be unknown in practice, but our empirical
studies found that the method's results are robust to the choice of
$K$.  In particular, we propose to choose a small $K$ to ensure that
it is smaller than the number of invariant models (SI~2).

\subsection*{Fitting target models (M4)}
In step \ref{it:model2}, for every $i\in\{1,\dots,n\}$, we perform
a regression to find a function $g^i\in M$ such that
[\ref{eq:model_fit}] is optimized across all repetitions $k$ belonging
to different experiments than $i$. This task is difficult for two
reasons. First, the derivative values $\dY{k}{t}$ are not directly
observed and, second, even if we had access to (noisy and unbiased
versions of) $\dY{k}{t}$, we are dealing with an error-in-variables
problem.  Nevertheless, for certain model classes it is possible to
perform this estimation consistently and since the predictions are
only used as constraints, one expects estimates to work as long as
they preserve the general dynamics. We propose two procedures: (i) a
general method that can be adapted to many model classes and (ii) a
method that performs better but assumes the target model to be linear
in parameters.

The first procedure estimates the derivatives and then performs a
regression based on the model class under consideration.  That is, one
fits the smoother $y_{a}^{(k)}$ from \ref{it:model1} and then
computes its derivatives. When using the first derivative of a
smoothing spline it has been argued that the penalty term in
[\ref{eq:smooth_trajectory}] contains the third rather than the second
derivative of $y$ \cite{RamsayBook}. We then regress the estimated
derivatives on the data.  As a regression procedure, one can use
ordinary least squares if the models are linear or random forests, for
example, if the functions are highly nonlinear.

The second method works for models that are linear in the parameters,
i.e., for models that consist of functions of the form
$ g_{\theta}(\B{x})=\sum_{j=1}^{p}\theta_jg_j(\B{x})$, where the
functions $g_1,\dots,g_p$ are known transformations.  This yields
\begin{equation*}
  Y_{t_{\ell}}^{(k)}-Y_{t_{\ell-1}}^{(k)}=\sum_{j=1}^p\theta_j\int_{t_{\ell-1}}^{t_{\ell}}g_j(\widetilde{\B{X}}_s^{(k)})ds.
\end{equation*}
This approach does not require estimation of the derivatives of $Y$
but instead uses the integral of the predictors. It is well-known
that integration is numerically more stable than differentiation \cite{chen2017network}.
Often, it suffices to approximate the integrals using the trapezoidal
rule, i.e.,
\begin{equation*}
  \int_{t_{\ell-1}}^{t_{\ell}}g_j(\widetilde{\B{X}}_s^{(k)})ds\approx\frac{g_j(\widetilde{\B{X}}_{t_{\ell}}^{(k)})+g_j(\widetilde{\B{X}}_{t_{\ell-1}}^{(k)})}{2}(t_{\ell}-t_{\ell-1}),
\end{equation*}
since the noise in the predictors is often stronger than the error
in this approximation. The resulting bias is then negligible.

As mentioned above, most regression procedures have difficulties
with errors-in-variables and therefore return biased
results. Sometimes it can therefore be helpful to use smoothing or
averaging of the predictors to reduce the impact of this
problem. Our procedure is flexible in the sense that other fitting
procedures, e.g., inspired by \cite{ramsay2007parameter, Oates2014,
  calderhead2009accelerating}, could be applied, too.

\subsection*{Experiment on metabolic network}
Defining the auxiliary variable $Z_t := 2 - Y_t$, we expect that the
target species $Y_t$ and $Z_t$ are tightly related:
$Y_t \rightleftharpoons Z_t$, i.e., $Y_t$ is formed into $Z_t$ and
vice versa.  We therefore expect models of the form
\begin{align*}
  \dot{Y}_t &= \theta_1 Z_t X_t^j X_t^k + \theta_2 Z_t X_t^p X_t^q -
              \theta_3 Y_t X_t^r X_t^s\\
  \dot{Z}_t &= - \theta_1 Z_t X_t^j X_t^k - \theta_2 Z_t X_t^p X_t^q + \theta_3 Y_t X_t^r X_t^s,
\end{align*}
where $j,k,p,q,r,s \in \{1, \ldots, 411\}$ and
$\theta_1, \theta_2, \theta_3 \geq 0$.  By the conservation of mass
both target equations mirror themselves, which makes it sufficient to
only learn the model for $Y_t$. More precisely, we use the model class
consisting of three term models of the form $Z_t X_t^j X_t^k$,
$Y_t X_t^j X_t^k$, $Z_t X_t^j$, $Y_t X_t^j$, $Z_t$, or $Y_t$, where
the sign of the parameter is constrained to being positive or negative
depending on whether the term contains $Z_t$ or $Y_t$, respectively.
We constrain ourselves to three terms, as we found this to be the
smallest number of terms that results in sufficiently good in-sample
fits.  Given sufficient computational resources, one may include more
terms, too, of course. The sign constraint can be incorporated into
our method by performing a constrained least squares fit instead of
OLS in step~\ref{it:model2}. This constrained regression can then
be solved efficiently by a quadratic program with linear constraints.

As the biological data is high-dimensional, our method first screens
down to $100$ terms and then searches over all models consisting of
$3$ terms. To get more accurate fits of the dynamics, we pool and
smooth over the three biological replicates and only work with the
smoothed data.

\subsection*{Significance of variable ranking}
We can test whether a given score $s_j$, defined in
[\ref{eq:variable_score}], is significant in the sense that the number
of top ranked models depending on variable $j$ is higher than one
would expect if the ranking of all models in $\mathcal{M}$ was
random. More precisely, consider the null hypothesis
\begin{equation*}
  H_0:\quad
  \begin{array}{l}
    \text{the top ranked models } M_{(1)},\dots,M_{(K)} \\
    \text{are drawn uniformly from all models in
    }\mathcal{M}.
  \end{array}  
\end{equation*}
It is straightforward to show that under $H_0$ it holds that
$K\cdot s_j$ follows a hypergeometric distribution with parameters
$\abs{\mathcal{M}}$ (population size),
$\abs{\{M\in\mathcal{M}\,\vert\, M\text{ depends on } j\}}$ (number
success in population) and $K$ (number of draws). For each variable we
can hence compute a $p$-value to assess whether it is significantly
important for stability.

\subsection*{Theoretical consistency guarantees}
We prove that both the model ranking and the proposed variable ranking
satisfy theoretical consistency guarantees. More precisely, under
suitable conditions and in the asymptotic setting where both the
number of realizations $n$ and the number of time points $L$ converge
to infinity, every invariant model will be ranked higher than all
non-invariant models. Given sufficient heterogeneity of the
experiments it additionally holds that the variable score $s_j$
defined in [\ref{eq:variable_score}] tends to one if and only if
$j\in S^*$, see~[\ref{eq:invariance}]. Details and proofs are
provided in SI~3.
  
\subsection*{Relation to causality}
Causal models enable us to model a system's behavior not only in an
observational state, but also under interventions.  There are various
ways to define causal models \citep{Pearl2009, Imbens2015}.  The
concept of structural causal models is well-suited for the setting of
this paper and its formalism can be adapted to the case of dynamical
models (SI~1).  If the experimental settings correspond to different
interventions on variables other than $Y$, choosing $S^*$ as the set
of causal parents of $Y$ satisfies [\ref{eq:invariance}].  If the
settings are sufficiently informative, no other set
satisfies~[\ref{eq:invariance}].

\subsection*{Code and data availability}
Well-documented code is available as an open source R package on 
CRAN (\url{https://cran.r-project.org/web/packages/CausalKinetiX}). It 
includes the ODE models used in the simulations, e.g., the Maillard 
reaction. More details, e.g., about reproducing all figures and a ported 
python version of the package are available at 
\url{http://www.causalkinetix.org}.

\section*{Acknowledgements}
We thank R.~Loewith, B.~Ryback, U.~Sauer, E.~M.~Sayas and J.~Stelling
for providing the biological data set as well as helpful biological
insights. We further thank N.~R.~Hansen and N.~Meinshausen for helpful
discussions, and K.~Ishikawa and A.~Orvieto for their help with Python
and d2d. This research was partially supported by the Max Planck ETH
Center for Learning Systems and the SystemsX.ch project SignalX.
J.P.\ was supported by a research grant (18968) from VILLUM
FONDEN. N.P. was partially supported by the European Research
Commission grant 786461 CausalStats - ERC-2017-ADG.

\pagebreak
\bibliographystyle{plainnat}
\bibliography{references}

\pagebreak
\appendix
\input{appendix}

\end{document}

%% file: appendix.tex
{\centering\huge\textbf{Supplementary Information}}
\vspace{1cm}

The supporting material is divided into the following 4 parts:
\begin{itemize}
\item Supplementary information~\ref{sec:causal}: Causality
\item Supplementary information~\ref{sec:parametric_models}:
  CausalKinetiX for parametric dynamical models
\item Supplementary information~\ref{sec:consistency_results_and_proof}: Consistency result and proof
\item Supplementary information~\ref{sec:extended_simulations}:
  Extended empirical results
\end{itemize}

\section{Causality}\label{sec:causal}

This supplementary note contains details on causal models and the relationship between CausalKinetiX and causality.
It
accompanies the article \emph{Learning stable and predictive structures in kinetic
  systems}.

\subsection{Structural causal models} \label{sec:SCMs} We first revise
the widely used concept of structural causal models \citep{Pearl2009}
for i.i.d.\ data without any time structure.  A structural causal
model (SCM) over $d$ random variables $X^1, \ldots, X^d$ consists of
$d$ assignments
$$
X^j := f^i(X^{\PA[]{j}}, \varepsilon^j),  \quad j = 1, \ldots, d
$$
together with a noise distribution over the random vector
$(\varepsilon^1, \ldots, \varepsilon^d)$, which we may assume to
factorize, implying that the noise variables are independent.  Here,
$\PA[]{j} \subseteq \{1, \ldots, d\}$ are called the (causal) parents
of $j$.  For each SCM we obtain a corresponding graph over the
vertices\footnote{By slight abuse of notation, we identify
  $(X^1, \ldots, X^d)$ with its indices $(1, \ldots, d)$.}
$(1, \ldots, d)$ by drawing edges from $\PA[]{j}$ to $j$,
$j \in \{1, \ldots, d\}$.  The corresponding directed graph is assumed
to be acyclic.  Under these conditions, the SCM can be shown to entail
an observational distribution over $X^1, \ldots, X^d$.

An intervention now corresponds to replacing one of these assignments,
such that the resulting graph is again acyclic. This ensures that the
new SCM again entails a joint distribution over $X^1, \ldots, X^d$,
the intervention distribution.  If
$X^k := f^k(X^{\PA[]{k}}, \varepsilon^k)$ is replaced by $X^k := 4$,
for example, that distribution is sometimes denoted by $do(X^k := 4)$,
the $do$-notation indicating that a variable has been intervened on
(rather than conditioned on).

The above concept of SCMs can be straightforwardly adapted to
dynamical systems.  A \emph{causal kinetic model} over a process
$\B{X} := (\B{X}_t)_t := (X^1_t, \ldots, X^d_t)_t$ is a collection of
$d$ assignments
\begin{align*}
  \dot{X}_t^1 &:= f^1(X_t^{\PA[]{1}}), \qquad X^1_0 := x^1_0\\
  \dot{X}_t^2 &:= f^2(X_t^{\PA[]{2}}), \qquad X^2_0 := x^2_0\\
              & \; \; \; \vdots \\
  \dot{X}_t^d &:= f^d(X_t^{\PA[]{d}}), \qquad X^d_0 := x^d_0,
\end{align*}   
where $\dot{X}_t^j$ is short-hand notation for $\frac{d}{dt}{X}_t^j$.
As above, $\PA[]{j} \subseteq \{1, \ldots, d\}$ are called the direct
(or causal) parents of $X^j$. We explicitly allow for cycles in the
corresponding graph (in particular, a node might be its own parent),
but require the above system of differential equations to be solvable.
Furthermore, we may assume observations from the above system are
corrupted by observational noise, e.g.,
$\widetilde{\B{X}}_t = \B{X}_t + \boldsymbol{\eps}_t$.

As in the i.i.d.\ case, an intervention replaces one (or more) of the
above assignments. The change may include the initial values, the
differential equation or both at the same time. Formally, an
intervention replaces the $k$th assignment with
\begin{equation*}
  X^k_0 := \xi
  \quad\text{or}\quad
  \dot{X}^k_t := g(X_t^{\PA[]{}}, X^k_t),
\end{equation*}
with $\PA[]{}$ being the new set of parents of $j$.  In either case,
the system of differential equations is still required to be
solvable. In the presence of observational noise, the noise is assumed
to enter after the intervention. Similar frameworks have been
suggested \citep{MooijJS2013,Blom2018,Rubenstein2016}, which usually
focus more on the equilibrium state of the system.  The above
framework allows for a variety of interventions, such as setting a
variable $X^k$ to a constant $c$, for example, or ``pulling'' a
variable $X^k$ to a certain value $c$.  One can change the dependence
of a variable on its parents or change the parent set of $X^k$
completely.  In the case of mass-action kinetics, a change in reaction
rates can be modeled as an intervention, for example.

\subsection{Stability of causal models}
In this work, we assume that we are given a target variable $Y_t$ and
a set $\mathbf{X}_t$ of $d$ predictors.  Suppose that there is an
underlying causal kinetic model (see Section~\ref{sec:SCMs}) over the
process $(\mathbf{X}_t, Y_t)$.  If we consider experimental
conditions that correspond to different interventions on variables
other than $Y$, it follows that $S^* := \PA[]{Y}$ yields a stable
model, i.e.,
$$\dY{i}{t} = f\big({X}^{S^*,(i)}_t\big), \, \text{ for all } i = 1, \ldots, n,
$$
see~[1] in the main article and \ref{sec:parametric_models}
below.  This principle is known as autonomy or modularity, see
\cite{Aldrich1989, Haavelmo1944, Pearl2009,ScholkopfJPSZMJ2012}, for
example.

\subsection{Causality through stability}
The above mentioned stability property of the parents of $Y$ can be
exploited for causal discovery.  Suppose that there is an underlying
causal kinetic model, whose structure is unknown.  Suppose further,
that we are given data from different experimental settings and that
the parents of $Y$ satisfy the above stability.  We can then search
for all sets $S \subseteq \{1, \ldots, d\}$, such that
$\dY{i}{t} = f({X}^{S,(i)}_t)$, for all $i = 1, \ldots, n$.  There can
be more than one invariant (or stable) set, but by assumption the set
of parents is one of these sets. It has been proposed to output the
intersection of all invariant sets, which is guaranteed to be a subset
of the causal parents \citep{Meinshausen2016pnas, Peters2016jrssb}.
The intersection itself, however, does not necessarily lead to a model
with good predictive performance. In CausalKinetiX, we propose to use
a trade-off between stability and predictability.  This yields models
that generalize well to unseen experimental conditions. If there are
sufficiently many and informative experimental conditions, the causal
parents of $Y$ are the unique set satisfying the invariance condition.
(This is, trivially, the case if all variables other than $Y$ are
intervened on, for example.)  Under such conditions, the consistency
results we develop in \ref{sec:consistency_results_and_proof} then
state that CausalKinetiX in the limit of infinite sample size indeed
recovers the set of causal parents of $Y$, see
Theorem~\ref{thm:rank_consistency} below.

\section{CausalKinetiX for parametric dynamical
  models}\label{sec:parametric_models}

This supplementary note contains details on the different types of
model classes that can be used in the CausalKinetiX
framework. Additionally, we give details on how to choose the tuning
parameter $K$ in the variable ranking and propose to explicit
screening procedures. It accompanies the article \emph{Learning stable
  and predictive structures in kinetic systems}.

\subsection{Connection between models and variables}
Given a differential equation $\dot{Y}_{t} = g(\B{X}_t) $ it will be
convenient to speak about the arguments $\B{X}_t$ of the function $g$
that have an influence on the outcome $\dot{Y}_{t}$. For any set
$S\subseteq\{1,\dots,d\}$, we therefore define
\begin{equation*}
  \mathcal{F}(S)\coloneqq\left\{g:\R^d\rightarrow\R \,\big\vert\,
    \exists f:\R^{\abs{S}}\rightarrow\R:
    \forall \B{x}\in\R^d \; g(\B{x}) = f(\B{x}^S) \right\}.
\end{equation*}
A set of functions $M \subseteq \mathcal{F}(S)$ (also referred to as
model), then contains only functions that do not depend on variables
outside $S$. In the class of mass-action kinetics, for example, we
could have
\begin{align*}
M_1 &= \{g\,|\,g(x) = 
\theta_1 x_1 +
\theta_7 x_7,
\text{ where }
\theta_1, \theta_7 \in \mathbb{R}\},\\
M_2 &= \{g\,|\,g(x) = 
\theta_2 x_2 +
\theta_3 x_3,
\text{ where }
\theta_2, \theta_3 \in \mathbb{R}\}, 
\end{align*}
which implies $M_1 \subseteq \mathcal{F}(\{1,7\})$ and
$M_2 \subseteq \mathcal{F}(\{2,3\})$, see
Section~\ref{sec:mass-action} for more details.  In practice, the
underlying structure is unknown and many methods therefore include a
model selection step. Here, we assume that we are given a family of
target models $\mathcal{M}=\{M^1,\dots,M^m\}$, where individual models
can depend on various different subsets of variables
$S\subseteq\{1,\dots,d\}$. We say a model $M\text{ depends on }j$ if
there is no $S \subseteq \{1, \ldots, d\}$ with $j \notin S$ and
$M\subseteq \mathcal{F}(S)$. Furthermore, we refer to $M$ as a
\emph{target model} and to $\mathcal{M}$ as a \emph{collection of target
  models}. Based on these definitions, we can make precise what it
means for the target trajectories $Y$ to be described by invariant
dynamics.
\begin{assumption}[invariance]\label{ass:inv0}
  There exists a set $S^*$ and a function
  $f^*: \mathbb{R}^{|S^*|} \rightarrow \mathbb{R}$ satisfying for all
  $i\in\{1, \ldots, n\}$ and all $t\in\B{t}$ that
  \begin{equation}
    \label{eq:ODEYi}
    \dY{i}{t} = f^*\big(\B{X}^{S^*,(i)}_t\big).
  \end{equation}
  Further, $S^*$ is minimal for $f^*$ in the following sense: there is
  no $S \subsetneq S^*$ such that $f^* \in \mathcal{F}(S)$.
\end{assumption}
For causal kinetic models, the pair $(f^Y, \PA[]{Y})$ satisfies
Assumption~\ref{ass:inv0} whenever the environments consist of
interventional data, which do not contain interventions on the target
$Y$ itself. Even if Assumption~\ref{ass:inv0} is satisfied the pair
$(f^*, S^*)$ is not necessarily unique, i.e., there may be one or
several pairs satisfying~[\ref{eq:ODEYi}]. In general, both the set
$S^*$ as well as the invariant function $f^*$ are of interest in
practice as both are strongly related to the causal mechanisms of the
underlying system.

\subsection{Parametric models for mass-action
  kinetics}\label{sec:mass-action}
Many ODE based systems in biology are described by the law of
mass-action kinetics. The resulting ODE models are linear
combinations of various orders of interactions between the predictor
variables $\B{X}$. Assuming that the underlying ODE model of our
target $Y$ is described by a version of the mass-action kinetic law,
the derivative $\dYshort{t}$ equals
\begin{equation} \label{eq:ydotmassaction}
\dYshort{t} = g_{\theta}(\B{X}_t) = \sum_{k=1}^d\theta_{0,k}X^{k}_t + \sum_{j=1}^d\sum_{k=j}^d\theta_{j,k}X^{j}_tX^{k}_t,
\end{equation}
where
$\theta=(\theta_{0,1},\dots,\theta_{0,d},\theta_{1,1},\theta_{1,2},\dots,\theta_{d,d})\in\R^{{d(d+1)}/{2}
  + d}$ is a parameter vector. Correspondingly, the function on the
right-hand side of [\ref{eq:ODEYi}] in Assumption~\ref{ass:inv0} has
such a parametric form, too.  The assumption that the model only
depends on the variables in $S^*$ can be expressed by a sparsity on
the parameter $\theta$, i.e., $\theta_{j,k}=0$ for all $j,k$ with
$j \notin S^*$ or $k \notin S^*$.  A target model $M$
can be constructed by a sparsity pattern.  Formally, such a
sparsity pattern is described by a vector
$v\in\{0,1\}^{{d(d+1)}/{2}+d}$ which specifies the zero entries in
$\theta$. For every such $v$, we define
\begin{equation*}
  M^v\coloneqq\Big\{g_{\theta}:\R^d\rightarrow\R \,\Big\vert\,
    \forall\B{x}\in\R^d:
    g_{\theta}(\B{x})=\sum_{k,j}\theta_{k,j}x^kx^j \text{ and } v*\theta=\theta\Big\},
\end{equation*}
where $*$ denotes the element-wise product.  In principle, one could
now search over all sparsity patterns of $\theta$, i.e., define
$\mathcal{M} = \{M^v, v\in\{0,1\}^{{d(d+1)}/{2}+d}\}$, but this
becomes computationally infeasible already for small values of $d$.
In this work, we suggest two different collections of target models.
Other choices, in particular those motivated by prior knowledge, are
possible, too, and can easily be included in our framework.

\paragraph{Exhaustive models.}
Using only the constraint on the number
of terms $p$ leads to the following collection of models
\begin{equation*}
  \mathcal{M}^{\operatorname{Exhaustive}}_p=\left\{M^v\,\big\vert\,
    v \text{ has at most } p \text{ non-zeros}\right\}.
\end{equation*}
Every model in $\mathcal{M}^{\operatorname{Exhaustive}}_p$ consists of
a linear combination of a fixed number of at most $p$ terms of the
form $X^1,\dots,X^d, X^1X^1,X^1X^2, \ldots, X^{d-1}X^d$ or $X^dX^d$.

\paragraph{Main effect models.} 
Alternatively, one can also add the restriction that the models including interaction terms for variables, include the corresponding main effects, too.
\begin{align*}
  \mathcal{M}^{\operatorname{MainEffect}}_p =\big\{M^v\,\big\vert\,
  &v \text{ has at most } p \text{ non-zeros and }\\
  &v_{0,j} 
    \neq 0 \text{ implies }
    v_{k,j} \neq 0
    \, \forall k < j 
    \text{ and } 
    v_{j,k} \neq 0
    \, \forall k \geq j
    \big\}.
\end{align*}
While the number of main effect models is much smaller it generally
requires to fit larger models, which can lead to overfitting. For
example, if the true invariant model only depends on the two terms
$X^1$ and $X^4X^5$ there exists a exhaustive model with two parameters
that is invariant, while the smallest main effect model has nine
parameters.

\subsection{Screening procedures and choosing parameter $K$ in
  variable selection}

\paragraph{Screening of terms (M2)}
For large scale applications, e.g., when $d$ is larger than $30$, the
computational complexity of the method can be significantly reduced by
including a prior screening step.  The collections of models from
mass-action kinetics scale as
\begin{equation*}
  \abs{\mathcal{M}^{\operatorname{Exhaustive}}_p}=\sum_{k=1}^p\binom{\tbinom{d}{2}}{k}=\mathcal{O}(d^{2p}).
\end{equation*}
Even though computation of the stability score for a single model is
fast, this shows that an exhaustive search is infeasible for settings
with large $d$.  We propose to reduce the model size by a screening
step that screens terms that are useful for prediction in the model
estimation in \ref{it:model2} then continue our procedure with the
reduced model class. This procedure allows the method to be applied in
high-dimensional settings.

Usually, the screening step will focus on predictability, and any
screening or variable selection method based on the model fit (see
[\ref{eq:model_fit}]) can be used. We provide two explicit options
based on $\ell^1$-penalized least squares, also known as Lasso
\cite{Tibshirani94}. The first method performs the regularization on
the level of the derivatives (gradient matching GM) and the second on
the integrated problem (difference matching DM), see SI~4,
Section~\ref{sec:competing_methods}. These two options can be used as
screening steps, even in combination with other methods, such as
nonlinear least squares (NONLSQ), or as parameter estimation methods
in themselves.  Even though we are not aware of any work that proposes
this precise implementation, the idea of using $\ell^1$-penalized
procedures for model inference has been used extensively
\cite{boninsegna2018sparse, brunton2016discovering, mikkelsen2017,
  rudy2017data, schaeffer2017learning, szederkenyi2011inference,
  tran2017exact, Wu2014}.

\paragraph{Choosing parameter $K$ in variable ranking (V2)}
Ideally, as mentioned in~\ref{it:var1}, $K$ should equal the number of
invariant models, see also the consistency results in SI~3.  In our
empirical studies, we found that the method's results are robust to
the choice of $K$.  In general, we propose to choose a small $K$ to
ensure that it is smaller than the number of invariant models.
Depending on the collection $\mathcal{M}$ of target models it is
sometimes possible to give a heuristic number of invariant models.
For example in the case of mass-action kinetics, we have the following
reasoning for the exhaustive models
$\mathcal{M}^{\operatorname{Exhaustive}}_{p+1}$. If the smallest
invariant model only has $p$ terms (i.e., it corresponds to a
$v^*\in V_{p+1}$ with $\sum_jv^*_j=p$), it follows that any
super-model (i.e., any $v\in V_{p+1}$ with $\sum_j\abs{v_j-v^*_j}=1$)
is also an invariant model.  The number of super-models, however, is
given by $2d+\binom{d}{2} - p$. Hence, if we use the model collection
$\mathcal{M}^{\operatorname{Exhaustive}}_{p+1}$, where $p$ is assumed
to be the expected number of terms contained in the smallest invariant
model, a reasonable choice is to set $K=2d+\binom{d}{2} - p$.

\section{Consistency result and proof}\label{sec:consistency_results_and_proof}

This supplementary note contains the precise formulation of the
consistency result mentioned in the article \emph{Learning stable and
  predictive structures in kinetic systems} as well as a detailed
proof. For completeness and to ensure better understanding, we first
recall all details required for a precise formulation of the result
(Section~\ref{sec:procedure} and Section~\ref{sec:consistency}) and
then present the detailed proof (Section~\ref{sec:proofs}).

\subsection{Detailed procedure}\label{sec:procedure}

In this section, we recall the step-wise procedures for model scoring
and variable ranking. Note, that the version of the model ranking
presented in Section~\ref{sec:stability_ranking} is slightly different
from the one in the main article. More specifically, step
\ref{it:model2} does not leave out experiments but rather pools across
all experiments. This allows for a simpler theoretical analysis but
neglects the additional stability. However, we believe that the theory
also holds for the more complicated leave-one-out procedure and
actually expect it to have better finite sample properties.

\subsubsection{Model scoring}\label{sec:stability_ranking}
For each model $M$ the score $T = T(M)$ is computed as follows. 
\begin{enumerate}[label=(M\arabic*)]
\item\label{it:input} \textbf{Input:} Data as described above and a collection
  $\mathcal{M} = \{M^1, M^2, \ldots,
  M^m\}$ of models over $d$ variables that is assumed to be
  rich enough to describe the desired kinetics.  In the case of
  mass-action kinetics, examples include
  $\mathcal{M} = \mathcal{M}^{\operatorname{MainEffect}}_p$ or
  $\mathcal{M} = \mathcal{M}^{\operatorname{Exhaustive}}_p$.
\item\label{it:modelscreening} \textbf{Screening of predictor terms
    (optional):} For large systems, reduce the search space to less
  predictor terms. This replaces the collection of models by a smaller
  collection, which we again denote by $\mathcal{M}$.
\item\label{it:model1} \textbf{Smooth target trajectories:} For each repetition
  $i\in\{1, \ldots, n\}$, smooth the (noisy) data
  $\widetilde{Y}^{(i)}_{t_1},\dots,\widetilde{Y}^{(i)}_{t_L}$ using a
  smoothing spline
  \begin{equation} \label{eq:smooth_trajectory}
    \hat{y}^{(i)}_{a} := \argmin_{y\in\mathcal{H}_{C}}\,
    \sum_{\ell = 1}^L \big(\widetilde{Y}^{(i)}_{t_\ell} - y(t_\ell)\big)^2 + \lambda \int \ddot{y}(s)^2\, ds,
  \end{equation}
  where $\lambda$ is a regularization parameter, which in practice is
  chosen using cross-validation; $\mathcal{H}_{C}$ contains all smooth
  functions $[0,T] \rightarrow \mathbb{R}$, for which values and first
  two derivatives are bounded in absolute value by $C$.  We denote the
  resulting functions by $\hat{y}^{(i)}_{a}:[0,T]\rightarrow\R$,
  $i \in \{1, \ldots, n\}$. For each of the $m$ candidate target
  models $M\in\mathcal{M}$ perform the steps
  \ref{it:model2}--\ref{it:model4}.
\item\label{it:model2} \textbf{Fit candidate target model:}  Fit the target model $M$, i.e., find the best fitting function
  $g\in M$ such that
  \begin{equation}
    \label{eq:model_fit}
    \dY{i}{t} = g\big(\B{X}^{(i)}_t\big),
  \end{equation}
  holds for all $i\in\{1,\dots,n\}$ and $t\in\B{t}$. Below, we
  describe two procedures for this estimation step resulting in an
  estimate $\hat{g}$. For each repetition $i\in\{1, \ldots, n\}$, this
  yields $L$ fitted values
  $\hat{g}(\widetilde{\B{X}}^{(i)}_{t_1}),\dots,\hat{g}(\widetilde{\B{X}}^{(i)}_{t_L})$.
\item\label{it:model3} \textbf{Smooth target trajectories with derivative constraint:} Refit the target trajectories for each repetition
  $i\in\{1, \ldots, n\}$ by constraining the smoother to these
  derivatives, i.e., find the functions
  $\hat{y}^{(i)}_{b}:[0,T]\rightarrow\R$ which minimize
  \begin{align}
    \label{eq:smooth_trajectory2}
    \begin{split}
      \hat{y}^{(i)}_{b}\coloneqq 
      &\argmin_{y\in\mathcal{H}_C}
      \sum_{\ell = 1}^L \big(\widetilde{Y}^{(i)}_{t_\ell} - y(t_\ell)\big)^2 + \lambda \int \ddot{y}(s)^2\, ds,\\
      &\text{such that}\quad
      \dot{y}(t_{\ell}) = \hat{g}(\widetilde{\B{X}}^{(i)}_{t_{\ell}}) \text{ for all } {\ell} = 1, \ldots, L.
    \end{split}
  \end{align} 
\item\label{it:model4} \textbf{Compute score:} If the candidate model $M$ allows for an
  invariant fit, the fitted values
  $\hat{g}(\widetilde{\B{X}}^{(i)}_1),\dots,\hat{g}(\widetilde{\B{X}}^{(i)}_L)$
 computed in \ref{it:model2} will be reasonable estimates of the
  derivatives $\dY{i}{t_1},\dots,\dY{i}{t_L}$. This, in particular,
  means that the constrained fit in \ref{it:model3} will be good, too. If, conversely, 
  the candidate model $M$
  does not allow for an invariant fit, the estimates
  produced in \ref{it:model2} will be poor. We thus score the models by
  comparing the fitted trajectories $\hat{y}^{(i)}_{a}$ and
  $\hat{y}^{(i)}_{b}$ across repetitions as follows
  \begin{equation}
    \label{eq:scoreTS2}
    T(M)\coloneqq\frac{1}{n}\sum_{i=1}^n\left[\abs{\operatorname{RSS}_{b}^{(i)}-\operatorname{RSS}_{a}^{(i)}}\right]
    /\left[\operatorname{RSS}_{a}^{(i)}\right],
  \end{equation}
  where $\operatorname{RSS}_{*}^{(i)}\coloneqq \frac{1}{L}\sum_{\ell=1}^L\big(\hat{y}^{(i)}_{*}(t_{\ell})-\widetilde{Y}^{(i)}_{t_{\ell}}\big)^2$.
\end{enumerate}
The scores $T(M)$ induce a ranking on the models in
$\mathcal{M}$, where models with a smaller score have more stable fits
than models with larger scores. Below, we show consistency of the
model ranking.

\subsubsection{Variable ranking}\label{sec:ranking_variables}
The following idea ranks individual variables according to their
importance in stabilizing models.  We score all models in the
collection $\mathcal{M}$ based on their stability as described above
and then rank the variables according to how many of the top ranked
models depend on them. This can be summarized in the following steps.
\begin{enumerate}[label=(V\arabic*)]
\item \textbf{Input:} same as in~\ref{it:input}.
\item\label{it:var1} \textbf{Compute stabilities:} For each model $M\in\mathcal{M}$ compute the stability
  score $T(M)$ as described in
  [\ref{eq:scoreTS2}]. Denote by $M_{(1)},\dots,M_{(K)}$ the
  $K$ top ranked models, where $K\in\N$ is chosen to be the number of
  expected invariant models in $\mathcal{M}$. See
 below 
 for how to choose $K$ in practice.
\item\label{it:var2} \textbf{Score variables:} For each variable $j\in\{1,\dots,d\}$ compute the following score
  \begin{equation}
    \label{eq:variable_score}
    s_j\coloneqq\frac{|\{k\in\{1,\dots,K\}\,\vert\,
      M_{(k)}\text{ depends on }j \}|}{K}.
  \end{equation}
  Here, ``$M_{(k)}\text{ depends on }j$'' means that there
  is no $S \subseteq \{1, \ldots, d\}$ with $j \notin S$ and
  $M_{(k)} \subseteq \mathcal{F}(S)$. If there are exactly
  $K$ invariant models, the above score then represents the fraction
  of invariant models that depend on variable $j$. In particular, it
  equals $1$ for a variable $j$ if and only if every invariant model
  depends on that variable.
\end{enumerate}

\subsection{Consistency of ranking procedure}\label{sec:consistency}
First, we recall the key concept of invariance, which will be an
important condition in order for our consistency result to hold.
\begin{assumption}[invariance]\label{ass:inv}
  There exists a set $S^*$ and a function
  $f^*: \mathbb{R}^{|S^*|} \rightarrow \mathbb{R}$ satisfying for all
  $i\in\{1, \ldots, n\}$ and all $t\in\B{t}$ that
  \begin{equation*}
    \dY{i}{t} = f^*\big(\B{X}^{S^*,(i)}_t\big).
  \end{equation*}
  Further, $S^*$ is minimal for $f^*$ in the following sense: there is
  no $S \subsetneq S^*$ such that $f^* \in \mathcal{F}(S)$.
\end{assumption}
Next, we provide conditions under which the proposed procedure is
consistent. To this end, we fix the number of environments (or
experiments) to $m$ and assume that there exist $R$ repetitions for
each experiment observed on $L$ time points. In total this means we
observe $n=m\cdot R$ trajectories, each on a grid of $L$ time points.
As asymptotics, consider a growing number of repetitions $R_n$ and
simultaneously a growing number of time points~$L_n$.  Here,
increasing repetitions $R$ and time points $L$ corresponds to
collecting more data and obtaining a finer time resolution,
respectively. Both of these effects are achieved by novel data
collection procedures. To make this more precise, assume that for each
$n\in\N$, we are given a time grid
\begin{equation*}
  \B{t}_n = (t_{n,1}, \ldots, t_{n,L_n})
\end{equation*}
on which the data are observed and such that $L_n \rightarrow \infty$
for $n \rightarrow \infty$. For simplicity we will only analyze the
case of a uniform grid, i.e., we assume that
$\Delta t\coloneqq t_{n,k+1}-t_{n,k}=\frac{1}{L_n}$ for all
$k\in\{1,\dots,L_n\}$. We denote the $m$ environments by
$e_1,\dots,e_m\subseteq\{1,\dots,n\}$ and assume for all
$k\in\{1,\dots,m\}$ that $\abs{e_k}=R_n$ which grows as $n$ increases.

To achieve consistency of our ranking procedures (both for models and
variables) we require the following three conditions (C1)--(C3) below.
These three conditions should be understood as high-level conditions
or guidelines. There may be, of course, other sufficient assumptions
that yield the desired result and that might cover other settings and
models.
\begin{enumerate}[label=(C\arabic*)]
\item\label{it:cond1} \textbf{Consistency of target smoothing:} The
  smoothing procedure in \ref{it:model1}, see
  Section~\ref{sec:stability_ranking}, satisfies the following
  consistency: For all $k\in\{1,\dots,m\}$ it holds that
  \begin{equation*}
    \lim_{n\rightarrow\infty}\E\left(\sup_{t\in[0,T]}\left(\hat{y}^{(e_k)}_{a}(t)-Y^{(e_k)}_{t}\right)^2\right)=0,
  \end{equation*}
  where, by slight abuse of notation, the superscript $(e_k)$ denotes
  a fixed repetition from the environment $e_k$.
\item\label{it:cond2} \textbf{Consistency of model estimation:} For
  every invariant model $M\in\mathcal{M}$, let $\hat{g}_n$ be the
  estimate from \ref{it:model2}. Then, for all $k\in\{1,\dots,m\}$ it
  holds that
  \begin{equation*}
    L_n\max_{\ell\in\{1,\dots,L_n\}}\abs{\hat{g}_n(\widetilde{\B{X}}^{(e_k)}_{t_{n,\ell}})-\dot{Y}^{(e_k)}_{t_{n,\ell}}}\overset{\prob}{\longrightarrow}0
  \end{equation*}
  as $n\rightarrow\infty$, i.e., the estimation procedure $\hat{g}_n$
  in \ref{it:model2} is consistent. Furthermore, for all non-invariant
  models $M\in\mathcal{M}$ there exists a smooth function
  $g\in M$ such that $g$ and its first derivative are bounded
  and it holds for all $t\in[0,T]$ and for all $k\in\{1,\dots,m\}$
  that
  \begin{equation*}
    L_n\max_{\ell\in\{1,\dots,L_n\}}\abs{\hat{g}_n(\widetilde{\B{X}}^{(e_k)}_{t_{\ell}})-g(\B{X}^{(e_k)}_{t_{\ell}})}\overset{\prob}{\longrightarrow}0
  \end{equation*}
  as $n\rightarrow\infty$, i.e., the estimation convergences to a
  fixed function.
\item\label{it:cond3} \textbf{Uniqueness of invariant model:} There
  exists a unique function
  $g^*\in\cup_{M\in\mathcal{M}}M$ and a unique set
  $S^*\subseteq\{1,\dots,d\}$ such that for all $n\in\{m,m+1,\dots\}$
  the pair $f^*(\B{x})\coloneqq g^*(\B{x}^{S^*})$ and $S^*$ satisfy
  Assumption~\ref{ass:inv}. This condition is fulfilled if the
  experiments are sufficiently heterogeneous, e.g., because there are
  sufficiently many and strong interventions.
\end{enumerate}
Note that (C2) relates to the problem of error-in-variables. Relying
on the conditions \ref{it:cond1}--\ref{it:cond3}, we are now able to
prove consistency results for both the model ranking from
Section~\ref{sec:stability_ranking} and the variable ranking from
Section~\ref{sec:ranking_variables}. Recalling the definition of
$T_n(M)$ given in [\ref{eq:scoreTS2}] (small values of
$T_n(M)$ indicate invariance), we define
\begin{equation}\label{eq:rank_accuracy}
  \operatorname{RankAccuracy}_n
  \coloneqq 1-\frac{\big\vert\{M\in\mathcal{M}\,|\,
    T_{n}(M)<\max_{\{\tilde{M}\in\mathcal{M}\mid
    \tilde{M}\text{ invariant}\}}T_n(\tilde{M})\text{ and }
    M\text{ not invariant}\}\big\vert}{\big\vert\{M\in\mathcal{M}\,|\,
    M\text{ not invariant}\}\big\vert}
\end{equation}
as performance measure of our model ranking. RankAccuracy is thus
equal to $1$ minus ``proportion of non-invariant models that are
ranked better than the worst invariant model''. In particular, it
equals $1$ if and only if all invariant models are ranked better than
all other models. If the collection $\mathcal{M}$ contains no
invariant models, we define the RankAccuracy to be $1$. Given the
above conditions the following consistency holds.
\begin{theorem}[rank consistency]
  \label{thm:rank_consistency}
  Let Assumption~\ref{ass:inv} and conditions \ref{it:cond1} and
  \ref{it:cond2} be satisfied. Additionally, assume that for all
  $k\in\{1,\dots,m\}$ it holds for all $i\in e_k$ and
  $\ell\in\{1,\dots,L_n\}$ that the noise variables
  $\eps_{t_{\ell}}^{(i)}$ are i.i.d., symmetric, sub-Gaussian and
  satisfy $\E(\eps_{t_{\ell}}^{(i)})=0$ and
  $\var(\eps_{t_{\ell}}^{(i)})=\sigma^2_k$. Let $Y_t$ and its first
  and second derivative be bounded and assume that for all
  non-invariant sets $M\in\mathcal{M}$ the sets
  $\{t\mapsto g(X_t)\,\vert\, g\in M\}$ are closed with
  respect to the supremum norm. Then, it holds that
  \begin{equation*}
    \lim_{n\rightarrow\infty} \E\left(\operatorname{RankAccuracy}_n\right)=1.
  \end{equation*}
  If, in addition, condition~\ref{it:cond3} holds, we have the
  following guarantee for the variable scores $s_j^n = s_j$, defined
  in [\ref{eq:variable_score}]:
  \begin{itemize}
  \item for all $j\in S^*$ it holds that $\lim_{n\rightarrow\infty}
    \E(s^n_j)=1$ and
  \item for all $j\not\in S^*$ it holds that $\lim_{n\rightarrow\infty} \E(s^n_j)\leq \frac{K-1}{K}$,
  \end{itemize}
where $K:=\abs{\{M\in\mathcal{M}\,\vert\, M \text{ is
      invariant}\}}$.
\end{theorem}
The result is proved in the following section, which also contains the
choice of $C$ for \ref{it:model1}.
Note that if screening is used, 
condition~\ref{it:cond3} 
can only be satisfied if the screening procedure 
did not remove all invariant models.

\subsection{Proof of
  Theorem~\ref{thm:rank_consistency}}\label{sec:proofs}

To simplify notation, we will whenever it is clear from the context
drop the $n$ in the grid time points $t_{n,\ell}$ and simply write
$t_{\ell}$.

\subsubsection{Intermediate results}
In order to prove Theorem~\ref{thm:rank_consistency} we require the
following two auxiliary results.
\begin{lemma}
  \label{thm:minimumdiff}
  Let $y_1,y_2:[0,T]\rightarrow\R$ be two smooth functions satisfying that
  there exists $c_1>0$  such that
  \begin{equation}
    \label{eq:tstar_bound}
   \exists t^*\in[0,T]\text{ with }\abs{\dot{y}_1(t^*)-\dot{y}_2(t^*)}\geq c_1. 
  \end{equation}
  Moreover, assume
  $c_2\coloneqq\sup_{t\in[0,T]}(\abs{\ddot{y}_1(t)}+\abs{\ddot{y}_2(t)})
  < \infty$.  Then, there exists an interval
  $[l_1,l_2]\subseteq [t^*-\frac{c_1}{4c_2}, t^*+\frac{c_1}{4c_2}]$
  satisfying that $l_2-l_1=\frac{c_1}{8c_2}$ and
  \begin{equation*}
    \inf_{t\in[l_1,l_2]}\abs{y_1(t)-y_2(t)}\geq
    \frac{c_1^2}{16c_2}.
  \end{equation*}
\end{lemma}

\begin{proof}
  To simplify presentation, we will assume that $t^*$ from
  [\ref{eq:tstar_bound}] is not on the boundary of the interval
  $[0,T]$ and that all the intervals considered in this proof are
  contained in $(0,T)$. We first show that the bound on the second
  derivative of the functions implies that the difference in first
  derivatives is lower bounded on a closed interval. Using a basic
  derivative inequality it holds for $i\in\{1,2\}$ and
  $t\in[t^*, t^*+\frac{c_1}{4c_2}]$ that
  \begin{equation*}
    \dot{y}_i(t)\leq \dot{y}_i(t^*)+\frac{c_1}{4c_2}\cdot\sup_{s\in[0, T]}\abs{\ddot{y}_i(s)}\leq\dot{y}_i(t^*)+\frac{c_1}{4}.
  \end{equation*}
  Similarly, for $i\in\{1,2\}$ and
  $t\in[t^*-\frac{c_1}{4c_2}, t^*]$ it holds that
  \begin{equation*}
    \dot{y}_i(t)\geq \dot{y}_i(t^*)-\frac{c_1}{4c_2}\cdot\sup_{s\in[0, T]}\abs{\ddot{y}_i(s)}\geq\dot{y}_i(t^*)-\frac{c_1}{4}.
  \end{equation*}
  Combining these inequalities with [\ref{eq:tstar_bound}] yields
  \begin{equation}
    \label{eq:infbddderiv}
    \inf_{t\in[t^*-\frac{c_1}{4c_2},
      t^*+\frac{c_1}{4c_2}]}(\dot{y}_1(t)-\dot{y}_2(t))\cdot\operatorname{sign}(\dot{y}_1(t^*)-\dot{y}_2(t^*))\geq\frac{c_1}{2}.
  \end{equation}
  Next, we show that this lower bound on the difference of the first
  derivatives implies the statement of the lemma. To this end,
  consider the two intervals
  $I_1=[t^*-\frac{c_1}{4c_2}, t^*-\frac{c_1}{8c_2}]$ and
  $I_2=[t^*+\frac{c_1}{8c_2}, t^*+\frac{c_1}{4c_2}]$. 
  We show that  at
  least one of the following two inequalities holds
  \begin{enumerate}[label=(\alph*)]
  \item $\inf_{t\in I_1}\abs{y_1(t)-y_2(t)}\geq \frac{c_1^2}{16c_2}$, \label{it:interval1}
  \item $\inf_{t\in I_2}\abs{y_1(t)-y_2(t)}\geq \frac{c_1^2}{16c_2}$\label{it:interval2}.
  \end{enumerate}
  Assume that \ref{it:interval1} does not hold. Then, there exists
  $t\in I_1$ such that
  \begin{equation}
    \label{eq:integralineq0}
    \abs{y_1(t)-y_2(t)}< \frac{c_1^2}{16c_2}.
  \end{equation}
  Let $s\in I_2$, then since the sign of the difference in first
  derivatives remains constant on the interval
  $[t^*-\frac{c_1}{4c_2}, t^*+\frac{c_1}{4c_2}]$ (see
  [\ref{eq:infbddderiv}]) it holds by integration that
  \begin{align}
    \int_t^s\abs{\dot{y}_1(r)-\dot{y}_2(r)}dr
    &=[(y_1(s)-y_2(s))-(y_1(t)-y_2(t))]\cdot
      \operatorname{sign}(\dot{y}_1(t^*)-\dot{y}_2(t^*))\nonumber\\
    &=\abs{y_1(s)-y_2(s)}-\abs{y_1(t)-y_2(t)}.\label{eq:integralineq1}
  \end{align}
  By [\ref{eq:infbddderiv}], it additionally holds that
  \begin{equation}
    \label{eq:integralineq2}
    \int_t^s\abs{\dot{y}_1(r)-\dot{y}_2(r)}dr\geq
    (s-t)\frac{c_1}{2}\geq \frac{c_1}{4c_2}\frac{c_1}{2}=\frac{c_1^2}{8c_2}.
  \end{equation}
  Finally, combining [\ref{eq:integralineq0}],
  [\ref{eq:integralineq1}] and [\ref{eq:integralineq2}] we get that
  \begin{equation*}
    \abs{y_1(s)-y_2(s)}\geq\frac{c_1^2}{16c_2},
  \end{equation*}
  which implies that \ref{it:interval2} holds since $s\in I_2$ was
  arbitrary. 
  An analogous 
  argument can be used if we assume that
  \ref{it:interval2} does not hold. Hence, since at least one of
  \ref{it:interval1} and \ref{it:interval2} holds, we have proved Lemma~\ref{thm:minimumdiff}.
\end{proof}
\begin{lemma}
  \label{lem:existenceopt}
  For any a smooth function $f: \mathbb{R} \rightarrow \mathbb{R}$,
  with
  $\max(\sup_{t} |f(t)|, \sup_{t} |\dot{f}(t)|, \sup_{t} |\ddot{f}(t)|) \leq C$,
  any $a>1$, and any $r \in \mathbb{R}$, there exists a smooth
  function $g$ satisfying $\dot{g}(0) = r$, 
    $g(t) = f(t)$ for all $\abs{t}\geq 1/a$,
and
  \begin{equation*}
    \max\left(\sup_{t} |g(t)|, \sup_{t} |\dot{g}(t)|, \sup_{t} |\ddot{g}(t)|\right) \leq 
    C + 16a\abs{r - \dot{f}(0)}.
  \end{equation*}
\end{lemma}
\begin{proof}
  Assume that we are given a smooth function $b_a$ that is supported
  on $[-1/a,1/a]$, and that has derivative $\dot{b}_a(0) = 1$.  We can then
  define
  $$
  g(t) := f(t) + \left(r - \dot{f}(0)\right) \cdot b_a(t),
  $$
  which is equal to $f$ outside the interval $[-1/a,1/a]$ and which
  satisfies $\dot{g}(0) = r$.

  Let us first create such a function $b_a$. To do so, define
  $$
  b(t) := \left\{
    \begin{array}{cl}
      \sin(t) \exp{\left(1-\tfrac{1}{1-t^2}\right)}&\text{if }\abs{t} < 1\\
      0&\text{otherwise.}
    \end{array}
  \right.
  $$
  This function is smooth and satisfies 
  $\sup_t \abs{b(t)} \leq 1$, 
  $\sup_t \abs{\dot{b}(t)} \leq 1$, 
  $\sup_t \abs{\ddot{b}(t)} \leq 16$, and 
  $\dot{b}(0) = 1$.
  We now define the function
  $$
  b_a(t) := \frac{1}{a} b(at),
  $$
  whose support contained in $[-1/a,1/a]$.
  Because of 
  $\dot{b}_a(t) = \dot{b}(at)$, we have
  $\dot{b}_a(0) = 1$.
  Finally, we find 
  \begin{align*}
    \sup_t |\dot{g}(t)| &
                     \leq C + |c - \dot{f}(0)| \sup_t |\dot{b}_a(t)| 
                     \leq C + |c - \dot{f}(0)|\\
    \sup_t \abs{\ddot{g}(t)} &\leq C + |c - \dot{f}(0)| 16 a,
  \end{align*}
  where the last line follows from $\ddot{b}_a(t) = a\ddot{b}(at)$. This
  completes the proof of Lemma~\ref{lem:existenceopt}.
\end{proof}

\begin{lemma}
  \label{thm:triangular_convergence}
  Let $((\eps_{n,k})_{k\in\{1,\dots,n\}})_{n\in\N}$ be a triangular
  array of i.i.d.\ sub-Gaussian (with parameter $\nu$) random
  variables. Moreover, assume
  $((X_{n,k})_{k\in\{1,\dots,n\}})_{n\in\N}$ is a triangular array of
  random variables which satisfies that
  \begin{equation*}
    \max_{k\in\{1,\dots,n\}}X_{n,k}\overset{\prob}{\longrightarrow}0\text{
    as } n\rightarrow\infty
    \quad\text{and}\quad
    \exists
    K>0:\,\sup_{n\in\N}\max_{k\in\{1,\dots,n\}}\abs{X_{n,k}}\leq K.
  \end{equation*}
  Then, it holds that
  \begin{equation*}
    \lim_{n\rightarrow\infty}\frac{1}{n}\sum_{k=1}^n\E\left(\abs{X_{n,k}\eps_{n,k}}\right)=0.
  \end{equation*}
\end{lemma}

\begin{proof}
  Fix $\delta,\theta>0$, then by the convergence in probability it holds that
  there exists $N\in\N$ such that for all $n\in\{N, N+1,\dots\}$ it
  holds that
  \begin{equation}
    \label{eq:conv_in_prob}
    \prob\left(\max_{k\in\{1,\dots,n\}}\abs{X_{n,k}}> n\right)\leq\prob\left(\max_{k\in\{1,\dots,n\}}\abs{X_{n,k}}> 1\right)\leq\frac{\delta}{2}.
  \end{equation}
  Furthermore, using independence, sub-Gaussianity and Bernoulli's
  inequality we get for all $c>0$ and $n\in\N$ that
  \begin{align}
    \prob\left(\max_{k\in\{1,\dots,n\}}\abs{\eps_{n,k}}> c\right)
    &=1-\prob\left(\max_{k\in\{1,\dots,n\}}\abs{\eps_{n,k}}\leq c\right)\nonumber\\
    &=1-\prob\left(\abs{\eps_{n,1}}\leq c\right)^n\nonumber\\
    &=1-\left(1-\prob\left(\abs{\eps_{n,1}}>c\right)\right)^n\nonumber\\
    &\leq 1-\left(1-C e^{-\nu c^2}\right)^n\nonumber\\
    &\leq nC e^{-\nu c^2}.\label{eq:subgauss}
  \end{align}
  Combining [\ref{eq:conv_in_prob}] and
  [\ref{eq:subgauss}] this proves that for all $n\in\{N,N+1,\dots\}$
  it holds that
  \begin{align*}
    \prob\left(\max_{k\in\{1,\dots,n\}}\abs{X_{n,k}\eps_{n,k}}>\theta\right)
    &=\prob\left(\max_{k\in\{1,\dots,n\}}\abs{X_{n,k}\eps_{n,k}}>\theta,
      \max_{k\in\{1,\dots,n\}}\abs{X_{n,k}}\leq n\right)\\
    &\qquad+\prob\left(\max_{k\in\{1,\dots,n\}}\abs{X_{n,k}\eps_{n,k}}>\theta,
      \max_{k\in\{1,\dots,n\}}\abs{X_{n,k}}> n\right)\\
    &\leq\prob\left(\max_{k\in\{1,\dots,n\}}\abs{\eps_{n,k}}>\frac{\theta}{n}\right)
      +\prob\left(\max_{k\in\{1,\dots,n\}}\abs{X_{n,k}}> n\right)\\
    &\leq nC e^{-\nu (\frac{\theta}{n})^2}
      +\frac{\delta}{2}.
  \end{align*}
  Since the term $nC e^{-\nu (\frac{\theta}{n})^2}$ converges to zeros
  as $n$ goes to infinity, there exists $N^*\in\{N,N+1,\dots\}$ such
  that for all $n\in\{N^*,N^*+1,\dots\}$ it holds that
  \begin{equation*}
    nC e^{-\nu (\frac{\theta}{n})^2}\leq \frac{\delta}{2}.
  \end{equation*}
  Finally, we combine these results to show that for all
  $n\in\{N^*,N^*+1,\dots\}$ it holds that
  \begin{equation*}
    \prob\left(\max_{k\in\{1,\dots,n\}}\abs{X_{n,k}\eps_{n,k}}>\theta\right)\leq \delta,
  \end{equation*}
  which implies that $\max_{k\in\{1,\dots,n\}}\abs{X_{n,k}\eps_{n,k}}$
  converges to zero in probability as $n\rightarrow\infty$. In
  particular, $\frac{1}{n}\sum_{k=1}^n\abs{X_{k,n}\eps_{k,n}}$ also
  converges to zero in probability as it is $\prob$-a.s.\ dominated by
  $\max_{k\in\{1,\dots,n\}}\abs{X_{n,k}\eps_{n,k}}$. Furthermore, due
  to boundedness assumption on $X_{n,k}$ it also holds that
  \begin{equation*}
    \sup_{n\in\N}\E\left(\abs[\Big]{\frac{1}{n}\sum_{k=1}^nX_{k,n}\eps_{k,n}}^2\right)<\infty,
  \end{equation*}
  which by de la Vallée-Poussin's theorem \citep[p.19 Theorem
  T22]{meyer1966} implies uniform integrability. Since uniform
  integrability and convergence in probability is equivalent to
  convergence in $L^1$, this completes the proof of
  Lemma~\ref{thm:triangular_convergence}.
\end{proof}

The following two lemmas are the key steps used in the proof of the
Theorem~\ref{thm:rank_consistency}. They prove some essential
properties related to the constraint optimization, i.e., the
estimation of $\hat{y}_{b}$.
\begin{lemma}
  \label{thm:consistency_yb}
  Consider the setting of Theorem~\ref{thm:rank_consistency}, that is,
  let Assumption~\ref{ass:inv} and conditions \ref{it:cond1} and
  \ref{it:cond2} be satisfied. Additionally, assume that for all
  $k\in\{1,\dots,m\}$ it holds for all $i\in e_k$ and
  $\ell\in\{1,\dots,L_n\}$ that the noise variables
  $\eps_{t_{\ell}}^{(i)}$ are i.i.d., symmetric, sub-Gaussian and
  satisfy $\E(\eps_{t_{\ell}}^{(i)})=0$ and
  $\var(\eps_{t_{\ell}}^{(i)})=\sigma^2_k$.  Let $Y_t$ and its first
  and second derivative be bounded by $c<\infty$ and define
  $C := c+16$ for the set $\mathcal{H}_C$, see \ref{it:model1}. Then,
  for an invariant model $M\in\mathcal{M}$ and for all
  $k\in\{1, \ldots, m\}$ it holds that
  \begin{equation*}
    \lim_{n\rightarrow\infty}\E\left(
    \sup_{t\in[0,T]}\left(\hat{y}^{(e_k)}_{b}(t)-Y^{(e_k)}_{t}\right)^2
    \right)=0,
  \end{equation*}
  i.e., the outcome of step~\ref{it:model3} converges towards the true
  target trajectory.  Furthermore, for $M\in\mathcal{M}$
  non-invariant there exists $k^*\in\{1,\dots,m\}$ and $c_{\min}>0$ such
  that
  \begin{equation}
    \label{eq:lemma_lowerbound}
    \liminf_{n\rightarrow\infty}\prob\left(\frac{1}{L_n}\sum_{\ell=1}^{L_n}\left(\hat{y}^{(e_{k^*})}_{b}(t_{\ell})-Y^{(e_{k^*})}_{t_{\ell}}\right)^2\geq
      c_{\min}\right)=1.
  \end{equation}
\end{lemma}
\begin{proof}
  First, recall the definition of $\mathcal{H}_C$
  (see~\ref{it:model1}) and define the smoother function
  $\hat{y}^{(i)}_{c}\in\mathcal{H}_C$ corresponding to the constrained
  optimization based on the true derivatives, i.e.,
  \begin{align*}
    \begin{split}
      \hat{y}^{(i)}_{c}\coloneqq 
      &\argmin_{y\in\mathcal{H}_C}
      \sum_{\ell = 1}^L \big(\widetilde{Y}^{(i)}_{t_\ell} - y(t_\ell)\big)^2 + \lambda \int \ddot{y}(s)^2\, ds,\\
      &\text{such that}\quad
      \dot{y}(t_{\ell}) = \dot{Y}^{(i)}_{t_{\ell}} \text{ for all } {\ell} = 1, \ldots, L_n.
    \end{split}
  \end{align*} 
  Fix $k\in\{1,\dots,m\}$. To simplify notation we will drop the
  superscript $(e_k)$ in the following. Fix $\delta\in(0,1)$ and define the
  sets
  \begin{equation*}
    A_{\delta}\coloneqq\left\{
      L_n    
      \max_{\ell\in\{1,\dots,L_n\}}\abs{\hat{g}_n(\widetilde{\B{X}}_{t_{\ell}})-\dot{Y}_{t_{\ell}}}\leq\delta\right\}
    \quad\text{and}\quad
    B_{\delta}\coloneqq\left\{\abs[\Big]{\frac{1}{L_n}\sum_{\ell=1}^{L_n}\eps_{t_{\ell}}}\leq\delta\right\}.
  \end{equation*}
  Then, by condition \ref{it:cond2} it holds that
  \begin{equation}
    \label{eq:convergenceofsetAdelta}
    \lim_{n\rightarrow\infty}\prob\left(A_{\delta}\right)=1,
  \end{equation}
  and, by the law of large numbers,
  \begin{equation}
    \label{eq:convergenceofsetBdelta}
    \lim_{n\rightarrow\infty}\prob\left(B_{\delta}\right)=1.
  \end{equation}
  Note that on the set $A_{\delta}$, our method is well-defined: for
  $a=L_n$, Lemma~\ref{lem:existenceopt} shows us that the function
  $\hat{y}_b$ exists since the corresponding optimization problem has
  at least one solution. Then, on the event
  $A_{\delta}\cap B_{\delta}$ it holds that
  \begin{align}
    \max_{\ell\in\{1,\dots,L_n\}}\abs{\hat{y}_b(t_{\ell})-\hat{y}_c(t_{\ell})}
    &\leq \sum_{k=1}^{L_n}\int_{t_{k-1}}^{t_{k}}\abs{\dot{\hat{y}}_b(s)-\dot{\hat{y}}_c(s)}ds+\abs{\hat{y}_b(t_{1})-\hat{y}_c(t_{1})}
    \nonumber\\
    &\leq L_n\max_{\ell\in\{2,\dots,L_n\}}\left(\int_{t_{\ell-1}}^{t_{\ell}}\tfrac{2C}{L_n}+\abs{\dot{\hat{y}}_b(t_{\ell-1})-\dot{\hat{y}}_c(t_{\ell-1})}ds\right)+\abs{\hat{y}_b(t_{1})-\hat{y}_c(t_{1})}\nonumber\\
    &\leq\tfrac{2C}{L_n}+\delta+\abs{\hat{y}_b(t_{1})-\hat{y}_c(t_{1})},\label{eq:bddmaxyhatbc}
  \end{align}
  where the second last inequality follows from the bound on the
  second derivative. Moreover, define the function
  $y_{b*}\coloneqq \hat{y}_b-\hat{y}_b(t_{1})+Y_{t_{1}}$ then similar
  arguments show that
  \begin{equation}
    \label{eq:ybstarbdd}
    \max_{\ell\in\{1,\dots,L_n\}}\abs{y_{b*}(t_{\ell})-Y_{t_{\ell}}}
    =\max_{\ell\in\{1,\dots,L_n\}}\abs{(\hat{y}_b(t_{\ell})-\hat{y}_b(t_{1}))-(Y_{t_{\ell}}-Y_{t_{1}})}
    \leq\tfrac{2C}{L_n}+\delta.
  \end{equation}
  Using that $\hat{y}_c$
  has the true derivatives as constraint the same argument implies for $y_{c*}\coloneqq
  \hat{y}_c-\hat{y}_c(t_{1})+Y_{t_{1}}$ that
  \begin{equation}
    \label{eq:ycstarbdd}
    \max_{\ell\in\{1,\dots,L_n\}}\abs{y_{c*}(t_{\ell})-Y_{t_{\ell}}}
    \leq\tfrac{2C}{L_n}.
  \end{equation}
  Next, define the loss function
  \begin{equation*}
    \operatorname{loss}_n(y)\coloneqq\sum_{\ell=1}^{L_n}\left(\widetilde{Y}_{t_{\ell}}-y(t_{\ell})\right)^2+\lambda_n\int_{0}^{T}\ddot{y}(s)^2ds.
  \end{equation*}
  Then using [\ref{eq:ybstarbdd}] and [\ref{eq:ycstarbdd}] it holds that
  \begin{align*}
    \operatorname{loss}_n(\hat{y}_b)
    &=\sum_{\ell=1}^{L_n}\left(\widetilde{Y}_{t_{\ell}}-\hat{y}_b(t_{\ell})\right)^2+\lambda_n\int_{0}^{T}\ddot{\hat{y}}_b(s)^2ds\\
    &=\operatorname{loss}_n(y_{b*})+\sum_{\ell=1}^{L_n}\left(Y_{t_{1}}-\hat{y}_b(t_{1})\right)^2+2\left(Y_{t_{1}}-\hat{y}_b(t_{1})\right)\sum_{\ell=1}^{L_n}\left(\widetilde{Y}_{t_{\ell}}-y_{b*}(t_{\ell})\right)\\
    &\geq\operatorname{loss}_n(y_{b*})+L_n\left(Y_{t_{1}}-\hat{y}_b(t_{1})\right)^2+2\abs{Y_{t_{1}}-\hat{y}_b(t_{1})}L_n\left(\tfrac{2C}{L_n^2}+\tfrac{\delta}{L_n}\right)+2\left(Y_{t_{1}}-\hat{y}_b(t_{1})\right)\sum_{\ell=1}^{L_n}\eps_{t_{\ell}}\\
  \end{align*}
  Now, $y_{b*}$ has the same derivatives as $\hat{y}_b$ and since
  $\hat{y}_b$ minimizes $\operatorname{loss}_n$ under fixed derivative
  constraints it holds that
  $\operatorname{loss}_n(\hat{y}_b)\leq\operatorname{loss}_n(y_{b*})$. This
  implies
  \begin{equation}
    \label{eq:yhatbpart1}
    L_n\left(Y_{t_{1}}-\hat{y}_b(t_{1})\right)^2
    \leq
    2\abs{Y_{t_{1}}-\hat{y}_b(t_{1})}L_n\left(\tfrac{2C}{L_n}+\delta\right)+2\left(Y_{t_{1}}-\hat{y}_b(t_{1})\right)\sum_{\ell=1}^{L_n}\eps_{t_{\ell}},
  \end{equation}
  which is equivalent to
  \begin{equation}
    \label{eq:yhatbpart2}
    \abs{Y_{t_{1}}-\hat{y}_b(t_{1})}
    \leq 2\cdot\left(\tfrac{2C}{L_n}+\delta\right)+2\abs[\bigg]{\frac{1}{L_n}\sum_{\ell=1}^{L_n}\eps_{t_{\ell}}}.
  \end{equation}
  Since, we are on the set $B_{\delta}$ this in particular as
  $n\rightarrow\infty$ implies that
  \begin{equation}
    \label{eq:yhatbyhatcneeded1}
    \limsup_{n\rightarrow\infty}\abs{Y_{t_{1}}-\hat{y}_b(t_{1})}
    \leq 4\delta.
  \end{equation}
  With the same arguments as in [\ref{eq:yhatbpart1}] and
  [\ref{eq:yhatbpart2}] for the function $\hat{y}_{c}$ we get that
  \begin{equation}
    \label{eq:yhatbyhatcneeded2}
    \limsup_{n\rightarrow\infty}\abs{Y_{t_{1}}-\hat{y}_c(t_{1})}
    \leq 2\delta.
  \end{equation}
  Combining [\ref{eq:yhatbyhatcneeded1}] and
  [\ref{eq:yhatbyhatcneeded2}] with the triangle inequality it holds
  that
  \begin{equation}
    \limsup_{n\rightarrow\infty}\abs{\hat{y}_{b}(t_{1})-\hat{y}_c(t_{1})}
    \leq 6\delta.
  \end{equation}
  Hence, we can combine this with [\ref{eq:bddmaxyhatbc}] to get that
  \begin{equation*}
    \limsup_{n\rightarrow\infty}\max_{\ell\in\{1,\dots,L_n\}}\abs{\hat{y}_b(t_{\ell})-\hat{y}_c(t_{\ell})}
    \leq 7\delta,
  \end{equation*}
  which together with the global bound on the first derivative
  also implies that
  \begin{equation*}
    \limsup_{n\rightarrow\infty}\sup_{t\in[0,T]}\abs{\hat{y}_b(t)-\hat{y}_c(t)}\leq
    \limsup_{n\rightarrow\infty}\left(\max_{\ell\in\{1,\dots,L_n\}}\abs{\hat{y}_b(t_{\ell})-\hat{y}_c(t_{\ell})}+\tfrac{C}{L_n}\right)
    \leq 7\delta.
  \end{equation*}
  Finally, we use this, the global bound and the dominated convergence
  theorem to show that
  \begin{align*}
    &\lim_{n\rightarrow\infty}\E\left(
      \sup_{t\in[0,T]}\left(\hat{y}^{(e_k)}_{b}(t)-Y^{(e_k)}_{t}\right)^2\right)\\
    &\quad=\lim_{n\rightarrow\infty}\left(\E\left(
      \sup_{t\in[0,T]}\left(\hat{y}^{(e_k)}_{b}(t)-Y^{(e_k)}_{t}\right)^2\mathds{1}_{A_{\delta}\cap
      B_{\delta}}\right)
      +\E\left(
      \sup_{t\in[0,T]}\left(\hat{y}^{(e_k)}_{b}(t)-Y^{(e_k)}_{t}\right)^2\mathds{1}_{A_{\delta}^c\cup
      B_{\delta}^c}\right)\right)\\
    &\quad\leq 7\delta
      +\lim_{n\rightarrow\infty}\prob\left(A_{\delta}^c\cup
      B_{\delta}^c\right)=7\delta.
  \end{align*}
  Since $\delta>0$ was arbitrary this proves the first part of the lemma.
    
  Next, we show the second part. To that end, let $M\in\mathcal{M}$ be
  non-invariant. Since we assumed that the set
  $\{t\mapsto g(X_t)\,\vert\, g\in M\}$ is closed with respect to the
  sup norm there exist $c>0$, $k^*\in\{1,\dots,m\}$ and
  $(t_n^*)_{n\in\N}\subseteq [0,T]$ such that for all $n\in\N$ it
  holds that
  \begin{equation}
    \label{eq:littlecbdd}
    \abs{\dot{Y}_{t_n^*}^{(e_{k^*})}-\hat{g}_n(\B{X}_{t_n^*}^{(e_{k^*})})}\geq
    c.
  \end{equation}
  Next, define
  $\ell^{*}_n\coloneqq\argmin_{\ell\in\{1,\dots,L_n\}}\abs{t_n^*-t_{\ell}}$
  then by the derivative constraint it in particular holds that
  $\dot{\hat{y}}_b^{(e_{k^*})}(t_{\ell_n^*})=\hat{g}_n(\widetilde{\B{X}}_{t_{\ell_n^*}}^{(e_{k^*})})$.
  Moreover, using the global bound from the function class
  $\mathcal{H}_C$ it holds that
  \begin{align}
    &\abs{\hat{g}_n(\B{X}_{t_n^*}^{(e_{k^*})})-\dot{\hat{y}}_b^{(e_{k^*})}(t_n^*)}\nonumber\\
    &\quad\leq\abs{\hat{g}_n(\B{X}_{t_{n}^*}^{(e_{k^*})})-\hat{g}_n(\widetilde{\B{X}}_{t_{\ell_n^*}}^{(e_{k^*})})}+\abs{\dot{\hat{y}}_b^{(e_{k^*})}(t_{n}^*)-\dot{\hat{y}}_b^{(e_{k^*})}(t_{\ell_n^*})}\nonumber\\
    &\quad\leq\abs{\hat{g}_n(\B{X}_{t_{n}^*}^{(e_{k^*})})-g(\B{X}_{t_n^*}^{(e_{k^*})})}+
      \abs{\hat{g}_n(\widetilde{\B{X}}_{t_{\ell_n^*}}^{(e_{k^*})})-g(\B{X}_{t_{\ell_n^*}}^{(e_{k^*})})}+
      \abs{g(\B{X}_{t_{n}^*}^{(e_{k^*})})-g(\B{X}_{t_{\ell_n^*}}^{(e_{k^*})})}+\frac{C}{L_n}\nonumber\\
    &\quad\leq\abs{\hat{g}_n(\B{X}_{t_{n}^*}^{(e_{k^*})})-g(\B{X}_{t_n^*}^{(e_{k^*})})}+
      \abs{\hat{g}_n(\widetilde{\B{X}}_{t_{\ell_n^*}}^{(e_{k^*})})-g(\B{X}_{t_{\ell_n^*}}^{(e_{k^*})})}+
      \frac{2C}{L_n}.\label{eq:2CLnbdd}
  \end{align}
  Combining the bounds in [\ref{eq:littlecbdd}] and [\ref{eq:2CLnbdd}]
  implies that
  \begin{align*}
    \abs{\dot{Y}_{t^*_n}^{(e_{k^*})}-\dot{\hat{y}}_b^{(e_{k^*})}(t^*_n)}&\geq
    \abs{\dot{Y}_{t_n^*}^{(e_{k^*})}-\hat{g}_n(\B{\B{X}}_{t_n^*}^{(e_{k^*})})}-
    \abs{\hat{g}_n(\B{X}_{t_n^*}^{(e_{k^*})})-\dot{\hat{y}}_b^{(e_{k^*})}(t_n^*)}\\
    &\geq c -\abs{\hat{g}_n(\B{X}_{t_{n}^*}^{(e_{k^*})})-g(\B{X}_{t_n^*}^{(e_{k^*})})}-
      \abs{\hat{g}_n(\widetilde{\B{X}}_{t_{\ell_n^*}}^{(e_{k^*})})-g(\B{X}_{t_{\ell_n^*}}^{(e_{k^*})})}-
      \frac{2C}{L_n}.
  \end{align*}
  Next, assume $n\in\N$ is large enough such that $c-\frac{2C}{L_n}>0$
  and define for $\delta\in(0,c-\frac{2C}{L_n})$ the event
  $C_{\delta}\coloneqq\{\abs{\hat{g}_n(\B{X}_{t^*_n}^{(e_{k^*})})-g(\B{X}_{t^*_n}^{(e_{k^*})})}+\abs{\hat{g}_n(\widetilde{\B{X}}_{t_{\ell_n^*}}^{(e_{k^*})})-g(\B{X}_{t_{\ell_n^*}}^{(e_{k^*})})}\leq
  \delta\}$ (which depends on $n$). Then on $C_{\delta}$ it holds by
  Lemma~\ref{thm:minimumdiff} that there exist intervals
  $[l_{1,n},l_{2,n}]\subseteq [0,T]$ with length strictly greater than
  a fixed constant (independent of $n$) and a constant $\mu>0$ (also
  independent of $n$) satisfying that
  \begin{equation*}
    \inf_{t\in[l_{1,n},l_{2,n}]}\abs{Y_{t}^{(e_{k^*})}-\hat{y}_b^{(e_{k^*})}(t)}\geq
    \mu.
  \end{equation*}
  Since we assumed an equally spaced grid it is clear that at least
  $\lfloor \frac{l_{n,2}-l_{1,n}}{T}n\rfloor$ grid points are contained in the
  interval $[l_{1,n},l_{2,n}]$. Hence, defining $c_{\min}\coloneqq\mu^2$ we
  get
  \begin{align*}
    &\liminf_{n\rightarrow\infty}\prob\left(\frac{1}{L_n}\sum_{\ell=1}^{L_n}\left(Y^{(e_{k^*})}_{t_{\ell}}-\hat{y}^{(e_{k^*})}_{b}(t_{\ell})\right)^2\geq
      c_{\min}\right)\\
    &\quad\geq\liminf_{n\rightarrow\infty}\prob\left(\left\lfloor \tfrac{l_{2,n}-l_{1,n}}{T}n\right\rfloor \inf_{t\in[l_{1,n},l_{2,n}]}\abs{Y_{t}^{(e_{k^*})}-\hat{y}_b^{(e_{k^*})}(t)}^2\geq
      c_{\min}\right)\\
    &\quad\geq\liminf_{n\rightarrow\infty}\prob\left(\left\{\left\lfloor \tfrac{l_{2,n}-l_{1,n}}{T}n\right\rfloor \inf_{t\in[l_{1,n},l_{2,n}]}\abs{Y_{t}^{(e_{k^*})}-\hat{y}_b^{(e_{k^*})}(t)}^2\geq
      c_{\min}\right\}\cap C_{\delta}\right)\\
    &\quad\geq\liminf_{n\rightarrow\infty}\prob\left(\left\{\left\lfloor \tfrac{l_{2,n}-l_{1,n}}{T}n\right\rfloor \mu^2\geq
      c_{\min}\right\}\cap C_{\delta}\right)\\
    &\quad=\liminf_{n\rightarrow\infty}\prob\left(C_{\delta}\right)=1,
  \end{align*}
  where in the last step we used the second part of condition
  \ref{it:cond2}. This completes the proof of
  Lemma~\ref{thm:consistency_yb}.
\end{proof}

Simply stated the following lemma proves that under
condition~\ref{it:cond2} it holds that for non-invariant
$M\in\mathcal{M}$ the estimates $\hat{y}_{b}$ corresponding
to the constraint optimization converge to a fixed function
$y_{\lim}$. The function $y_{\lim}(\cdot)$ can be explicitly constructed as
the integral of the derivative function $g(X_{\cdot})$ shifted by a fixed
constant that is chosen to minimize the area between $y_{\lim}(\cdot)$
and the true function $Y_{\cdot}$.

\begin{lemma}
  \label{thm:limitfunction}
  Let condition \ref{it:cond2} be satisfied. Additionally, assume that
  for all $k\in\{1,\dots,m\}$ it holds for all $i\in e_k$ and
  $\ell\in\{1,\dots,L_n\}$ that the noise variables
  $\eps_{t_{\ell}}^{(i)}$ are i.i.d., symmetric, sub-Gaussian and
  satisfy $\E(\eps_{t_{\ell}}^{(i)})=0$ and
  $\var(\eps_{t_{\ell}}^{(i)})=\sigma^2_k$.  Let $Y_t$ and its first
  and second derivative be bounded by $c<\infty$ and define
  $C := c+16$ for the set $\mathcal{H}_C$, see \ref{it:model1}. Then,
  for any non-invariant $M\in\mathcal{M}$ with
  $g\in M$ the limit function from condition \ref{it:cond2}
  it holds that for all $k\in\{1,\dots,m\}$ the functions
  $y_{\lim}^{(e_k)}:[0,T]\rightarrow\R$ defined for all $t\in[0,T]$ by
  \begin{equation*}
    y_{\lim}^{(e_k)}(t)\coloneqq\int_{0}^{t}g(X_s^{(e_k)})ds + \frac{1}{T}\int_{0}^{T}\left(Y_s^{(e_k)}-\int_{0}^{s}g(X_r^{(e_k)})dr\right)ds,
  \end{equation*}
  satisfy that
  \begin{equation*}
    \sup_{t\in[0,T]}\abs{\hat{y}_b^{(e_k)}(t)-y_{\lim}^{(e_k)}(t)}\overset{\prob}{\longrightarrow}0,
  \end{equation*}
  as $n\rightarrow\infty$.
\end{lemma}

\begin{proof}
  The proof is very similar in spirit to the proof of the second part
  of Lemma~\ref{thm:consistency_yb}. Let $M\in\mathcal{M}$
  be non-invariant, fix $k\in\{1,\dots,m\}$ and let $g\in M$
  be the function from the second part of condition~\ref{it:cond2}. To
  simplify notation we will drop the superscript $(e_k)$ in the
  remainder of this proof. Next, let $\delta\in(0,1)$ and define the
  sets
  \begin{equation*}
    A_{\delta}\coloneqq\left\{
      L_n    
      \max_{\ell\in\{1,\dots,L_n\}}\abs{\hat{g}_n(\widetilde{\B{X}}_{t_{\ell}})-g(\B{X}_{t_{\ell}})}\leq\delta\right\}
    \quad\text{and}\quad
    B_{\delta}\coloneqq\left\{\abs[\Big]{\frac{1}{L_n}\sum_{\ell=1}^{L_n}\eps_{t_{\ell}}}\leq\delta\right\}.
  \end{equation*}
  Then, by condition \ref{it:cond2} it holds that
  \begin{equation}
    \label{eq:convergenceofsetAdelta2}
    \lim_{n\rightarrow\infty}\prob\left(A_{\delta}\right)=1,
  \end{equation}
  and, by the law of large numbers,
  \begin{equation}
    \label{eq:convergenceofsetBdelta2}
    \lim_{n\rightarrow\infty}\prob\left(B_{\delta}\right)=1.
  \end{equation}  
  Note that on the set $A_{\delta}$, our method is well-defined: for
  $a=L_n$, Lemma~\ref{lem:existenceopt} shows us that the function
  $\hat{y}_b$ exists since the corresponding optimization problem has
  at least one solution. Then, on the event
  $A_{\delta}\cap B_{\delta}$ it holds that
  \begin{align}
    \max_{\ell\in\{1,\dots,L_n\}}\abs{\hat{y}_b(t_{\ell})-y_{\lim}(t_{\ell})}
    &\leq\sum_{k=1}^{L_n}\int_{t_{k-1}}^{t_{k}}\abs{\dot{\hat{y}}_b(s)-\dot{y}_{\lim}(s)}ds+\abs{\hat{y}_b(t_{1})-y_{\lim}(t_{1})}
    \nonumber\\
    &\leq L_n\max_{\ell\in\{2,\dots,L_n\}}\left(\int_{t_{\ell-1}}^{t_{\ell}}\tfrac{2C}{L_n}+\abs{\dot{\hat{y}}_b(t_{\ell-1})-\dot{y}_{\lim}(t_{\ell-1})}ds\right)+\abs{\hat{y}_b(t_{1})-y_{\lim}(t_{1})}\nonumber\\
    &\leq\tfrac{2C}{L_n}+\delta+\abs{\hat{y}_b(t_{1})-y_{\lim}(t_{1})},\label{eq:bddmaxyhatbc2}
  \end{align}
  where we used that $\dot{y}_{\lim}(t)=g(X_t)$. Moreover, define the
  function
  $y_{b*}\coloneqq \hat{y}_b-\hat{y}_b(t_{1})+y_{\lim}(t_{1})$ then
  similar arguments show that
  \begin{equation}
    \label{eq:ybstarbdd2}
    \max_{\ell\in\{1,\dots,L_n\}}\abs{y_{b*}(t_{\ell})-y_{\lim}(t_{\ell})}
    =\max_{\ell\in\{1,\dots,L_n\}}\abs{(\hat{y}_b(t_{\ell})-\hat{y}_b(t_{1}))-(y_{\lim}(t_{\ell})-y_{\lim}(t_{1}))}
    \leq\tfrac{2C}{L_n}+\delta.
  \end{equation}
  Next, define the loss function
  \begin{equation*}
    \operatorname{loss}_n(y)\coloneqq\sum_{\ell=1}^{L_n}\left(\widetilde{Y}_{t_{\ell}}-y(t_{\ell})\right)^2+\lambda_n\int_{0}^{T}\ddot{y}(s)^2ds.
  \end{equation*}
  Moreover, it holds that
  \begin{align*}
    \operatorname{loss}_n(\hat{y}_b)
    &=\sum_{\ell=1}^{L_n}\left(\widetilde{Y}_{t_{\ell}}-\hat{y}_b(t_{\ell})\right)^2+\lambda_n\int_{0}^{T}\ddot{\hat{y}}_b(s)^2ds\\
    &=\operatorname{loss}_n(y_{b*})+\sum_{\ell=1}^{L_n}\left(y_{\lim}(t_{1})-\hat{y}_b(t_{1})\right)^2+2\left(y_{\lim}(t_{1})-\hat{y}_b(t_{1})\right)\sum_{\ell=1}^{L_n}\left(\widetilde{Y}_{t_{\ell}}-y_{b*}(t_{\ell})\right)\\
    &=\operatorname{loss}_n(y_{b*})+L_n\left(y_{\lim}(t_{1})-\hat{y}_b(t_{1})\right)^2+2\left(y_{\lim}(t_{1})-\hat{y}_b(t_{1})\right)\left[\sum_{\ell=1}^{L_n}\eps_{t_{\ell}}+\sum_{\ell=1}^{L_n}\left(Y_{t_{\ell}}-y_{b*}(t_{\ell})\right)\right]\\
  \end{align*}
  Now, $y_{b*}$ has the same derivatives as $\hat{y}_b$ and since
  $\hat{y}_b$ minimizes $\operatorname{loss}_n$ under fixed derivative
  constraints it holds that
  $\operatorname{loss}_n(\hat{y}_b)\leq\operatorname{loss}_n(y_{b*})$. This
  implies
  \begin{equation}
    \label{eq:yhatbpart12}
    L_n\left(y_{\lim}(t_{1})-\hat{y}_b(t_{1})\right)^2
    \leq
    2\left(y_{\lim}(t_{1})-\hat{y}_b(t_{1})\right)\left[\sum_{\ell=1}^{L_n}\eps_{t_{\ell}}+\sum_{\ell=1}^{L_n}\left(Y_{t_{\ell}}-y_{b*}(t_{\ell})\right)\right],
  \end{equation}
  which further implies
  \begin{equation}
    \label{eq:yhatbpart22}
    \abs{y_{\lim}(t_{1})-\hat{y}_b(t_{1})}
    \leq 2\cdot\abs[\bigg]{\frac{1}{L_n}\sum_{\ell=1}^{L_n}\eps_{t_{\ell}}+\frac{1}{L_n}\sum_{\ell=1}^{L_n}\left(Y_{t_{\ell}}-y_{b*}(t_{\ell})\right)}.
  \end{equation}
  Firstly, since we are on the set $B_{\delta}$ we get that
  \begin{equation}
    \label{eq:eps_part}
    \abs[\bigg]{\frac{1}{L_n}\sum_{\ell=1}^{L_n}\eps_{t_{\ell}}}\leq \delta.
  \end{equation}
  Secondly, using [\ref{eq:ybstarbdd2}] and the definition of the
  Riemann integral we get that
  \begin{align}
    &\limsup_{n\rightarrow\infty}\abs[\bigg]{\frac{1}{L_n}\sum_{\ell=1}^{L_n}\left(Y_{t_{\ell}}-y_{b*}(t_{\ell})\right)}\nonumber\\
    &\quad\leq
      \limsup_{n\rightarrow\infty}\abs[\bigg]{\frac{1}{L_n}\sum_{\ell=1}^{L_n}\left(Y_{t_{\ell}}-y_{\lim}(t_{\ell})\right)}+\limsup_{n\rightarrow\infty}\abs[\bigg]{\frac{1}{L_n}\sum_{\ell=1}^{L_n}\left(y_{b*}(t_{\ell})-y_{\lim}(t_{\ell})\right)}\nonumber\\
    &\quad\leq
      \abs[\bigg]{\int_{0}^T\left(Y_{s}-y_{\lim}(s)\right)ds}+\delta\nonumber\\
    &\quad=\delta,\label{eq:Yt_part}
  \end{align}
  where in the last step we used the definition of the function
  $y_{\lim}$. Hence, combining [\ref{eq:yhatbpart22}] with
  [\ref{eq:eps_part}] and [\ref{eq:Yt_part}] we get that
  \begin{equation}
    \label{eq:yhatbyhatcneeded12}
    \limsup_{n\rightarrow\infty}\abs{y_{\lim}(t_{1})-\hat{y}_b(t_{1})}
    \leq 4\delta.
  \end{equation}
  Furthermore, we can combine this with [\ref{eq:bddmaxyhatbc2}] to
  get that
  \begin{equation*}
    \limsup_{n\rightarrow\infty}\max_{\ell\in\{1,\dots,L_n\}}\abs{\hat{y}_b(t_{\ell})-y_{\lim}(t_{\ell})}
    \leq 5\delta,
  \end{equation*}
  which together with the global bound on the first derivative
  also implies that
  \begin{equation*}
    \limsup_{n\rightarrow\infty}\sup_{t\in[0,T]}\abs{\hat{y}_b(t)-y_{\lim}(t)}\leq
    \limsup_{n\rightarrow\infty}\left(\max_{\ell\in\{1,\dots,L_n\}}\abs{\hat{y}_b(t_{\ell})-y_{\lim}(t_{\ell})}+\tfrac{C}{L_n}\right)
    \leq 5\delta.
  \end{equation*}
  Since $\delta\in(0,1)$ was arbitrary this proves that
  $\sup_{t\in[0,T]}\abs{\hat{y}_b(t)-y_{\lim}(t)}$ converges in
  probability to zero, which completes the proof of Lemma~\ref{thm:limitfunction}.
\end{proof}

\subsubsection{Proof of theorem} In this section we prove
Theorem~\ref{thm:rank_consistency}.

\begin{proof}
  Assume that $Y_t$ and its first and second derivative be bounded by
  $c<\infty$ and define $C\coloneqq c+16$ for the set $\mathcal{H}_C$,
  see \ref{it:model1}.  The proof of
  Theorem~\ref{thm:rank_consistency} consists of two parts. First we
  assume that the following two claims are true and show that they
  suffice in proving the result.  Afterwards, we prove both claims.
  
  \textbf{Claim 1:} For all invariant
  $M\in\mathcal{M}$ it holds that
  \begin{equation*}
    \lim_{n\rightarrow\infty}\E\left(T_n(M)\right)=0
  \end{equation*}
  
  \textbf{Claim 2:} There exists a $c>0$ such that for all non-invariant $M\in\mathcal{M}$
  it holds that
  \begin{equation*}
    \liminf_{n\rightarrow\infty}\E\left(T_n(M)\right)\geq c.
  \end{equation*}

  Combining both claims and using Markov's inequality we get that
  \begin{align*}
    &\lim_{n\rightarrow\infty}\E\left(\big\vert\{M\in\mathcal{M}\,|\,
      T_n(M)<\max_{\{\tilde{M}\in\mathcal{M}\mid
      \tilde{M}\text{ invariant}\}}T_n(\tilde{M})\text{ and }
      M\text{ not invariant}\}\big\vert\right)\\
    &\qquad= \lim_{n\rightarrow\infty}\sum_{\substack{M\in\mathcal{M}:\\
    M\text{ not invariant}}}\E\left(\mathds{1}_{\{T_n(M)<\max_{\{\tilde{M}\in\mathcal{M}\mid
    \tilde{M}\text{
    invariant}\}}T_n(\tilde{M})\}}\right)\\
    &\qquad= \sum_{\substack{M\in\mathcal{M}:\\
    M\text{ not invariant}}}\lim_{n\rightarrow\infty}\prob\left(T_n(M)<\max_{\{\tilde{M}\in\mathcal{M}\mid
    \tilde{M}\text{
    invariant}\}}T_n(\tilde{M})\right)\\
    &\qquad= \sum_{\substack{M\in\mathcal{M}:\\
    M\text{ not invariant}}}\lim_{n\rightarrow\infty}\prob\left(\E\left(T_n(M)\right)<\max_{\{\tilde{M}\in\mathcal{M}\mid
    \tilde{M}\text{
    invariant}\}}T_n(\tilde{M})-T_n(M)+\E\left(T_n(M)\right)\right)\\
    &\qquad\leq \sum_{\substack{M\in\mathcal{M}:\\
    M\text{ not invariant}}}\lim_{n\rightarrow\infty}\prob\left(\E\left(T_n(M)\right)<\abs[\Big]{\max_{\{\tilde{M}\in\mathcal{M}\mid
    \tilde{M}\text{
    invariant}\}}T_n(\tilde{M})-T_n(M)+\E\left(T_n(M)\right)}\right)\\
    &\quad\overset{\text{Markov}}{\leq} \sum_{\substack{M\in\mathcal{M}:\\
    M\text{ not invariant}}}\lim_{n\rightarrow\infty}\frac{\E\left(\abs[\big]{\max_{\{\tilde{M}\in\mathcal{M}\mid
    \tilde{M}\text{
    invariant}\}}T_n(\tilde{M})-T_n(M)+\E\left(T_n(M)\right)}\right)}{\E\left(T_n(M)\right)}\\
    &\qquad\leq \sum_{\substack{M\in\mathcal{M}:\\
    M\text{ not invariant}}}\lim_{n\rightarrow\infty}\frac{\E\left(\abs[\big]{\max_{\{\tilde{M}\in\mathcal{M}\mid
    \tilde{M}\text{
    invariant}\}}T_n(\tilde{M})}\right)+\E\left(\abs{T_n(M)-\E\left(T_n(M)\right)}\right)}{\E\left(T_n(M)\right)}\\
    &\quad\overset{\text{claim 2}}{\leq} \sum_{\substack{M\in\mathcal{M}:\\
    M\text{ not invariant}}}\lim_{n\rightarrow\infty}\frac{\E\left(\abs[\big]{\max_{\{\tilde{M}\in\mathcal{M}\mid
    \tilde{M}\text{
    invariant}\}}T_n(\tilde{M})}\right)+\E\left(\abs{T_n(M)-\E\left(T_n(M)\right)}\right)}{c}\\
    &\quad\overset{\text{claim 1}}{=}0,
  \end{align*}
  which proves that
  $\lim_{n\rightarrow\infty}
  \E\left(\operatorname{RankAccuracy}_n\right)=1$. This result also
  proves the second part of Theorem~\ref{thm:rank_consistency}.  In
  the limit of infinitely many data points, any invariant model
  depends on all variables in $S^*$ (otherwise the set $S^*$ would not
  be unique, see \ref{it:cond3}).  Each variable $j \in S^*$ therefore
  receives a score of one.  On the other hand, any variable
  $j \notin S^*$ receives a score less or equal to $(K-1)/K$ since
  there exists at least one invariant model, namely the pair
  $S^*, g^*(\B{x}^{S^*})$ that does not depend on variable $j$.
  
  It therefore remains to prove claim~1 and claim~2.

  \textbf{Proof of claim 1:} Let $M\in\mathcal{M}$ be
  invariant and fix $k\in\{1,\dots,m\}$. 
In the remainder of this proof, the
  residual sum of square terms $\operatorname{RSS}_{a}^{(e_k)}$ and
  $\operatorname{RSS}_{b}^{(e_k)}$ depend on $n$, which will not be reflected in our notation. 
  First, observe that the triangle inequality
  implies that
  \begin{align}
    &\E\left(\abs{\operatorname{RSS}_{b}^{(e_k)}-\operatorname{RSS}_{a}^{(e_k)}}\right)
      \nonumber\\
    &\quad\leq\frac{1}{L_n}\sum_{\ell=1}^{L_n}\E\left(\abs{(\hat{y}_b^{(e_k)}(t_{\ell})-\widetilde{Y}^{(e_k)}_{t_{\ell}})^2-(\hat{y}_a^{(e_k)}(t_{\ell})-\widetilde{Y}^{(e_k)}_{t_{\ell}})^2}\right)\nonumber\\
    &\quad=\frac{1}{L_n}\sum_{\ell=1}^{L_n}\E\left(\abs{(\hat{y}_b^{(e_k)}(t_{\ell})-\hat{y}_a^{(e_k)}(t_{\ell}))(\hat{y}_b^{(e_k)}(t_{\ell})+\hat{y}_a^{(e_k)}(t_{\ell})-2\widetilde{Y}^{(e_k)}_{t_{\ell}})}\right)\nonumber\\
    &\quad=\frac{1}{L_n}\sum_{\ell=1}^{L_n}\E\left(\abs{[(\hat{y}_b^{(e_k)}(t_{\ell})-Y_{t_{\ell}}^{(e_k)})-(\hat{y}_a^{(e_k)}(t_{\ell})-Y_{t_{\ell}}^{(e_k)})][(\hat{y}_b^{(e_k)}(t_{\ell})-Y_{t_{\ell}}^{(e_k)})+(\hat{y}_a^{(e_k)}(t_{\ell})-Y_{t_{\ell}}^{(e_k)})-2\eps^{(e_k)}_{t_{\ell}})]}\right)\nonumber\\
    &\quad\leq\frac{1}{L_n}\sum_{\ell=1}^{L_n}\left[A(t_{\ell},k)+B(t_{\ell},k)+C(t_{\ell},k)+D(t_{\ell},k)+E(t_{\ell},k)\right],\label{eq:proof_RSSARSSB1}
  \end{align}
  where we used the following definitions
  \begin{align*}
    A(t_{\ell},k)&\coloneqq\E\left((\hat{y}_b^{(e_k)}(t_{\ell})-Y_{t_{\ell}}^{(e_k)})^2\right)\\
    B(t_{\ell},k)&\coloneqq\E\left((\hat{y}_a^{(e_k)}(t_{\ell})-Y_{t_{\ell}}^{(e_k)})^2\right)\\
    C(t_{\ell},k)&\coloneqq2\E\left(\abs{(\hat{y}_b^{(e_k)}(t_{\ell})-Y_{t_{\ell}}^{(e_k)})(\hat{y}_a^{(e_k)}(t_{\ell})-Y_{t_{\ell}}^{(e_k)})}\right)\\
    D(t_{\ell},k)&\coloneqq2\E\left(\abs{(\hat{y}_b^{(e_k)}(t_{\ell})-Y_{t_{\ell}}^{(e_k)})\eps_{t_{\ell}}^{(e_k)}}\right)\\
    E(t_{\ell},k)&\coloneqq2\E\left(\abs{(\hat{y}_a^{(e_k)}(t_{\ell})-Y_{t_{\ell}}^{(e_k)})\eps_{t_{\ell}}^{(e_k)}}\right).
  \end{align*}
  First, it holds that
  \begin{equation}
    \label{eq:convergenceAandB}
    \lim_{n\rightarrow\infty}\frac{1}{L_n}\sum_{\ell=1}^{L_n}A(t_{\ell},k)=0
    \quad\text{and}\quad
    \lim_{n\rightarrow\infty}\frac{1}{L_n}\sum_{\ell=1}^{L_n}B(t_{\ell},k)=0,
  \end{equation}
  where the first statement holds by the first part of
  Lemma~\ref{thm:consistency_yb} and the second by
  condition~\ref{it:cond1}. Together with the fact that the functions
  $\hat{y}_a^{(e_k)}\in\mathcal{H}_{C}$,
  $\hat{y}_b^{(e_k)}\in\mathcal{H}_{C}$ and
  $Y_{\cdot}^{(e_k)}\in\mathcal{H}_{C}$ it holds $\prob$-a.s. that
  \begin{equation}
    \label{eq:uniformbdd}
    \sup_{n\in\N}\sup_{t\in[0,T]}\abs{\hat{y}_a^{(e_k)}(t)-Y_{t}^{(e_k)}}\leq
    2C
    \quad\text{and}\quad
    \sup_{n\in\N}\sup_{t\in[0,T]}\abs{\hat{y}_b^{(e_k)}(t)-Y_{t}^{(e_k)}}\leq
    2C.
  \end{equation}
  Using the second statement together with condition \ref{it:cond1} we
  get that
  \begin{align*}
    \lim_{n\rightarrow\infty}\frac{1}{L_n}\sum_{\ell=1}^{L_n}C(t_{\ell},k)
    &=\lim_{n\rightarrow\infty}\frac{1}{L_n}\sum_{\ell=1}^{L_n}2\E\left(\abs{(\hat{y}_b^{(e_k)}(t_{\ell})-Y_{t_{\ell}}^{(e_k)})(\hat{y}_a^{(e_k)}(t_{\ell})-Y_{t_{\ell}}^{(e_k)})}\right)\\
    &\leq 4C\cdot\lim_{n\rightarrow\infty}\frac{1}{L_n}\sum_{\ell=1}^{L_n}\E\left(\abs{\hat{y}_a^{(e_k)}(t_{\ell})-Y_{t_{\ell}}^{(e_k)}}\right)\\
    &=0.
  \end{align*}
  Using both bounds in [\ref{eq:uniformbdd}], condition~\ref{it:cond1}
  and Lemma~\ref{thm:consistency_yb} we can apply
  Lemma~\ref{thm:triangular_convergence} to get that
  \begin{equation}
    \label{eq:convergenceCandD}
    \lim_{n\rightarrow\infty}\frac{1}{L_n}\sum_{\ell=1}^{L_n}D(t_{\ell},k)=0
    \quad\text{and}\quad
    \lim_{n\rightarrow\infty}\frac{1}{L_n}\sum_{\ell=1}^{L_n}E(t_{\ell},k)=0.    
  \end{equation}
  Hence, by taking the limit of [\ref{eq:proof_RSSARSSB1}], we have
  shown that
  \begin{equation}
    \label{eq:numerator}
    \lim_{n\rightarrow\infty}\E\left(\abs{\operatorname{RSS}_{b}^{(e_k)}-\operatorname{RSS}_{a}^{(e_k)}}\right)=0.
  \end{equation}
  Moreover, we can make the following decomposition.
  \begin{align}
    \E\left(\operatorname{RSS}_{a}^{(e_k)}\right)
    &=\frac{1}{L_n}\sum_{\ell=1}^{L_n}\E\left(\left(\hat{y}_a^{(e_k)}(t_{\ell})-\widetilde{Y}_{t_{\ell}}^{(e_k)}\right)^2\right)\nonumber\\
    &=\frac{1}{L_n}\sum_{\ell=1}^{L_n}\E\left(\left(\hat{y}_a^{(e_k)}(t_{\ell})-Y_{t_{\ell}}^{(e_k)}+Y_{t_{\ell}}^{(e_k)}-\widetilde{Y}_{t_{\ell}}^{(e_k)}\right)^2\right)\nonumber\\
    &=\frac{1}{L_n}\sum_{\ell=1}^{L_n}B(t_{\ell},k)+\frac{1}{L_n}\sum_{\ell=1}^{L_n}\E\left((\eps^{(e_k)}_{t_{\ell}})^2\right)+\frac{2}{L_n}\sum_{\ell=1}^{L_n}\E\left((\hat{y}_a^{(e_k)}(t_{\ell})-Y_{t_{\ell}}^{(e_k)})\eps^{(e_k)}_{t_{\ell}}\right).\label{eq:proof_RSSA1}
  \end{align}
  Using [\ref{eq:convergenceAandB}] and [\ref{eq:convergenceCandD}]
  and taking the limit of [\ref{eq:proof_RSSA1}] it holds that
  \begin{equation}
    \label{eq:denominator}
    \lim_{n\rightarrow\infty}\E\left(\operatorname{RSS}_{a}^{(e_k)}\right)=\sigma^2_k.
  \end{equation}
  Combining [\ref{eq:numerator}] and [\ref{eq:denominator}] with
  Slutsky's theorem this shows that
  $\frac{\abs{\operatorname{RSS}_{b}^{(e_k)}-\operatorname{RSS}_{a}^{(e_k)}}}{\operatorname{RSS}_{a}^{(e_k)}}\overset{\prob}{\longrightarrow}0$
  as $n\rightarrow\infty$. By [\ref{eq:uniformbdd}] and
  [\ref{eq:denominator}] it also holds that
  \begin{equation*}
    \sup_{n\in\N}\E\left(\left(\frac{\abs{\operatorname{RSS}_{b}^{(e_k)}-\operatorname{RSS}_{a}^{(e_k)}}}{\operatorname{RSS}_{a}^{(e_k)}}\right)^2\right)<\infty,
  \end{equation*}
  which together with de la Vallée-Poussin's theorem
  \cite[p.19 Theorem T22]{meyer1966} implies uniform
  integrability and thus $L^1$ convergence, i.e.,
  \begin{equation*}
    \lim_{n\rightarrow\infty}\E\left(\frac{\abs{\operatorname{RSS}_{b}^{(e_k)}-\operatorname{RSS}_{a}^{(e_k)}}}{\operatorname{RSS}_{a}^{(e_k)}}\right)=0.
  \end{equation*}
  Finally, since the number of environments $m$ is fixed and it holds
  for all $i\in e_k$ that
  \begin{equation}
    \label{eq:dist_equal}
    \frac{\abs{\operatorname{RSS}_{b}^{(i)}-\operatorname{RSS}_{a}^{(i)}}}{\operatorname{RSS}_{a}^{(i)}}
    \overset{d}{=}\frac{\abs{\operatorname{RSS}_{b}^{(e_k)}-\operatorname{RSS}_{a}^{(e_k)}}}{\operatorname{RSS}_{a}^{(e_k)}},
  \end{equation}
  it holds that
  \begin{align*}
    \lim_{n\rightarrow\infty}\E\left(T(M)_n\right)
    &=\lim_{n\rightarrow\infty}\frac{1}{n}\sum_{i=1}^n\E\left(\frac{\abs{\operatorname{RSS}_{b}^{(i)}-\operatorname{RSS}_{a}^{(i)}}}{\operatorname{RSS}_{a}^{(i)}}\right)\\
    &=\lim_{n\rightarrow\infty}\frac{1}{m}\sum_{k=1}^m\E\left(\frac{\abs{\operatorname{RSS}_{b}^{(e_k)}-\operatorname{RSS}_{a}^{(e_k)}}}{\operatorname{RSS}_{a}^{(e_k)}}\right)=0.
  \end{align*}
  This completes the proof of claim~1.

    \textbf{Proof of claim 2:} Let $M\in\mathcal{M}$ be
  non-invariant. Let $k^*\in\{1,\dots,m\}$ be the index 
  and $c_{\min}$ the constant 
  for which
  [\ref{eq:lemma_lowerbound}] in Lemma~\ref{thm:consistency_yb} is
  satisfied.
For every $\delta>0$ define the following sets
  \begin{align*}
    A_{\delta}&\coloneqq\left\{\abs[\Big]{\frac{1}{L_n}\sum_{\ell=1}^{L_n}(\hat{y}_a^{(e_{k^*})}(t_{\ell})-Y_{t_{\ell}}^{(e_{k^*})})^2}+\abs[\Big]{\frac{2}{L_n}\sum_{\ell=1}^{L_n}(\hat{y}_a^{(e_{k^*})}(t_{\ell})-Y_{t_{\ell}}^{(e_{k^*})})\eps_{t_{\ell}}^{(e_{k^*})}}\leq
                \delta\right\}\\
    B_{\delta}&\coloneqq\left\{\abs[\Big]{\frac{1}{L_n}\sum_{\ell=1}^{L_n}\left(\eps_{t_{\ell}}^{(e_{k^*})}\right)^2-\sigma_{k^*}}\leq
                \delta\right\}\\
    C_{\delta}&\coloneqq\left\{\abs[\Big]{\frac{1}{L_n}\sum_{\ell=1}^{L_n}(\hat{y}_b^{(e_{k^*})}(t_{\ell})-Y_{t_{\ell}}^{(e_{k^*})})\eps_{t_{\ell}}^{(e_{k^*})}}\leq
               \delta\right\}\\
    D_{\delta}&\coloneqq\left\{\frac{1}{L_n}\sum_{\ell=1}^{L_n}(\hat{y}_b^{(e_{k^*})}(t_{\ell})-Y_{t_{\ell}}^{(e_{k^*})})^2\geq
               c_{\min} - \delta\right\}.
  \end{align*}
  Using that both summands in the definition of $A_{\delta}$ converge
  in $L^1$ (this follows in exactly the same way, we obtained
  [\ref{eq:convergenceAandB}] and [\ref{eq:convergenceCandD}]) it
  holds that the sum convergences in probability. This in particular
  implies that there exists $N_A\in\N$ such that for all
  $n\in\{N_A,N_A+1,\dots\}$ it holds that
  \begin{equation}
    \label{eq:termAset}
    \prob\left(A_{\delta}\right)\geq 1-\delta.
  \end{equation}
  Next, by the law of large numbers it holds that
  $\frac{1}{L_n}\sum_{\ell=1}^{L_n}\left(\eps_{t_{\ell}}^{(e_{k^*})}\right)^2$
  converges to $\sigma_{k^*}^2$ in probability. This implies that
  there exists $N_B\in\N$ such that for all
  $n\in\{N_B,N_B+1,\dots\}$ it holds that
  \begin{equation}
    \label{eq:termBset}
    \prob\left(B_{\delta}\right)\geq 1-\delta.
  \end{equation}
  Finally, observe that since $\eps_{t_{\ell}}^{(e_{k^*})}$ has mean zero it
  holds that
  \begin{equation*}
    \E\left((\hat{y}_b^{(e_{k^*})}(t_{\ell})-Y_{t_{\ell}}^{(e_{k^*})})\eps^{(e_{k^*})}_{t_{\ell}}\right)=\E\left((\hat{y}_b^{(e_{k^*})}(t_{\ell})-y_{\lim}^{(e_{k^*})}(t_{\ell}))\eps^{(e_{k^*})}_{t_{\ell}}\right),
  \end{equation*}
  where $y_{\lim}^{(e_{k^*})}$ is the limit function given in
  Lemma~\ref{thm:limitfunction}. The statement of
  Lemma~\ref{thm:limitfunction} together with the boundedness of the
  functions allows us to apply Lemma~\ref{thm:triangular_convergence}
  to get that
  \begin{equation*}
    \lim_{n\rightarrow\infty}\frac{2}{L_n}\sum_{\ell=1}^{L_n}\E\left|\left(\hat{y}_b^{(e_{k^*})}(t_{\ell})-Y_{t_{\ell}}^{(e_{k^*})}\right)\eps^{(e_{k^*})}_{t_{\ell}}\right|=0.
  \end{equation*}
  Hence, this term also converges in probability and thus there exists
  $N_C\in\N$ such that for all $n\in\{N_C,N_C+1,\dots\}$ it holds that
  \begin{equation}
    \label{eq:termCset}
    \prob\left(C_{\delta}\right)\geq 1-\delta.
  \end{equation}
  Finally, applying Lemma~\ref{thm:consistency_yb} there exists
  $N_D\in\N$ such that for all $n\in\{N_D,N_D+1,\dots\}$ it holds that
  \begin{equation}
    \label{eq:termDset}
    \prob\left(D_{\delta}\right)\geq 1-\delta.
  \end{equation}
  Combining [\ref{eq:termAset}], [\ref{eq:termBset}],
  [\ref{eq:termCset}] and [\ref{eq:termDset}] we get for all
  $n\in\{N^{\max},N^{\max}+1,\dots\}$ with $N^{\max}\coloneqq\max\{N_A,N_B,N_C,N_D\}$ that
  \begin{align*}
    \E\left(\frac{\abs{\operatorname{RSS}_b^{(e_{k^*})}-\operatorname{RSS}_a^{(e_{k^*})}}}{\operatorname{RSS}_a^{(e_{k^*})}}\right)
    &\geq\E\left(\frac{\operatorname{RSS}_b^{(e_{k^*})}}{\operatorname{RSS}_a^{(e_{k^*})}}\right)-1\\
    &\geq\E\left(\frac{\operatorname{RSS}_b^{(e_{k^*})}}{\operatorname{RSS}_a^{(e_{k^*})}}\mathds{1}_{A_{\delta}}\mathds{1}_{B_{\delta}}\mathds{1}_{C_{\delta}}\mathds{1}_{D_{\delta}}\right)-1\\
    &
    \geq\E\left(\frac{c_{\min}-\delta-\delta+\sigma^2_{k^*}-\delta}{2\delta+\sigma^2_{k^*}+\delta}\mathds{1}_{A_{\delta}}\mathds{1}_{B_{\delta}}\mathds{1}_{C_{\delta}}\mathds{1}_{D_{\delta}}\right)-1\\
    &=\frac{c_{\min}-3\delta+\sigma^2_{k^*}}{3\delta+\sigma^2_{k^*}}\prob\left(A_{\delta} \cap B_{\delta} \cap C_{\delta} \cap D_{\delta}\right)-1\\
    &\geq\frac{c_{\min}-3\delta+\sigma^2_{k^*}}{3\delta+\sigma^2_{k^*}}(1-4\delta)-1,
  \end{align*}
  where for the third inequality we used the expansion
  \begin{equation*}
    \operatorname{RSS}_{*}^{(e_{k^*})}=\frac{1}{L_n}\sum_{\ell=1}^{L_n}(\hat{y}_{*}^{(e_{k^*})}-Y^{(e_{k^*})}_{t_{\ell}})^2-\frac{2}{L_n}\sum_{\ell=1}^{L_n}(\hat{y}_{*}^{(e_{k^*})}-Y^{(e_{k^*})}_{t_{\ell}})\eps_{t_{\ell}}^{(e_{k^*})}+\frac{1}{L_n}\sum_{\ell=1}^{L_n}(\eps_{t_{\ell}}^{(e_{k^*})})^2
  \end{equation*}
  together with the normal and reverse triangle inequality and the
  definitions of the sets $A_{\delta}$, $B_{\delta}$, $C_{\delta}$ and
  $D_{\delta}$. Since $\delta$ was arbitrary we can let $\delta$ tend
  to zero which implies that
  \begin{equation*}
    \liminf_{n\rightarrow\infty}\E\left(\frac{\abs{\operatorname{RSS}_b^{(e_{k^*})}-\operatorname{RSS}_a^{(e_{k^*})}}}{\operatorname{RSS}_a^{(e_{k^*})}}\right)\geq\frac{c_{\min}}{\sigma^2_{k^*}}.
  \end{equation*}
  Finally, using this together with [\ref{eq:dist_equal}] we get that
  \begin{align*}
    \liminf_{n\rightarrow\infty}\E\left(T_n(M)\right)
    &=
      \liminf_{n\rightarrow\infty}\frac{1}{n}\sum_{i=1}^n\E\left(\frac{\abs{\operatorname{RSS}_b^{(i)}-\operatorname{RSS}_a^{(i)}}}{\operatorname{RSS}_a^{(i)}}\right)\\
    &\geq
      \liminf_{n\rightarrow\infty}\frac{1}{m}\sum_{k=1}^m\E\left(\frac{\abs{\operatorname{RSS}_b^{(e_{k^*})}-\operatorname{RSS}_a^{(e_{k^*})}}}{\operatorname{RSS}_a^{(e_{k^*})}}\right)\\
    &\geq \frac{c_{\min}}{\sigma^2_{k^*}}>0,
  \end{align*}
  which completes the proof of claim~2 and also completes the proof of
  Theorem~\ref{thm:rank_consistency}.
\end{proof}

\section{Extended empirical results}\label{sec:extended_simulations}

This supplementary note contains detailed further empirical results
intended to complement the article \emph{Learning stable and
  predictive structures in kinetic systems}.  In particular, it
includes experiments on identification of causal predictors, large
sample experiments supporting the theoretical results on consistency,
scalability experiments with a large number of variables and
experiments on robustness against model
misspecifications. Additionally, it includes some auxiliary results
for the metabolic pathway analysis.

\subsection{Competing methods}\label{sec:competing_methods}
A large number of methods have been proposed to perform model
selection for ODEs. Essentially, these can be group into two major
categories: (i) methods that model the underlying ODE and (ii)
approximate methods that assume a simplified underlying structure,
e.g., a dependence graph.

CausalKinetiX belongs to group (i). In essence, these types of methods
combine parameter inference with classical model selection techniques
such as information criteria, e.g., AIC or BIC, or $\ell^1$-penalized
approaches. Apart from CausalKinetiX, they optimize predictive
performance of the resulting model and do not make use of any
heterogeneity in the data. In this paper, we compare with two basic
$\ell^1$-penalized approaches (see below) that we regard as
representative for this group of methods. The first method performs
the regularization on the level of the derivatives (gradient matching
GM) and the second on the integrated problem (difference matching
DM). Both are common in literature \citep[e.g.][]{Wu2014,
  brunton2016discovering} and can also be used as screening procedures
in our method. Furthermore, we also compare with a more involved
method called adaptive integral matching (AIM) introduced by
\cite{mikkelsen2017}. Rather than only fitting the target equation it
fits an entire system of ODEs on all variables, hence utilizing
information shared across different variables.

Category (ii) neglects the underlying ODE structure and assumes it can
be approximated by a dependency model between the
variables. Essentially, any type of graphical model procedure can be
applied. While these methods are often a lot faster than methods that
take into account the underlying ODE structure, they often lead to
rather poor results, due to the model misspecification. We consider
several different Bayesian network based methods: BN-PC, BN-GES,
DBN-CondInd and DBN-Greedy. Here, BN and DBN stand for Bayesian
network and dynamic Bayesian network \cite{Koller09},
respectively. BN-PC learns the graph structure using the well-known
PC-algorithm \cite{Spirtes2000} based on conditional independence and
BN-GES uses the greedy equivalent search algorithm
\cite{Chickering2002}. Both dynamic Bayesian network methods are
implemented in the R-package \texttt{bnlearn} \citep{bnlearn}, where
DBN-CondInd uses a conditional independence based algorithm and
DBN-Greedy a greedy score based algorithm. Each of these methods is
applied to the measurements directly by pooling across different
experimental settings. To get a more complete picture of the
performance of these algorithms we additionally consider modified
versions where we apply them to $(Y_t-Y_{t-1}, X_{t-1})$ instead of
$(Y_t, X_t)$ in order to make the linear assumption more likely to be
satisfied. We denote these versions by adding ``(diffY)'' to the end
of the name. Moreover, both DBN methods are restricted to feed-forward
edges and in particular do not involve instantaneous
edges. Additionally, we included an in-degree constraint for both
DBN-Greedy and BN-GES by constraining the number of parents for each
node to be at most $4$ (this is also a constraint used in
CausalKinetiX).

\paragraph{Penalized gradient matching (GM)}
This method can be used whenever the considered model is linear in its
parameters, e.g., models of mass-action type as in
[\ref{eq:ydotmassaction}]. It fits a smoother to each trajectory of
the target variable $Y$ and computes the corresponding
derivatives. Then, one fits an $\ell^1$-penalized sparse linear model
\cite{tibshirani2015statistical, donoho2006, candes2006} on these
estimated derivatives
\begin{equation*}
  \dYest{i}{t_{\ell}}
  =\sum_{j=1}^d\sum_{k=j}^d \theta_{k,l} X^{k,(i)}_{t_{\ell}}X^{j,(i)}_{t_{\ell}}+\eps^{(i)}_{t_{\ell}},
\end{equation*}
where $\eps^{(i)}_{t_{\ell}}$ are assumed independent and
identically distributed Gaussian noise variables and the regression
coefficient $\theta$ is assumed to be sparse, i.e., the loss
function has the form
\begin{equation*}
  l(\theta, X,
  Y)=\sum_{\ell=1}^{L}\norm{\dY{i}{t_{\ell}}-\dYest{i}{t_{\ell}}}^2_2+\lambda\norm{\theta}_1.
\end{equation*}
This results in a ranking of terms $X^{k,(i)}X^{j,(i)}$ by when they
enter the model (i.e. non-zero $\theta$ coefficient) for the first
time.

\paragraph{Penalized difference matching (DM)}
The GM method can be adapted by integrating the linear model to avoid
estimating the often numerically unstable derivatives. In this case,
one fits a $\ell^1$-penalized sparse linear model on the difference of
the form
\begin{align*}
  \widehat{Y}_{t_{\ell}}^{(i)}-\widehat{Y}_{t_{\ell-1}}^{(i)}
  &=\sum_{{j=1, k=j}}^d\theta_{k,l}\int_{t_{\ell-1}}^{t_{\ell}}X_s^{k,(i)}X_s^{j,(i)}ds+\eps^{(i)}_{t_{\ell}}\\
  \approx\sum_{{j=1,k=j}}^d &\theta_{k,l}\tfrac{X_{t_{\ell}}^{k,(i)}X_{t_{\ell}}^{j,(i)}+X_{t_{\ell-1}}^{k,(i)}X_{t_{\ell-1}}^{j,(i)}}{2}(t_{\ell}-t_{\ell-1})+\eps^{(i)}_{t_{\ell}}
\end{align*}
where, similarly, $\eps^{(i)}_{t_{\ell}}$ are assumed independent and
identically distributed Gaussian noise and the regression coefficient
$\theta$ is assumed to be sparse. Again, one obtains a ranking of the
term $X^{k,(i)}X^{j,(i)}$ depending on when they first enter the
model.  In our numerical simulation study in
Section~\ref{sec:screen_sim}, DM performs better than GM.
Intuitively, this is the case whenever the dynamics are hard to detect
due to noise as the estimated derivatives will then have strongly
time-dependent biases. Similar observations have been made before
\cite{chen2017network}.

\subsection{Simulation experiments}\label{sec:simulation_exps}
We perform experiments on three different ODE models.  The first is a
biological model of the Maillard reaction \citep{maillard1912}, the
second and third models are artificially constructed ODE models.  The
relatively small sizes of these systems ($d \leq 12$) allow for fast
data simulation which enables us to compare the performance under
various settings and conditions.

\subsubsection{Finding causal predictors in the Maillard reaction}\label{sec:maillard_reaction}
The first simulation study is based on a biological ODE system from
the \emph{BioModels Database} due to \cite{li2010biomodels}. More
specifically, we use the model BIOMD0000000052 due to
\cite{brands2002} which describes reactions in heated
monosaccharide-casein systems.  This system is relatively small (11
variables), it consists entirely of mass-action type equations, and it
remains stable under various random interventions (that we can use to
simulate different experimental conditions). The simulation setup is
described in \hyperlink{dat:maillard}{Data Set 1}, additional details
can be found in Section~\ref{sec:additionbiomod52}.

\begin{mdframed}[roundcorner=5pt,
  frametitle={\hypertarget{dat:maillard}{Data Set 1}: Maillard reaction}]
  The ODE structure is given in Section~\ref{sec:biomodel52}. For the
  simulations, we randomly select one of the $d=11$ variables to be
  the target and generate data from $5$ experimental settings and
  sample $3$ repetitions for each experiment.  The experimental
  conditions are as follows.
  \begin{compactitem}
  \item \textbf{Experimental condition $1$ (observational data):}\\
    Trajectories are simulated using the parameters given in
    BIOMD0000000052, see Section~\ref{sec:biomodel52}.
  \item \textbf{Experimental conditions $2$ to $5$ (interventional data):}\\
    Trajectories are simulated based on the following two types of interventions
    \begin{compactitem}
    \item \textbf{initial value intervention:} Initial values are
      sampled for $[\text{Glu}]$ and
    $[\text{Fru}]$ uniform between $0$ and $5\cdot160$ and for $[\text{lys R}]$
    uniform between $0$ and $5\cdot15$, the remainder of the quantities are
    kept at zero initially as they are all products of the reactions.
    \item \textbf{blocking reactions:} Random reactions that do not
      belong to the target equations are set to zero by fixing the
      corresponding reaction constant $k_i\equiv 0$. The expected
      number of reactions set to zero is $3$.
    \end{compactitem}
  \end{compactitem}
  Based on these experimental conditions each of the true model
  trajectories are computed using numerical integration. Finally, the
  observations are given as noisy versions of the values of these
  trajectories on a quadratic time grid with $L$ time points between
  $0$ and $100$. For most experiments we use $L=11$ time points. The
  noise at each observation is independently normal distributed with
  mean $0$ and standard deviation proportional to the total variation
  norm of the trajectory plus a small positive constant (in case the
  trajectory is constant), i.e.,
  $\sigma=c\cdot\norm{y}_{\operatorname{TV}}+10^{-7}$, where $y(t)$ are the
  true trajectories. For all experiments apart from the one shown in
  Figure~3 \textbf{b} in the main article, we take
  $c\sim\operatorname{Unif}(0.01,0.1)$. Sample trajectories are given
  in Figure~\ref{fig:sample_trajectories_bio52}.
\end{mdframed}

As a first assessment of our method, we analyze its ability to recover
causal structure. To do so, we sample $B=500$ realizations of the
system described in \hyperlink{dat:maillard}{Data Set~1} and apply our
method as well as the competing methods to rank the variables
according to which variable is most likely to be a parent variable of
the target. To remove any effect resulting from ordering of the
variables we relabel them in each repetition by randomly permuting the
labels. Each ranking is then assessed by computing the area under the
operator receiver curve (AUROC) based on the known ground truth (i.e.,
parents of the target $\PA[]{Y}$). The results are given in
Figure~\ref{fig:simulation_biomodel}. Here, we applied CausalKinetiX
using the exhaustive model class (see \ref{sec:parametric_models}
Section~\ref{sec:mass-action}) and considered all possible models
consisting of individual variables and interactions ($66$ potential
predictor terms). We restricted the search to models with at most $4$
such terms after reducing the number of terms by a prior screening
step to $33$. The results show that our method can improve on all
competing methods. In particular, we are able to get a median AUROC of
$1$ implying that in more than half of all repetitions our method
ranks the correct models first. Moreover, by comparing with DM one can
see that utilizing the heterogeneity (via the stability score) does
indeed improve on plain prediction based methods. The results in
Figure~3 in the main article are based on the same experiment (for the
simulation shown in Figure~3 \textbf{b} we screened to $20$ terms to
decrease the computational burden).

\begin{figure}[h]
  \centering
  \includegraphics{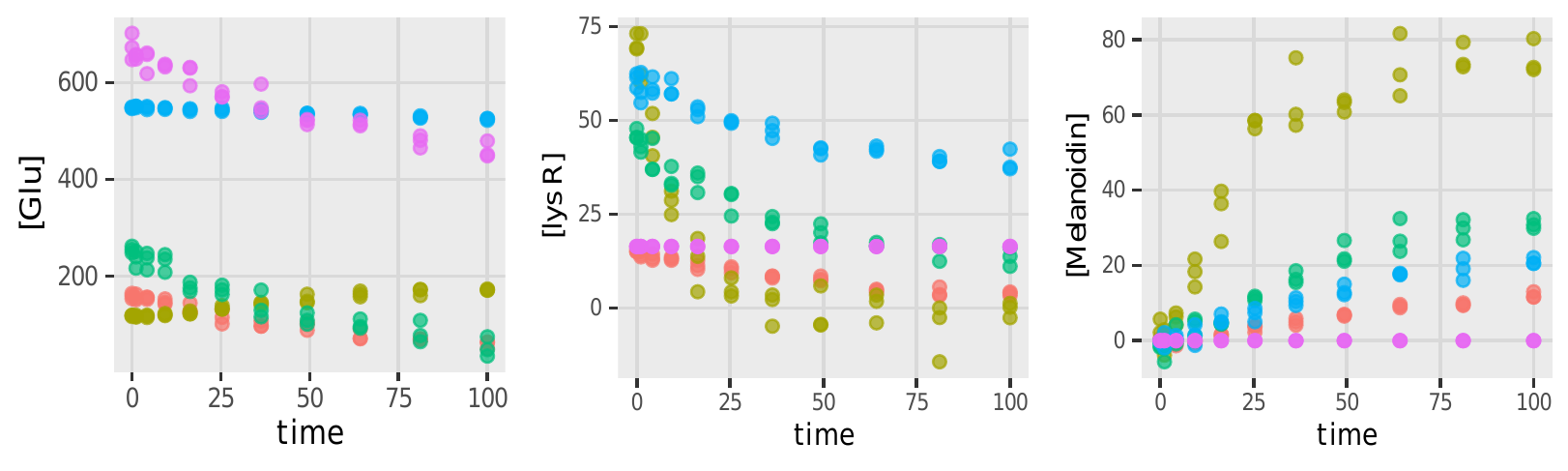}
  \caption{Sample observations (the method's input) for the variables
    Glu, lys R and Melanoidin from
    \protect\hyperlink{dat:maillard}{Data Set~1}. Points represent
    noisy observations with different colors for the $5$ different
    experimental conditions, e.g., red corresponds to experimental
    condition~1.}
  \label{fig:sample_trajectories_bio52}
\end{figure}

\begin{figure}[h!]
  \includegraphics[width = 0.99\textwidth]{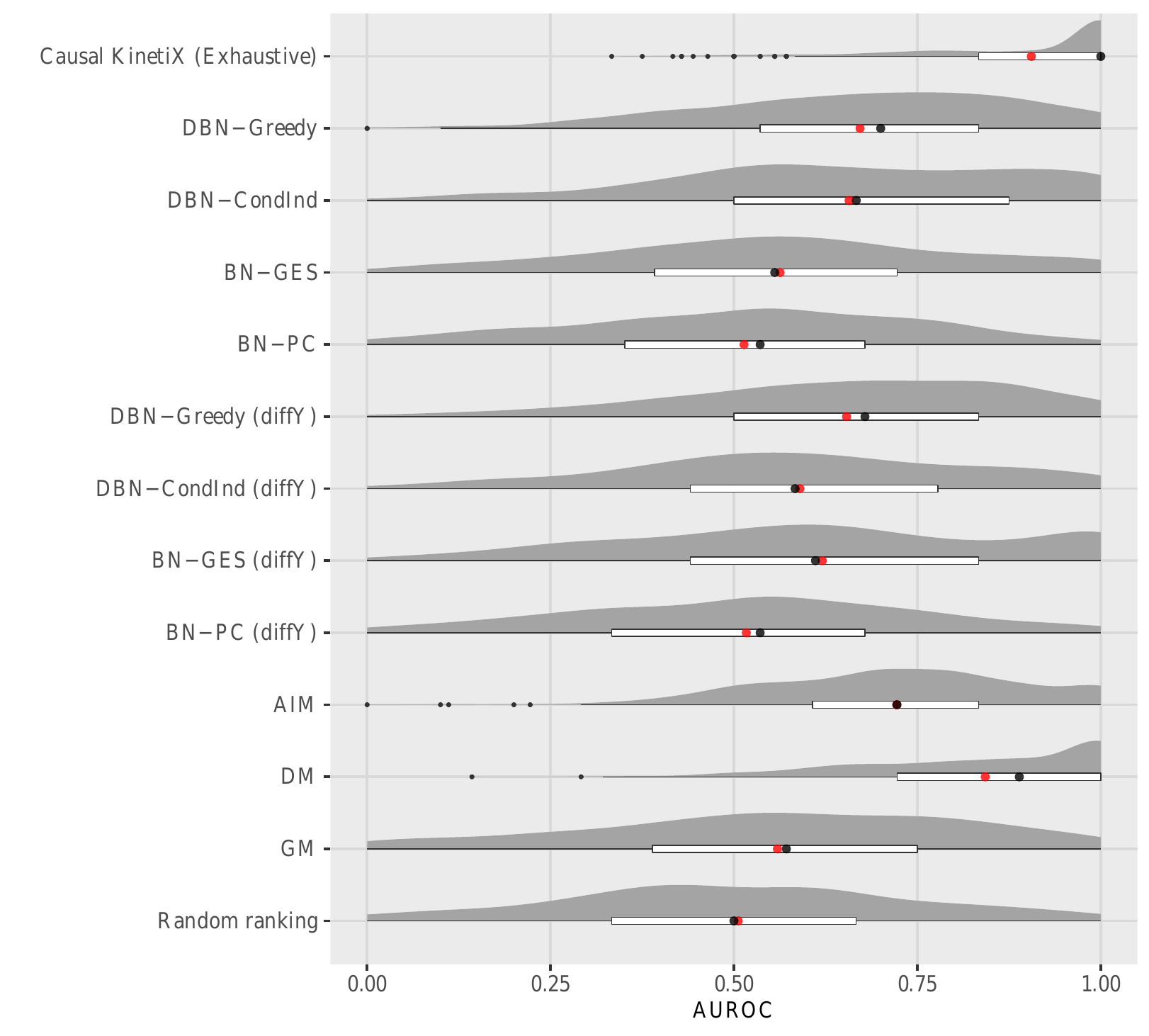}
  \caption{Results for simulation in
    Section~\ref{sec:maillard_reaction}.  In each of the $500$
    simulations, the methods rank predictors for a randomly chosen
    target. If the AUROC equals one, the correct variables are ranked
    highest. Red points correspond to mean AUROC, black points to
    median AUROC.}
  \label{fig:simulation_biomodel}
\end{figure}

\FloatBarrier

\subsubsection{Overfitting of the variable ranking}\label{sec:overfitting_vars}
We now demonstrate that incorporating stability as a learning
principle also helps in terms of overfitting. We again use data that
were simulated according to \hyperlink{dat:maillard}{Data Set~1}. We
compare CausalKinetiX with a modified version that does not hold out
experiments in step (M4) of the procedure. This version therefore
focuses on prediction and mostly neglects stability.  We now compare
the performance of the two procedures while increasing the number of
terms in the model classes we search over. As can be seen from
Figure~\ref{fig:simulation_overfitting_vars} the decrease in AUROC is
stronger for the method that focuses less on stability. We regard this
as evidence that stability as an inference principle indeed helps
against overfitting in these types of models (see also
Section~\ref{sec:overfitting_traj}).

\begin{figure}[h]
  \centering
  \includegraphics{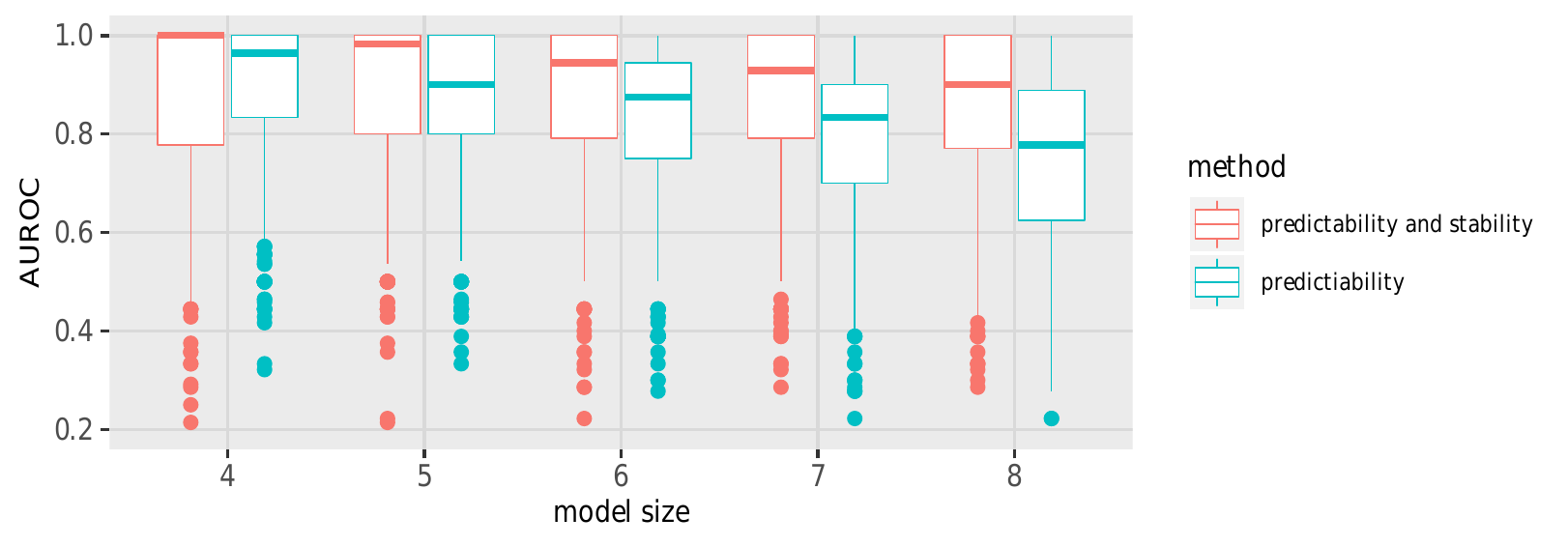}
  \caption{Results illustrating how CausalKinetiX behaves when model
    class is increased. Enforcing stability results in better
    regularization even on large model classes.}
  \label{fig:simulation_overfitting_vars}
\end{figure}

\FloatBarrier

\subsubsection{Effect of time-dependent measurement noise}

The data simulated according to \hyperlink{dat:maillard}{Data Set~1}
is corrupted with time-independent measurement noise. However, we do
not regard the assumption of i.i.d.\ measurement noise as crucial for
our procedure: The noise becomes most relevant during the smoothing
step of CausalKinetiX.  As long as the distribution of the noise does
not negatively impact the smoothing, the output of CausalKinetiX will
be unaffected. Furthermore, even if there are biases induced in the
smoother due to time-dependent noise, they do not necessarily harm
CausalKinetiX as long as the dynamics are still captured to a
satisfactory degree.  This is why we expect CausalKinetiX to be robust
with respect to time-dependent measurement noise. We verify this
empirically by performing the same experiment as in
Section~\ref{sec:maillard_reaction} but replacing the independent
Gaussian noise with auto-regressive noise of the form
\begin{equation*}
  \eps_{t}=a\cdot\eps_{t-1} + W_t,
\end{equation*}
where $a\in(-1,1)$ and $W_t$ is standard normal noise. (We furthermore
screen down to $15$ terms to decrease the computation burden of the
experiment.)  The results for different values of $a$ are given in
Figure~\ref{fig:time_dependent_noise} and illustrate that there is
indeed no negative effect of time-dependent noise in this case.

\begin{figure}[h]
  \centering
  \includegraphics{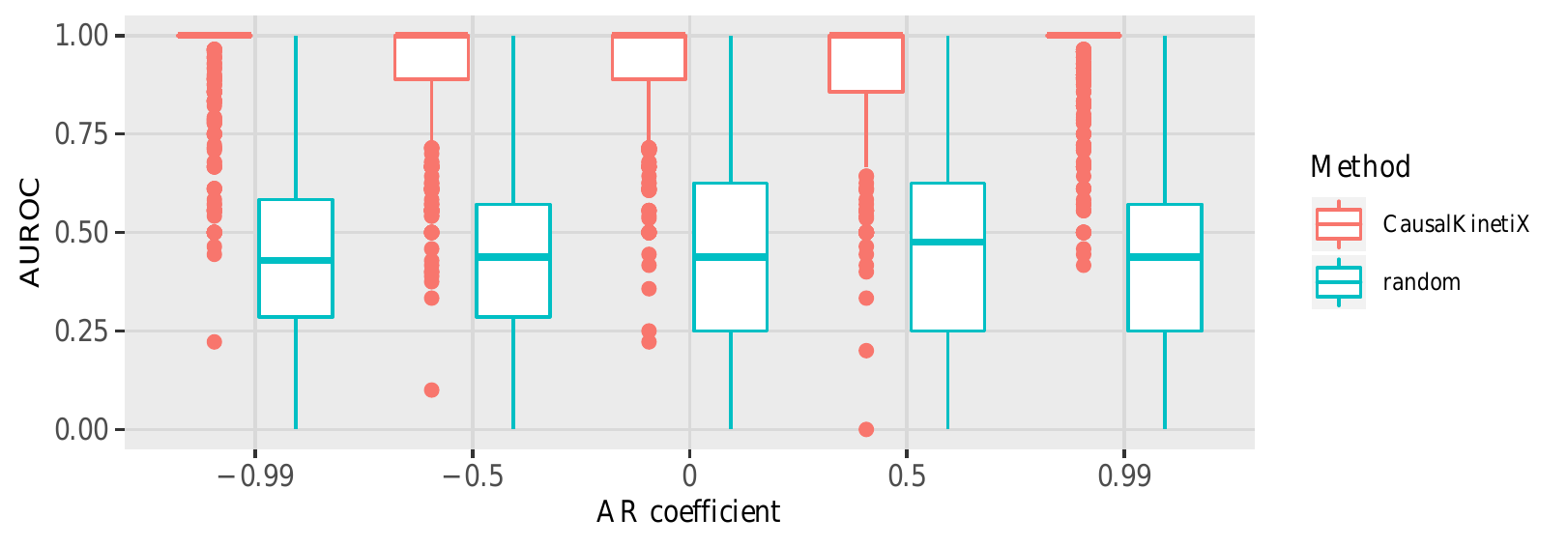}
  \caption{Results illustrating the affect of time-dependence in the
    measurement noise on the variable ranking performance of
    CausalKinetiX. Added time-dependence has no negative effect on the
    performance, which is expected since CausalKinetiX is based on
    dynamics rather than the trajectories directly.}
  \label{fig:time_dependent_noise}
\end{figure}

\FloatBarrier

\subsubsection{Comparison of screening procedures}\label{sec:screen_sim}
Using data generated as in \hyperlink{dat:maillard}{Data Set~1}, we
compare the two screening methods based on GM and DM. To this end, we
sample $B=1000$ data sets and apply both DM and GM to rank all
$11\cdot 10 \cdot 0.5 + 22 = 77$ individual terms of the form $X^kX^j$
and $X^j$ ($d = 11$) based on their first entrance into the model. For
each data set, we then compute the worst rank of any true term and
plot them in Figure~\ref{fig:histo_screening}. For comparison, we also
include the results from a random ranking, i.e., a random permutation
of the terms.  Both methods perform better than the random baseline,
and DM outperforms GM in this setting.  This might be because the
integral approximation used in DM is more robust than the estimation
of the derivatives required for GM.
\begin{figure}[t]
  \centering
  \includegraphics{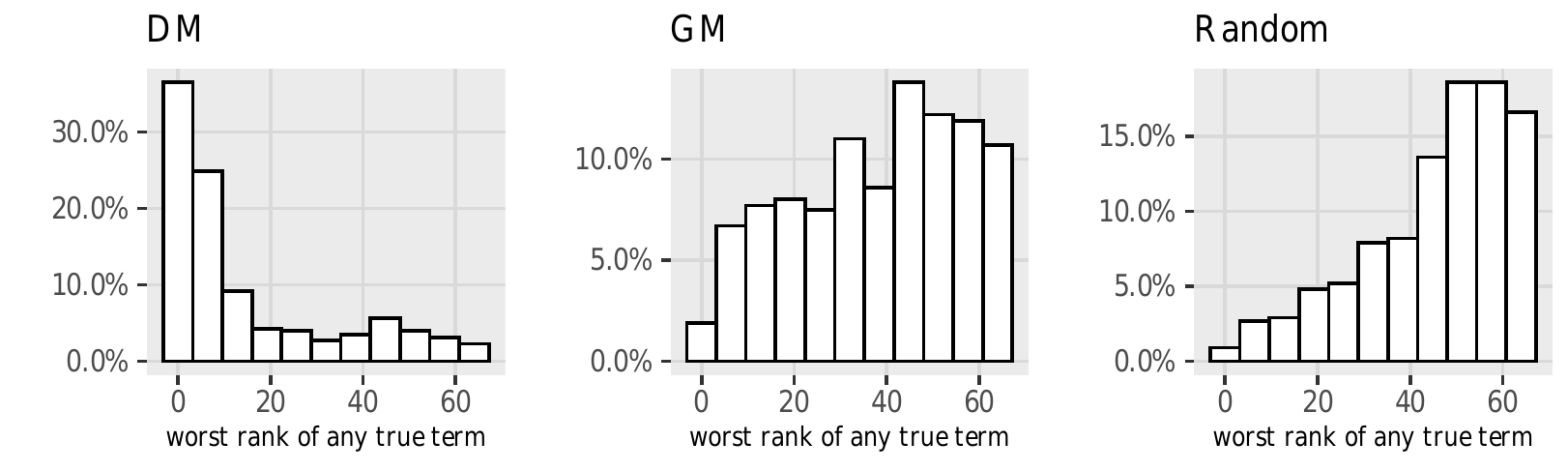}
  \caption{Comparison between different screening methods based on
    $B=1000$ simulations from \protect\hyperlink{dat:maillard}{Data
      Set~1} in Section~\ref{sec:maillard_reaction}. All $77$ terms of
    the form $X^kX^j$ and $X^j$ are ranked according to the screening
    procedure.  The x-axis shows the rank of the worst ranked term
    from the true model. High concentration on the left implies good
    screening performance.  Here, DM outperforms
    GM.}
  \label{fig:histo_screening}
\end{figure}

\FloatBarrier

\subsubsection{Consistency analysis}\label{sec:consistency_exp}
We now illustrate our theoretical consistency result presented in the
Methods section and Section~\ref{sec:parametric_models}. Again, we
simulate from \hyperlink{dat:maillard}{Data Set~1}, where we consider
different values of $L$ and $n$ to analyze the asymptotic
behavior. Instead of increasing the value of $n$ we decrease the noise
variance as this has a similar effect but is computationally
faster. Since computing the RankAccuracy requires knowledge of
the invariant sets, we use a setting with many experimental conditions
(here, $10$ experimental conditions) and assume that in this limit only
super-sets of the true model are invariant. The results shown in
Figure~\ref{fig:consistency} demonstrate the convergence of the
RankAccuracy towards one as the number of time steps $L$ goes to
infinity and the noise variance goes to zero.

\begin{figure}[h]
  \centering
  \includegraphics{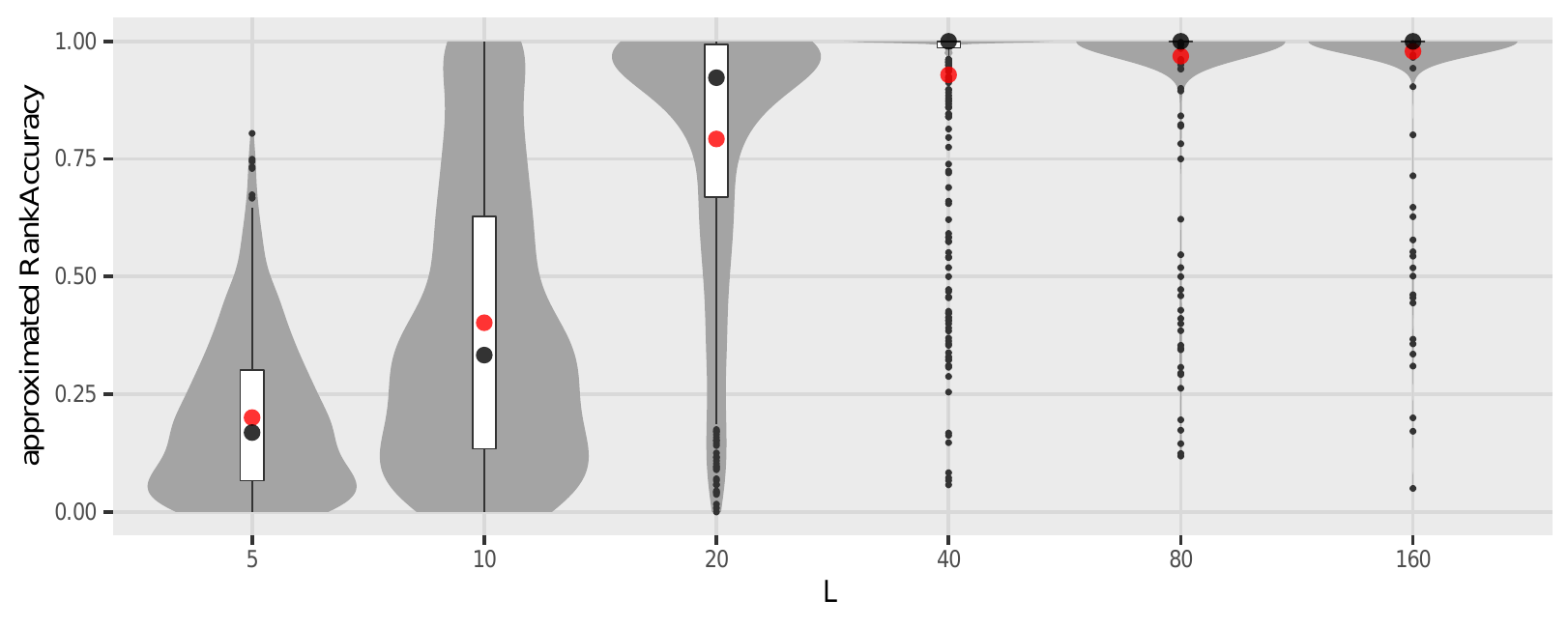}
  \caption{Results for simulation in
    Section~\ref{sec:consistency_exp}. For different numbers of time
    points $L$ and noise variance proportional~$\frac{10}{L^2}$ we
    sampled $500$ simulations from
    \protect\hyperlink{dat:maillard}{Data Set~1}. For each simulation
    we compute the $\operatorname{RankAccuracy}$. Red points
    correspond to mean $\operatorname{RankAccuracy}$, large black points to
    median $\operatorname{RankAccuracy}$.}
  \label{fig:consistency}
\end{figure}

\FloatBarrier

\subsubsection{Increasing experimental
  conditions}\label{sec:increasing_exp}

In this section, we illustrate how an increase in experimental
conditions affects the variable ranking performance of
CausalKinetiX. Again, we simulate from \hyperlink{dat:maillard}{Data
  Set~1}, where we consider different numbers of experimental
conditions. To ensure the comparison is fair, we choose the number of
repetitions per experimental condition to ensure that the number of
total observations is fixed at 16. Furthermore, in order to remove any
effect resulting from ordering of the variables we relabel them in
each repetition by randomly permuting the labels. Each ranking is then
assessed by computing the area under the operator receiver curve
(AUROC) based on the known ground truth (i.e., parents of the target
$\PA[]{Y}$). The results are given in
Figure~\ref{fig:increasing_exp}. Here, we applied CausalKinetiX using
the exhaustive model class (see \ref{sec:parametric_models}
Section~\ref{sec:mass-action}) and considered all possible models
consisting of individual variables and interactions ($66$ potential
predictor terms). We restricted the search to models with at most $4$
such terms after reducing the number of terms by a prior screening
step to $11$. The results show that CausalKinetiX benefits from an
increased number of experimental conditions.

\begin{figure}[h]
  \centering
  \includegraphics{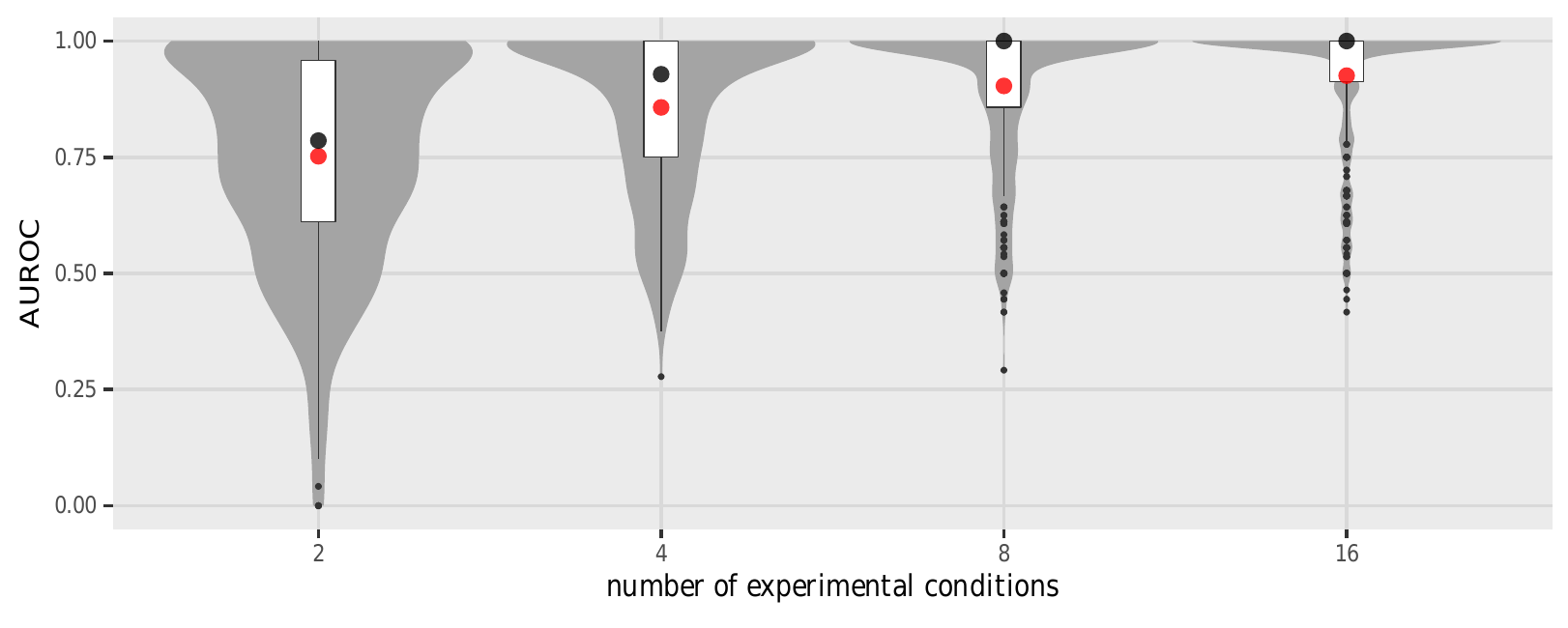}
  \caption{Results for simulation in
    Section~\ref{sec:increasing_exp}. For different numbers of
    experimental conditions we sampled $500$ simulations from
    \protect\hyperlink{dat:maillard}{Data Set~1}. For each simulation
    we compute the AUROC. Red points
    correspond to mean AUROC, large black
    points to median AUROC.}
  \label{fig:increasing_exp}
\end{figure}

\FloatBarrier

\subsubsection{Scalability} \label{sec:scalability} Next, we analyze
how our method scales with the number of variables $d$, the number of
environments $m$, the number of repetitions in each environment $R$,
and the number of observed time points for each trajectory $L$.
Figure~\ref{fig:runtime} illustrates the run-time of our method when
one of these parameters is varied while the others are kept fixed. The
data are generated according to \hyperlink{dat:targetmodel}{Data
  Set~2}. The key steps driving the computational cost of our
procedure are the smoothing in steps (M3) and~(M5), as well as the
estimation step (M4).  In our case, the cost of the estimation
procedure, fitting a linear model with ordinary least squares, is
negligible.  Since the number of smoothing operations, we have to
perform grows linearly with respect to $m$ and $R$, we expect a linear
increase in run-time.  Accordingly, the slopes in
Figure~\ref{fig:runtime} (bottom left and top right) are close to one.

We compute the smoothing spline in~(M3) using a convex quadratic
program, which can be solved in polynomial time -- even if the number
of constraints grows linearly, see~(M5).  The data points in
Figure~\ref{fig:runtime} (bottom right) do not lie on a straight line,
which may be due to some computational overhead for small values of
$L$ or due to the quadratic program itself.  The worst case complexity
of convex quadratic programming is cubic in sample size, but many
instances can be solved more efficiently.  Correspondingly, the slope
in Figure~\ref{fig:runtime} (bottom right) is not larger than three.
When only values $L \geq 64$ are taken into account, the slope is
estimated as $2.9$, which is close to the worst case guarantee of $3$.
Finally, varying the number of variables impacts the size of
$\mathcal{M}$, that is, the number of models.  In
Figure~\ref{fig:runtime}, we consider the case of main-effects models
of up to three variables, which results in $\mathcal{O}(d^3)$ models.
If we again assume that run-time of the estimation step can be
neglected, we expect a slope of $3$ in the log-log plot.  In our
empirical experiments the slope is estimated as $2.8$
(Figure~\ref{fig:runtime} top left).
\begin{figure}[h]
  \centering
  \includegraphics{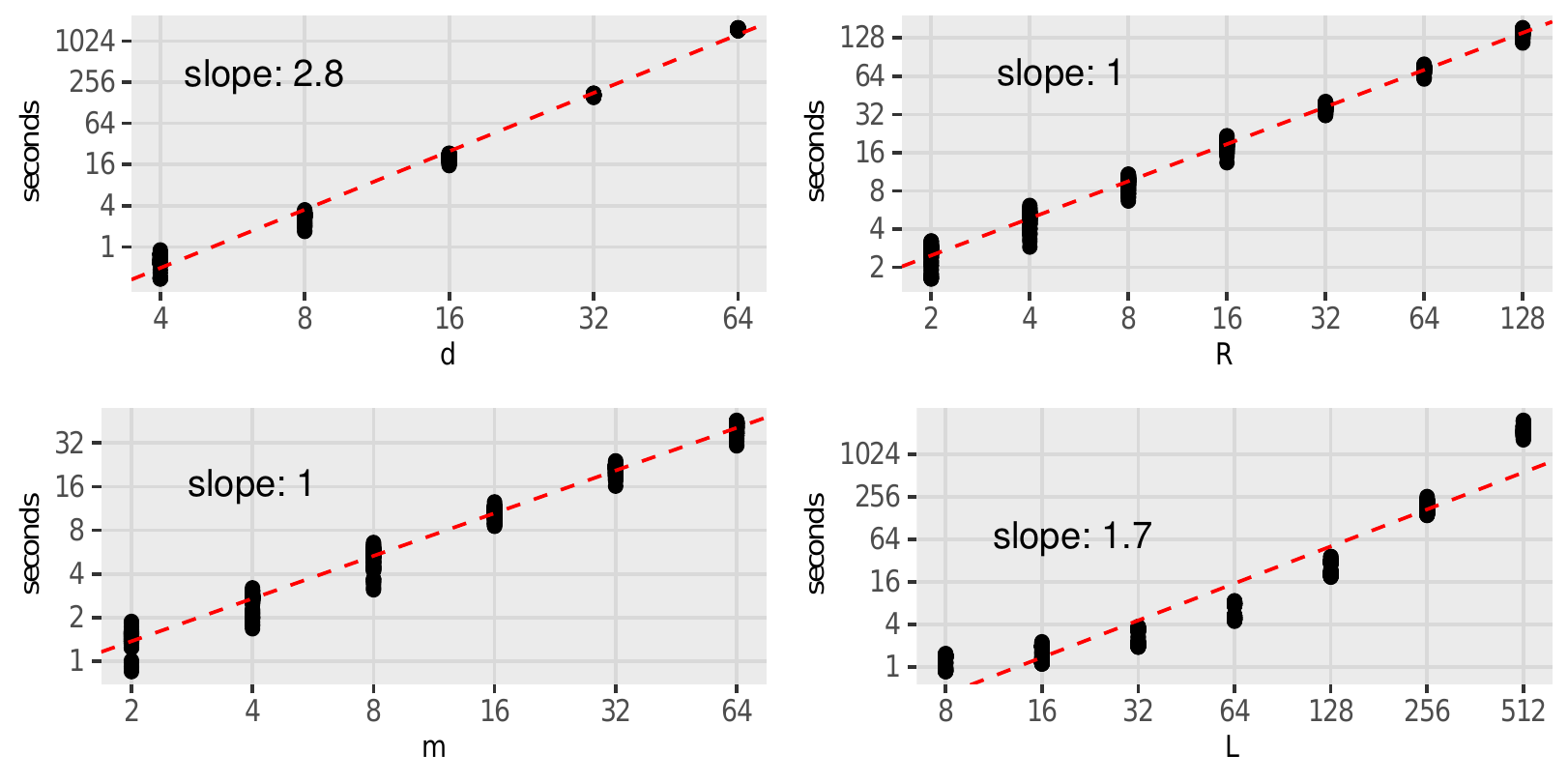}
  \caption{Run-time analysis for the parameters number of variables
    $d$, number of repetitions $R$, number of environments $m$ and
    number of time points $L$. In each plot one parameter is
    varied while the remaining are kept constant and the run-time is
    computed $100$ times for a full application of our method. The
    dotted red line is a linear fit to the log-log-transformed plots where
    the slope estimates the polynomial runtime order.}
  \label{fig:runtime}
\end{figure}

\FloatBarrier

\subsubsection{Allowing for complex predictor
  models}\label{sec:targetmodel_sim}
Our procedure requires that only the dynamics of the target variable
are given by an ODE model.  We do not model the dynamics of the
predictors, which as a consequence may follow any arbitrarily complex
model.  As an illustration we sample trajectories such that the
predictors are completely random and only the target variable
satisfies an invariant model. The details of the data generation are
shown in \hyperlink{dat:targetmodel}{Data Set~2}.

\begin{mdframed}[roundcorner=5pt,
  frametitle={\hypertarget{dat:targetmodel}{Data Set~2}: Target
    model based on predictor trajectories}]
  Consider functions of the form
  \begin{equation*}
    f_{c_1,c_2,c_3,c_4}(t)=\frac{c_1}{1+e^{c_2 (t-3)}}+\frac{c_3}{1+e^{c_4 (t-3)}},
  \end{equation*}
  i.e., these functions are linear combinations of sigmoids which have
  smooth trajectories that imitate dynamics observed in some real data
  experiment. For each of the $5$ experimental conditions we sample
  $d=12$ trajectories $X_t^j=f_{c_1,c_2,c_3,c_4}(t)$ for $t\in[0,6]$,
  where $c_1,c_2,c_3,c_4$ are i.i.d.\ standard normal. Based on these
  trajectories and the ODE given by
    \begin{equation*}
    \dot{Y}_t=\theta_1X^1+\theta_2X^2,\quad Y_0=0,
  \end{equation*}
  where $\theta_1=0.0001$ and $\theta_2=0.0002$,
    we compute
  the trajectories of the target variable
  $Y$ 
  by numerical integration. Finally, the
  observations are given as noisy versions of the values of these
  trajectories on an equally spaced time grid with $L=15$ time points
  between $0$ and $10$. The noise at each observation is independently
  normal distributed with mean $0$ and variance proportional to the
  total variation norm of the trajectory plus a small positive
  constant (in case the trajectory is constant), i.e.,
  $\sigma=c\cdot\norm{y}_{\operatorname{TV}}+10^{-7}$, where $y(t)$ are the
  true trajectories and $c\sim\operatorname{Unif}(0.05,0.15)$. Sample trajectories
  are given in Figure~\ref{fig:sample_trajectories_targetmodel}.
\end{mdframed}

The results are shown in
Figure~\ref{fig:simulation_targetmodel}. Here, we applied
CausalKinetiX for both the exhaustive and main effects model
class. For the exhaustive model class we again consider individual
variables and interactions as possible terms ($78$ terms) and reduce
to half these ($39$ terms) using screening and then apply our method
for all models with at most $3$ terms. For the main effect models we
perform no screening and consider all models consisting of at
most $3$ variables. Even though none of the predictors follows an ODE
model our procedure is capable of recovering the true causal parents
and again improves on plain prediction (DM and GM).

\begin{figure}[h]
  \centering
  \includegraphics{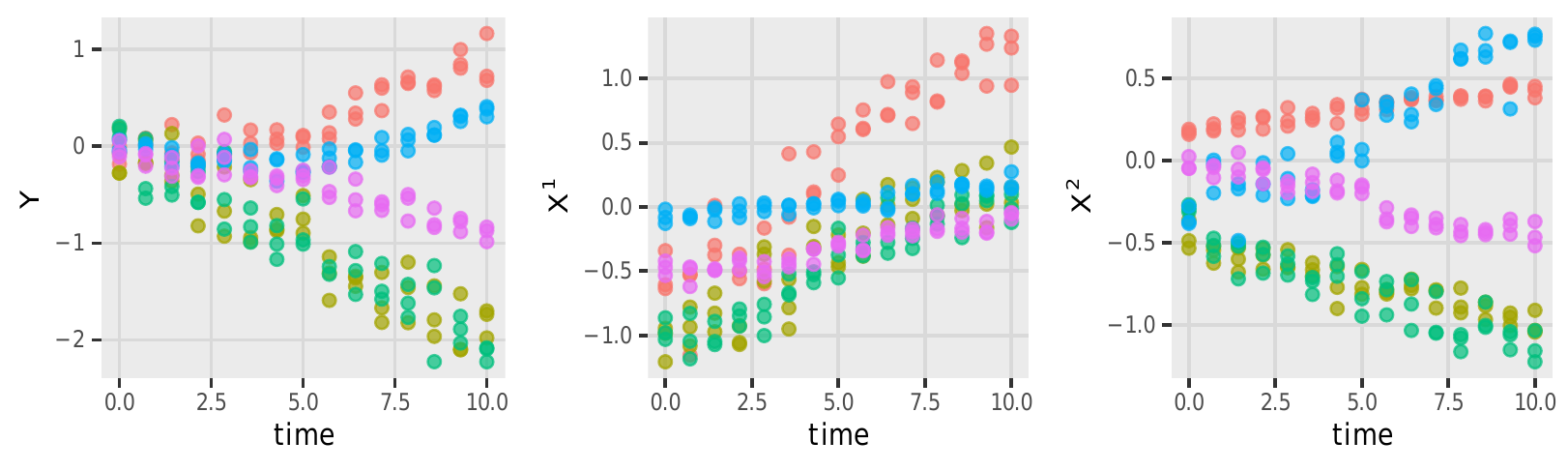}
  \caption{Sample observations for the target variable $Y$ and its
    two parents $X^1$ and $X^2$
    \protect\hyperlink{dat:targetmodel}{Data Set~2}. Points represent
    noisy observations with different colors used for the $5$
    different experiments, e.g., red corresponds to experimental
    condition~1.}
  \label{fig:sample_trajectories_targetmodel}
\end{figure}

\begin{figure}[h]
  \centering
  \includegraphics{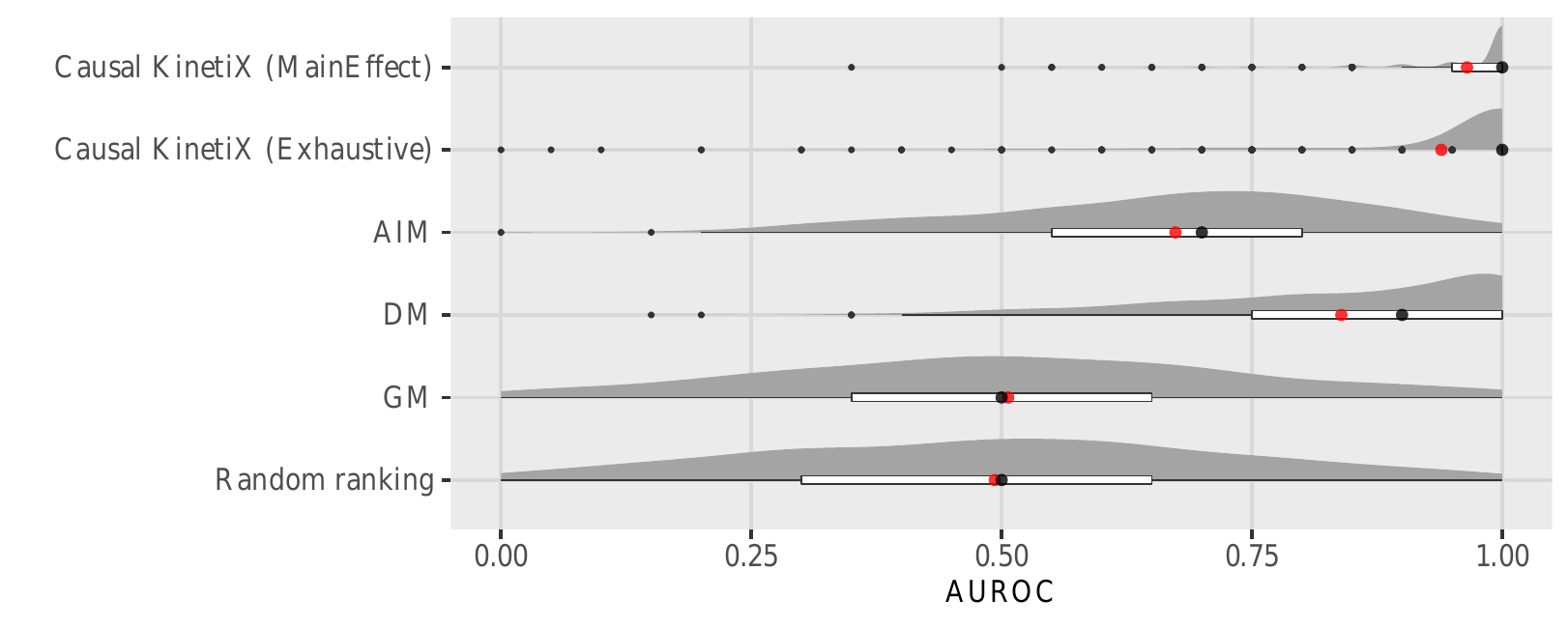}
  \caption{Results for simulation in
    Section~\ref{sec:targetmodel_sim}. Red points correspond to mean
    AUROC, large black points to median AUROC.}
  \label{fig:simulation_targetmodel}
\end{figure}

\FloatBarrier

\subsubsection{Robustness in the presence of hidden
  variables}\label{sec:hidden_exp}

In many practical applications hidden (unobserved) variables are
omnipresent. Since we only model the target equation, hidden variables
do not affect our methodology if they appear anywhere outside the
target variable $Y$.  In this section, we show that even if they enter
the target equation our procedure is generally expected to behave
well. The data in this section are generated according to
\hyperlink{dat:hidden}{Data Set~3}, which is based on an artificially
constructed ODE system, for which some of the variables are assumed to
be hidden. Example trajectories are shown in
Figure~\ref{fig:sample_trajectories_hidden}.

\begin{mdframed}[roundcorner=5pt,
  frametitle={\hypertarget{dat:hidden}{Data Set~3}: Hidden variable model}]
  The exact ODE structure is given in
  Section~\ref{sec:hidden_model}. We generate data from $16$
  experimental conditions and sample $3$ repetitions for each
  experiment. The experimental conditions are the following.
  \begin{compactitem}
  \item \textbf{Experimental condition $1$ (observational data):}\\
    Trajectories are simulated using the parameters given in
    Section~\ref{sec:hidden_model}.
  \item \textbf{Experimental conditions $2$ to $16$ (interventional data):}\\
    Trajectories are simulated based on the following two types of interventions
    \begin{compactitem}
    \item \textbf{initial value intervention:} Initial values are
      sampled for $X^1$, $X^2$ and $X^5$ uniform between $0$ and $10$,
      the remainder of the quantities are kept at zero initially as
      they are all products of the reactions.
    \item \textbf{blocking reactions:} Random reactions other than
      $k_4$, $k_5$ and $k_7$ are set to zero by fixing the
      corresponding reaction constant $k_i\equiv 0$. The expected
      number of reactions set to zero is $2$. Additionally, the rate
      $k_7$ is randomly perturbed either by sampling it uniform on
      $[0,0.2]$ or uniform on $[-0.1, 0.3]$.
    \end{compactitem}
  \end{compactitem}
  Based on these experimental conditions each true trajectory is
  computed using numerical integration. Finally, the observations are
  noisy versions of the values of these trajectories on an exponential
  time grid with $L=20$ time points between $0$ and $100$. The noise
  at each observation is independently normal distributed with mean
  $0$ and variance proportional to the total variation norm of the
  trajectory plus a small positive constant (in case the trajectory is
  constant), i.e.,
  $\sigma=c\cdot\norm{y}_{\operatorname{TV}}+10^{-7}$, where $y(t)$
  are the true trajectories and
  $c\sim\operatorname{Unif}(0.01,0.1)$. Example trajectories for the
  variables $X^2$ and $H^2$ depending on the values of $k_7$ are
  illustrated in Figure~\ref{fig:sample_trajectories_hidden}.
\end{mdframed}

\begin{figure}[h]
  \centering
  \includegraphics[height = 0.26\textheight]{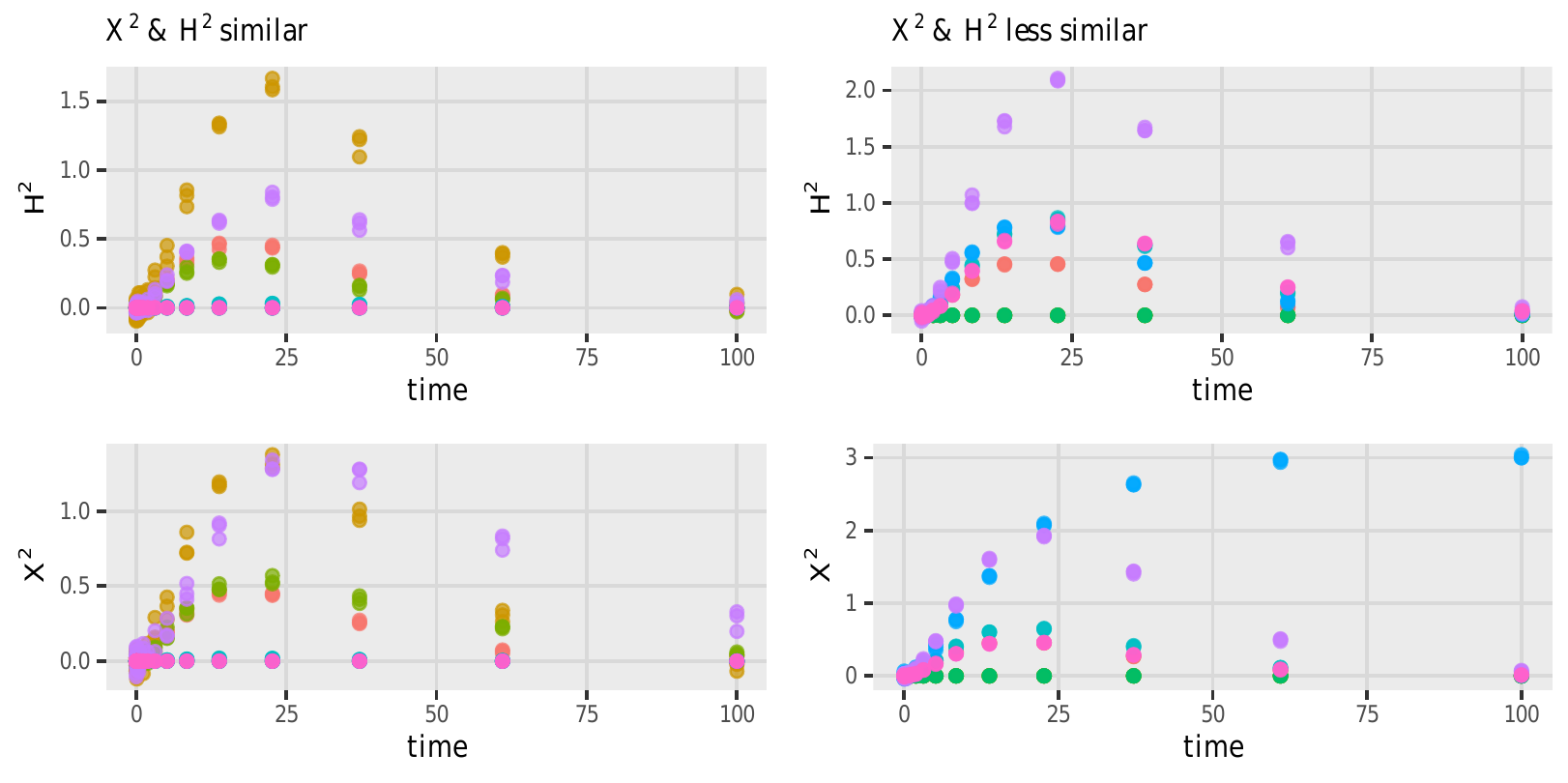}
  \caption{Sample observations of the two predictors $X^2$ and $H^2$
    from \protect\hyperlink{dat:hidden}{Data Set~3} (only first $8$
    experiments) for two different choices of perturbations of
    $k_7$. From left to right: For $k_7$ uniform on $[0,0.2]$ the
    dynamics are similar but not identical, for $k_7$ uniform on
    $[-0.1,0.3]$ the dynamics become very different. Points represent
    noisy observations of the underlying ODE trajectories.}
  \label{fig:sample_trajectories_hidden}
\end{figure}

We conduct three experiments, whose results are shown in
Figure~\ref{fig:hidden_ranking2} and
Figure~\ref{fig:hidden_ranking}. In the first setting (left plots),
all variables are observed.  The system of equations is built such
that $X^2$ and $H^2$ obey very similar but not identical trajectories
(here $k_7$ is perturbed less). Most methods are able to correctly
identify $X^3$ and $H^2$ as the direct causes of $Y$ -- those
variables are usually ranked highest, see
Figures~\ref{fig:hidden_ranking2} and~\ref{fig:hidden_ranking} (left).
In the second setting (middle plots), $H^2$ is unobserved.  Because of
the similarity between $H^2$ and $X^2$, the methods now infer $X^2$ as
a direct cause.  Finally, the third setting (right plots) differ from
the second setting in the sense that $H^2$ and $X^2$ are significantly
different (here $k_7$ is perturbed more).  The latter variable still
helps for prediction but does not yield an invariant
model. CausalKinetiX still reliably infers $X^3$ as a direct cause,
which is usually ranked higher than any of the other variables, see
Figures~\ref{fig:hidden_ranking2} and~\ref{fig:hidden_ranking}
(right). The results show that CausalKinetiX is relatively robust against
the existence of unobserved variables.

\begin{figure}[h]
  \centering{
  \includegraphics[height = 0.15\textheight]{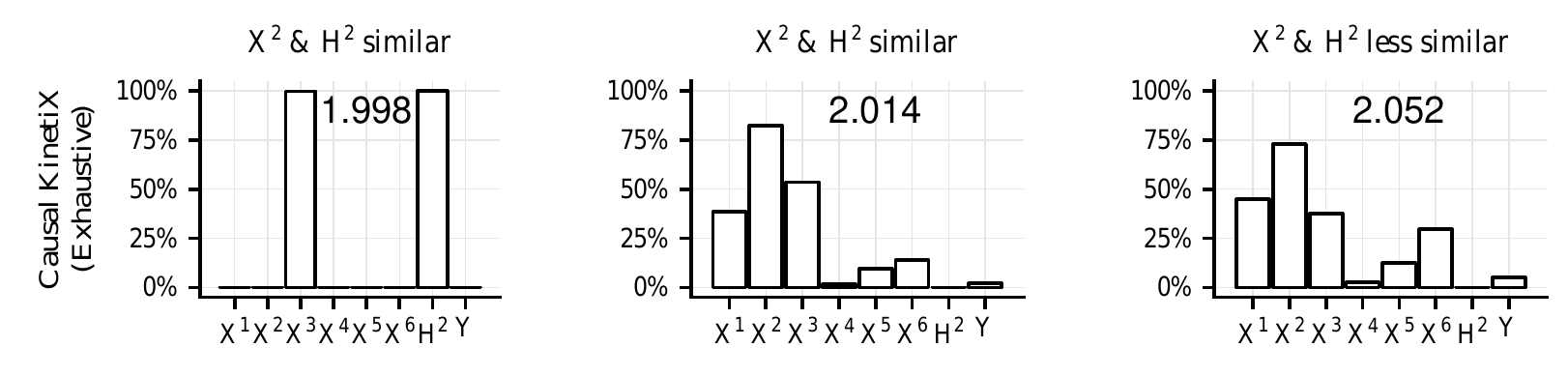} 
  \includegraphics[height = 0.15\textheight]{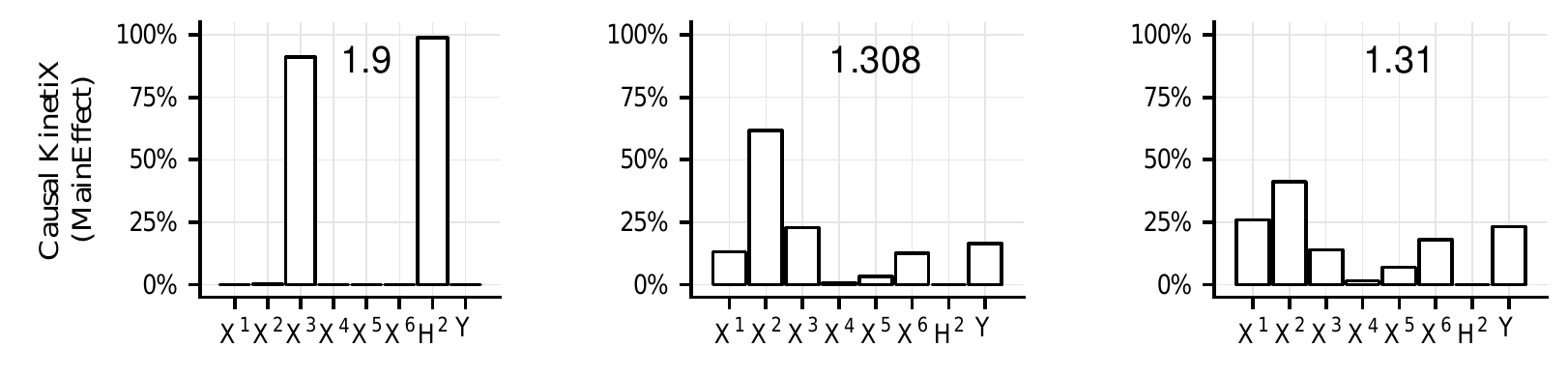}}
\caption{\label{fig:hidden_ranking2}Results for the experiment
  described in Section~\ref{sec:hidden_exp} (hidden variables).  Plot
  shows how often each variable gets a $p$-value smaller than $0.01$.
  The number on each histogram is the average number of significant
  variables at a $1\%$ level.}
\end{figure}

\begin{figure}
  \centering{
  \includegraphics[width=0.9\textwidth]{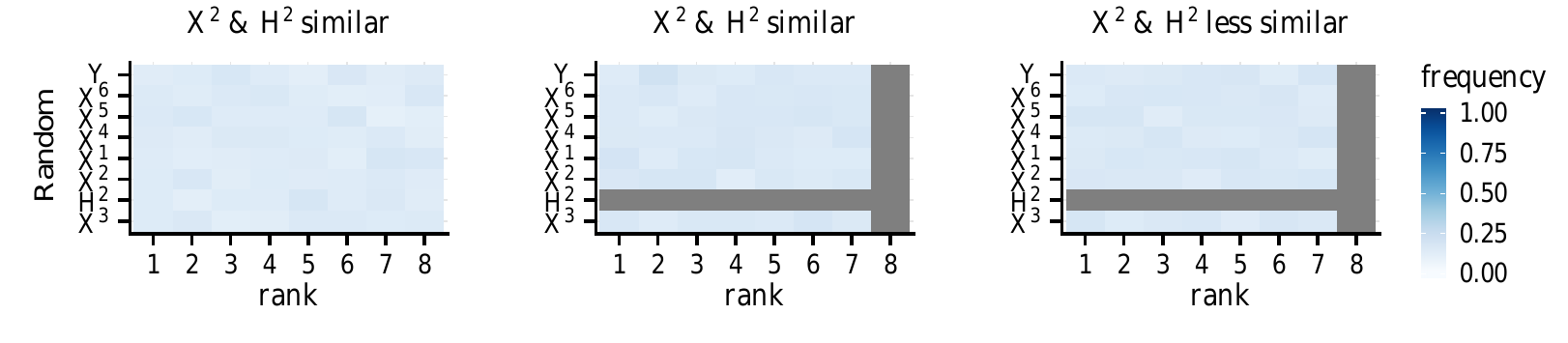}
  \includegraphics[width=0.9\textwidth]{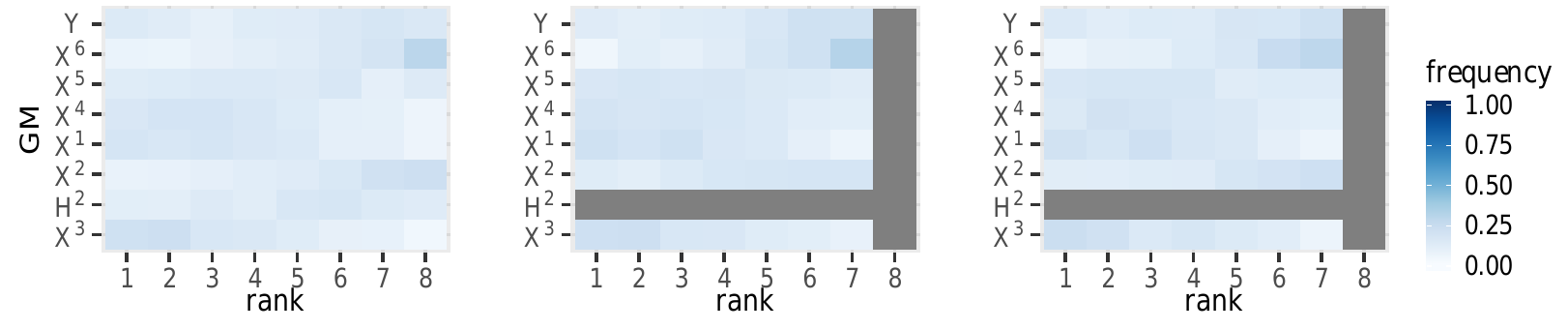}
  \includegraphics[width=0.9\textwidth]{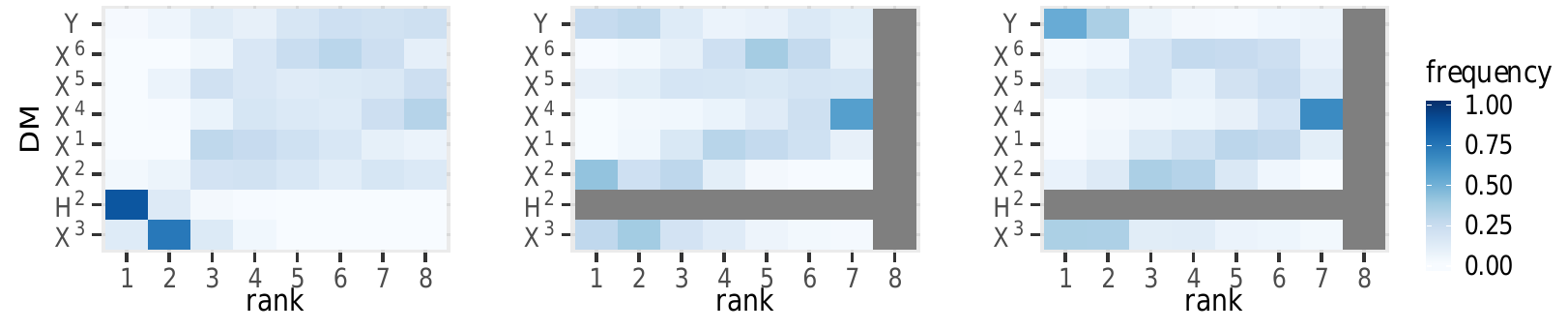}
  \includegraphics[width=0.9\textwidth]{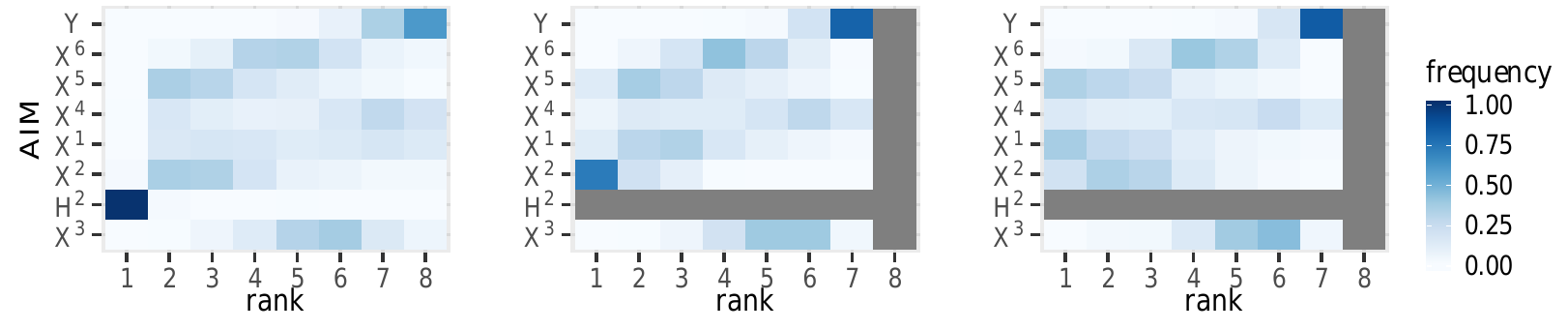}
  \includegraphics[width=0.9\textwidth]{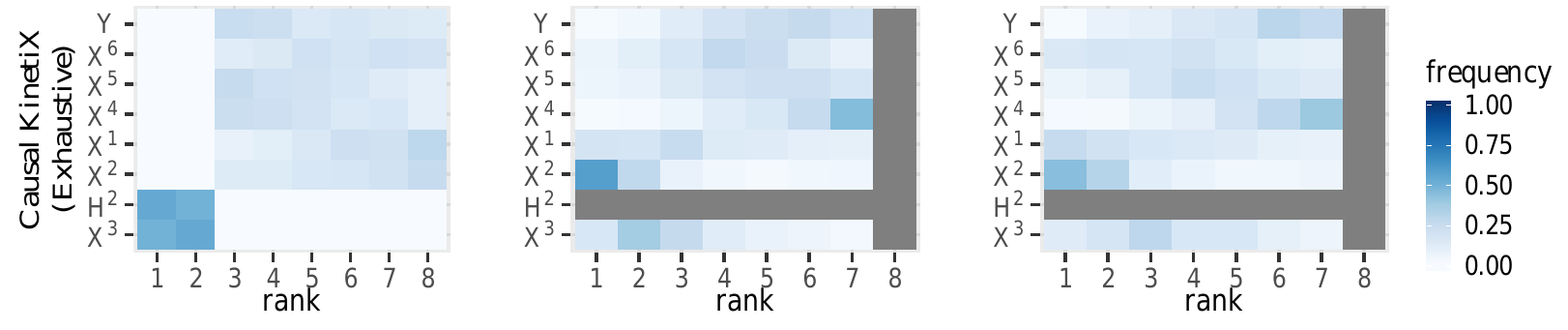}
  \includegraphics[width=0.9\textwidth]{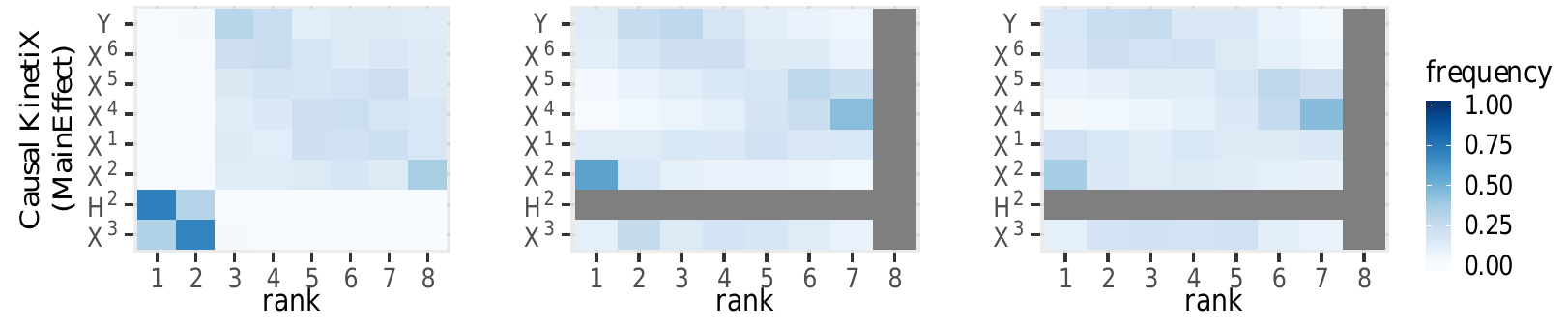}}
  \caption{  \label{fig:hidden_ranking}
Results for the experiment described in
    Section~\ref{sec:hidden_exp} (hidden variables). Left: all
    variables are observed.  From top to bottom: Random, GM,
    DM, AIM, CausalKinetiX (Exhaustive), CausalKinetiX
    (Main Effect)}
\end{figure}

\FloatBarrier

\subsection{Supplementary results to metabolic networks analysis}
The resulting integrated model fits when the entire model search
procedure is performed by holding out one experiment is given in
Figure~\ref{fig:fully_out_of_sample}. Despite the lack of
heterogeneity in two of these held-out experiments, the CausalKinetiX
variable ranking is very robust. To show this we look at the variable
rankings of the fully-out-of-sample experiments from
Figure~\ref{fig:fully_out_of_sample}. The results are presented in
the table below. The true causal variables, as
well as the true causal model, are unknown. For illustration purposes,
we indicate which of the highly ranked variables appear in the model
from above which has obtained the best score when based on all five
experiments. (This model was able to explain all the variation in the
different experiments, as illustrated by the plots in the main paper.)
As a comparison, when screening down to only three terms one obtains
the following different model
\begin{equation*}
  \dot{Y}_t = \theta_1 Z_t X_t^{128} X_t^{128}
  + \theta_2 Z_t X_t^{242} X_t^{298}
  -\theta_3 Y_t X_t^{33} X_t^{138},
\end{equation*}
which are the terms included in DM-NONLSQ-3.

\definecolor{darkgreen}{RGB}{0,100,0}
\small{
  \begin{mdframed}[roundcorner=5pt]
    \begin{minipage}[c]{0.6\textwidth}
      \centering
      \begin{tabular}{@{}cccccc@{}}
        & \multicolumn{5}{c}{\textbf{held-out-experiment}}\\
        \textbf{rank} &  1 &  2 &  3 &  4 &  5 \\ \midrule
        $\mathbf{1}$ & $\mathbf{\textcolor{darkgreen}{X^{33}}}$ & $\mathbf{\textcolor{darkgreen}{X^{33}}}$ & $\mathbf{\textcolor{darkgreen}{X^{33}}}$ & $\mathbf{\textcolor{darkgreen}{X^{33}}}$ & $\mathbf{\textcolor{darkgreen}{X^{33}}}$\\
        $\mathbf{2}$ & $\mathbf{\textcolor{darkgreen}{X^{56}}}$ & $X^{38}$ & $X^{73}$ & $X^{59}$ & $\mathbf{\textcolor{darkgreen}{X^{56}}}$\\
        $\mathbf{3}$ & $\mathbf{\textcolor{darkgreen}{X^{122}}}$ & $X^{61}$ &
                                                                              $\mathbf{\textcolor{darkgreen}{X^{122}}}$ & $\mathbf{\textcolor{darkgreen}{X^{128}}}$ & $\mathbf{\textcolor{darkgreen}{X^{122}}}$\\
        $\mathbf{4}$ & $\mathbf{\textcolor{darkgreen}{X^{128}}}$ &
                                                                   $\mathbf{\textcolor{darkgreen}{X^{128}}}$ & $\mathbf{\textcolor{darkgreen}{X^{138}}}$ & $\mathbf{\textcolor{darkgreen}{X^{168}}}$ & $\mathbf{\textcolor{darkgreen}{X^{128}}}$\\
        $\mathbf{5}$ & $\mathbf{\textcolor{darkgreen}{X^{138}}}$ & $\mathbf{\textcolor{darkgreen}{X^{138}}}$ & $\mathbf{\textcolor{darkgreen}{X^{168}}}$ & $X^{246}$ & $\mathbf{\textcolor{darkgreen}{X^{138}}}$\\
        $\mathbf{6}$ & $\mathbf{\textcolor{darkgreen}{X^{168}}}$ & $\mathbf{\textcolor{darkgreen}{X^{168}}}$ & $X^{215}$ & $X^{61}$ & $\mathbf{\textcolor{darkgreen}{X^{168}}}$
      \end{tabular}
    \end{minipage}
    \hfill\vline\hfill\hspace{0.02\textwidth}
    \begin{minipage}[c]{0.35\textwidth}
      \begin{align}
        \dot{Y}_t &= \theta_1 Z_t \mathbf{\textcolor{darkgreen}{X_t^{56} X_t^{122}}}\nonumber\\
        &\quad+ \theta_2 Z_t \mathbf{\textcolor{darkgreen}{X_t^{128} X_t^{168}}}\nonumber\\
        &\quad-\theta_3 Y_t \mathbf{\textcolor{darkgreen}{X_t^{33} X_t^{138}}}\label{eq:topmod}
      \end{align}
    \end{minipage}
  \end{mdframed}

\begin{figure}[h]
  \centering
  \includegraphics{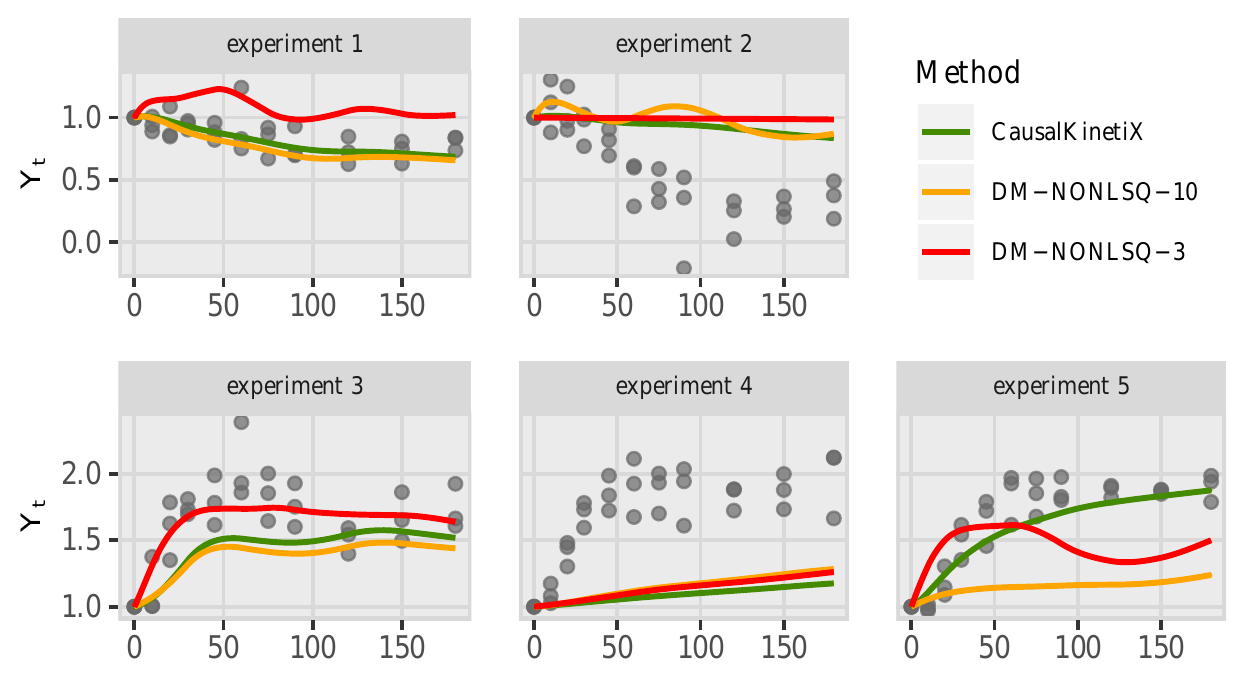}
  \caption{Metabolic network analysis. Fully~out-of-sample fit.  The
    plot shows the models' ability to generalize to new experiments.
    Each plot shows model-based trajectories that are obtained when
    that experiment is neither used for model identification nor
    parameter estimation. This is a very hard problem. CausalKinetiX
    shows the best generalization performance.}
  \label{fig:fully_out_of_sample}
\end{figure}

\subsubsection{Overfitting of trajectories in metabolic network}\label{sec:overfitting_traj}

To underscore the findings in Section~\ref{sec:overfitting_vars}, we
further illustrate the regularizing effect of stability component of
CausalKinetiX, we perform an additional experiment on the metabolic
data. Starting from the model [\ref{eq:topmod}] we proceed by adding 5
terms (from the top $1000$ screened terms) in a greedy fashion based
on two scores: (i) the standard CausalKinetiX score and (ii) the
modified CausalKinetiX score which does not hold out experiments in
step (M4) of the procedure. The results in
Figure~\ref{fig:overfitting_traj} highlight that stability is indeed
helpful for regularizing the fitted trajectories.

\clearpage

\begin{figure}[h]
  \centering
  \includegraphics{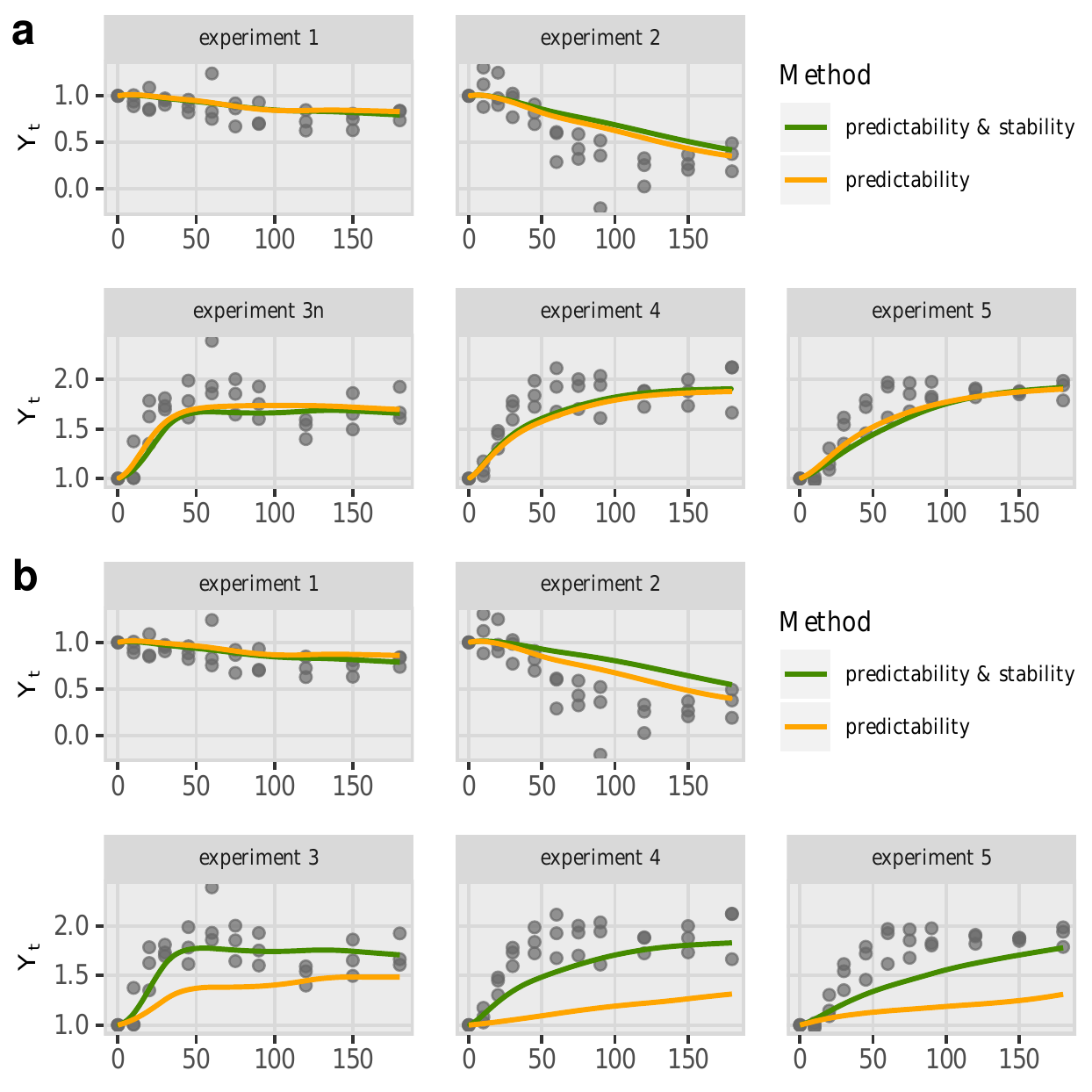}
  \caption{Stability as regularization to overfitting in metabolic
    network analysis. Comparison between two models consisting of 8
    terms where each is constructed in a greedy fashion using a score
    measuring mainly predictability and our proposed score which
    includes stability. In \textbf{a}, the in-sample trajectories are
    shown and the model based solely on predictability performs
    better. In \textbf{b}, the out-of-experiment performance (same
    type of sample-splitting as in the main article Figure 4
    \textbf{b}) of the same models are compared, illustrating the
    regularizing effect of including stability into the score.}
  \label{fig:overfitting_traj}
\end{figure}

\clearpage

\subsection{Additional details on simulation settings} \label{sec:additionbiomod52}

\subsubsection{Biomodel 52}\label{sec:biomodel52}~\\

\scriptsize{
\begin{minipage}{0.3\textwidth}
  \begin{mdframed}[roundcorner=5pt, frametitle={Reactions equations}]
    \begin{align*}
      \text{Glu} &\overset{k_1}{\longrightarrow} \text{Fru}\\
      \text{Fru} &\overset{k_2}{\longrightarrow} \text{Glu}\\
      \text{Glu} &\overset{k_3}{\longrightarrow} \text{Formic acid} +
                     \text{C5}\\
      \text{Fru} &\overset{k_4}{\longrightarrow} \text{Formic acid} +
                     \text{C5}\\
      \text{Fru} &\overset{k_5}{\longrightarrow} 2\cdot\text{Triose}\\
      \text{Triose} &\overset{k_6}{\longrightarrow} \text{Cn} +
                        \text{Acetic acid}\\
      \text{lys R} + \text{Glu} &\overset{k_7}{\longrightarrow}
                                      \text{Amadori}\\
      \text{Amadori} &\overset{k_8}{\longrightarrow} \text{Acetic
                         acid} + \text{lys R}\\
      \text{Amadori} &\overset{k_9}{\longrightarrow} \text{AMP}\\
      \text{lys R} + \text{Fru} &\overset{k_{10}}{\longrightarrow} \text{AMP}\\
      \text{AMP} &\overset{k_{11}}{\longrightarrow} \text{Melanoidin}
    \end{align*}
  \end{mdframed}
\end{minipage}
\hspace{2cm}
\begin{minipage}{0.5\textwidth}
  \begin{mdframed}[roundcorner=5pt, frametitle={ODE equations}]
    \begin{align*}
      \tfrac{\text{d}}{\text{dt}}[\text{Glu}] &= -(k_1+k_3)[\text{Glu}] + k_2[\text{Fru}] +
                                                k_7[\text{Glu}][\text{lys R}]\\
      \tfrac{\text{d}}{\text{dt}}[\text{Fru}] &= k_1[\text{Glu}] - (k_2+k_4+k_5)[\text{Fru}]
                                                -k_{10}[\text{Fru}][\text{lys R}]\\
      \tfrac{\text{d}}{\text{dt}}[\text{Formic acid}] &= k_3[\text{Glu}] + k_4[\text{Fru}]\\
      \tfrac{\text{d}}{\text{dt}}[\text{Triose}] &= 2k_5[\text{Fru}] - k_6[\text{Triose}]\\
      \tfrac{\text{d}}{\text{dt}}[\text{Acetic acid}] &= k_6[\text{Triose}] + k_8[\text{Amadori}]\\
      \tfrac{\text{d}}{\text{dt}}[\text{Cn}] &= k_6[\text{Triose}]\\
      \tfrac{\text{d}}{\text{dt}}[\text{Amadori}] &= -(k_8+k_9)[\text{Amadori}] + k_7[\text{Glu}][\text{lys R}]\\
      \tfrac{\text{d}}{\text{dt}}[\text{AMP}] &= k_9[\text{Amadori}] - k_{11}[\text{AMP}] +
                                                k_{10}[\text{Fru}][\text{lys R}]\\
      \tfrac{\text{d}}{\text{dt}}[\text{C5}] &= k_3[\text{Glu}] + k_4[\text{Fru}]\\
      \tfrac{\text{d}}{\text{dt}}[\text{lys R}] &= k_8[\text{Amadori}] - k_7[\text{Glu}][\text{lys R}] -
                                                  k_{10}[\text{Fru}][\text{lys R}]\\
      \tfrac{\text{d}}{\text{dt}}[\text{Melanoidin}] &=
                                                       k_{11}[\text{AMP}]
    \end{align*}
  \end{mdframed}
\end{minipage}
}

\scriptsize{
\begin{mdframed}[roundcorner=5pt, frametitle={Parameters and initial
    conditions}]
  \begin{minipage}{0.4\textwidth}
    \begin{align*}
      k_1&=0.01\\
      k_2&=0.00509\\
      k_3&=0.00047\\
      k_4&=0.0011\\
      k_5&=0.00712\\
      k_6&=0.00439\\
      k_7&=0.00018\\
      k_8&=0.11134\\
      k_9&=0.14359\\
      k_{10}&=0.00015\\
      k_{11}&=0.12514
    \end{align*}
  \end{minipage}
  \hfill\vline\hfill
  \begin{minipage}{0.4\textwidth}
    \begin{align*}
      [\text{Glu}]\mid_{t=0} &= 160\\
      [\text{Fru}]\mid_{t=0} &= 0\\
      [\text{Formic acid}]\mid_{t=0} &= 0\\
      [\text{Triose}]\mid_{t=0} &= 0\\
      [\text{Acetic acid}]\mid_{t=0} &= 0\\
      [\text{Cn}]\mid_{t=0} &= 0\\
      [\text{Amadori}]\mid_{t=0} &= 0\\
      [\text{AMP}]\mid_{t=0} &= 0\\
      [\text{C5}]\mid_{t=0} &= 0\\
      [\text{lys R}]\mid_{t=0} &= 15\\
      [\text{Melanoidin}]\mid_{t=0} &=0
    \end{align*}
  \end{minipage}
\end{mdframed}
}

\clearpage

\subsubsection{Artificial hidden variable
  model}\label{sec:hidden_model}~\\

\scriptsize{
\begin{minipage}{0.3\textwidth}
  \begin{mdframed}[roundcorner=5pt, frametitle={Reactions equations}]
    \begin{align*}
      X^1 &\overset{k_1}{\longrightarrow} H^1\\
      H^1 &\overset{k_2}{\longrightarrow} X^2 + H^2\\
      H^1 &\overset{k_3}{\longrightarrow} X^1\\
      H^2 &\overset{k_4}{\longrightarrow} Y\\
      X^3 &\overset{k_5}{\longrightarrow} Y + X^4\\
      X^1 + X^4 &\overset{k_6}{\longrightarrow} X^3\\
      X^2 &\overset{k_7}{\longrightarrow} X^6\\
      X^5 &\overset{k_8}{\longrightarrow} X^3\\
      X^4 &\overset{k_9}{\longrightarrow} X^5
    \end{align*}
  \end{mdframed}
\end{minipage}
\hspace{2cm}
\begin{minipage}{0.5\textwidth}
  \begin{mdframed}[roundcorner=5pt, frametitle={ODE equations}]
    \begin{align*}
      \tfrac{\text{d}}{\text{dt}}[X^1] &= -k_1[X^1] + k_3[H^1] - k_6[X^1][X^4]\\
      \tfrac{\text{d}}{\text{dt}}[X^2] &= k_2[H^1] - k_7[X^2]\\
      \tfrac{\text{d}}{\text{dt}}[X^3] &= -k_5[X^3] + k_6[X^1][X^4] + k_8[X^5]\\
      \tfrac{\text{d}}{\text{dt}}[X^4] &= k_5[X^3] - k_6[X^1][X^4] - k_9[X^4]\\
      \tfrac{\text{d}}{\text{dt}}[X^5] &= k_9[X^4] - k_8[X^5]\\
      \tfrac{\text{d}}{\text{dt}}[X^6] &= k_7[X^2]\\
      \tfrac{\text{d}}{\text{dt}}[H^1] &= k_1[X^1] - (k_2+k_3)[H^1]\\
      \tfrac{\text{d}}{\text{dt}}[H^2] &= k_2[H^1] - k_4[H^2]\\
      \tfrac{\text{d}}{\text{dt}}[Y] &= k_4[H^2] + k_5[X^3]
    \end{align*}
  \end{mdframed}
\end{minipage}
}

\scriptsize{
\begin{mdframed}[roundcorner=5pt, frametitle={Parameters and initial
    conditions}]
  \begin{minipage}{0.4\textwidth}
    \begin{align*}
      k_1&=0.08\\
      k_2&=0.08\\
      k_3&=0.01\\
      k_4&=0.1\\
      k_5&=0.003\\
      k_6&=0.06\\
      k_7&=0.1\\
      k_8&=0.02\\
      k_9&=0.05
    \end{align*}
  \end{minipage}
  \hfill\vline\hfill
  \begin{minipage}{0.4\textwidth}
    \begin{align*}
      [X^1]\mid_{t=0} &= 5\\
      [X^2]\mid_{t=0} &= 0\\
      [X^3]\mid_{t=0} &= 0\\
      [X^4]\mid_{t=0} &= 5\\
      [X^5]\mid_{t=0} &= 0\\
      [X^6]\mid_{t=0} &= 0\\
      [H^1]\mid_{t=0} &= 0\\
      [H^2]\mid_{t=0} &= 0\\
      [Y]\mid_{t=0} &= 0
    \end{align*}
  \end{minipage}
\end{mdframed}
}

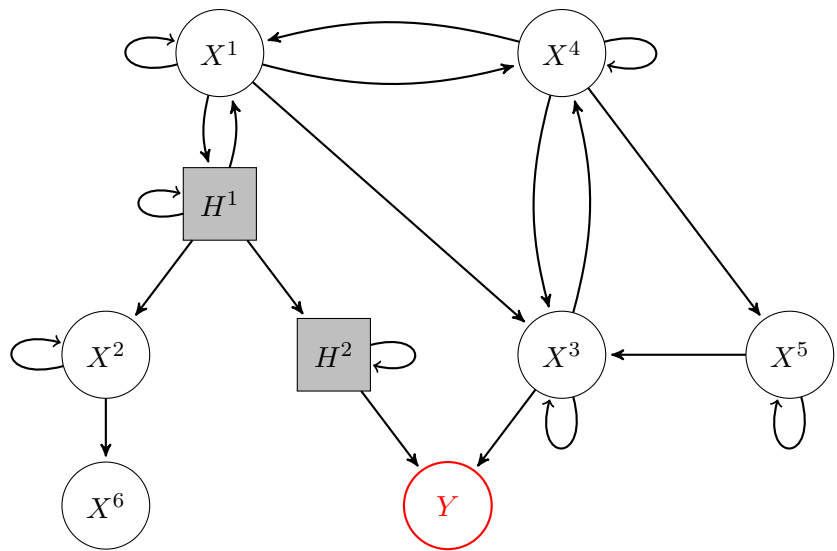
\begin{figure}[h!]
  \centering
  \begin{tikzpicture}[scale=1]
    \tikzstyle{VertexStyle} = [shape = circle, minimum width = 3em,draw]
    \SetGraphUnit{2}
    \Vertex[Math,L=X^1,x=-3,y=4]{X1}
    \Vertex[Math,L=X^2,x=-4.5,y=0]{X2}
    \Vertex[Math,L=X^3,x=1.5,y=0]{X3}
    \Vertex[Math,L=X^4,x=1.5,y=4]{X4}
    \Vertex[Math,L=X^5,x=4.5,y=0]{X5}
    \Vertex[Math,L=X^6,x=-4.5,y=-2]{X6}
    \tikzstyle{VertexStyle} = [shape = circle, minimum width =
    3em, draw, color=red, thick]
    \Vertex[Math,L=Y,x=0,y=-2]{Y}
    \tikzstyle{VertexStyle} = [shape = rectangle, minimum width = 2.5em,
    minimum height = 2.5em, draw, fill=lightgray]
    \Vertex[Math,L=H^1,x=-3,y=2]{H1}
    \Vertex[Math,L=H^2,x=-1.5,y=0]{H2}
    \tikzstyle{EdgeStyle} = [->,>=stealth',shorten > = 2pt]
    \Edge(X1)(X3)
    \Edge(X4)(X5)
    \Edge(X5)(X3)
    \Edge(H1)(H2)
    \Edge(X3)(Y)
    \Edge(H2)(Y)
    \Edge(H1)(X2)
    \Edge(X2)(X6)
    \tikzset{EdgeStyle/.append style = {->, bend right=15}}
    \Edge(X4)(X1)
    \Edge(X1)(X4)
    \Edge(H1)(X1)
    \Edge(X1)(H1)
    \Edge(X4)(X3)
    \Edge(X3)(X4)
    \draw (X1) edge [loop left, thick](X1);
    \draw (X2) edge [loop left, thick](X2);
    \draw (X3) edge [loop below, thick](X3);
    \draw (X4) edge [loop right, thick](X4);
    \draw (X5) edge [loop below, thick](X5);
    \draw (H1) edge [loop left, thick](H1);
    \draw (H2) edge [loop right, thick](H2);
  \end{tikzpicture}
  \caption{Graph representation of hidden variable ODE model. If the
    rate $k_4$ is equal to the rate $k_7$ the variables $X^2$ and
    $H^2$ will have identical dynamics.}
  \label{fig:graph}
\end{figure}}
